\documentclass[10pt]{article}
\usepackage[preprint]{tmlr}

\usepackage[utf8]{inputenc}
\usepackage[T1]{fontenc}
\usepackage[english]{babel}

\usepackage{amsmath}
\usepackage{amssymb}
\usepackage{amsfonts}
\usepackage{amsthm}
\usepackage{latexsym}
\usepackage{bm}
\usepackage{xspace}
\usepackage{graphicx}
\usepackage{bbm}
\usepackage{xcolor}
\usepackage{listings}
\usepackage{caption}
\usepackage{float}
\usepackage{subcaption}
\usepackage[ruled,vlined]{algorithm2e}
\usepackage{enumerate}
\usepackage{wrapfig}
\usepackage{longtable}
\usepackage{hyperref}
\usepackage{url}

\usepackage{dsfont}
\def\bbR{\mathbb{R}}

\def\ba{\mathbf{a}}

\def\bff{\mathbf{f}}
\def\bg{\mathbf{g}}
\def\bh{\mathbf{h}}

\def\bp{\mathbf{p}}

\def\br{\mathbf{r}}
\def\bs{\mathbf{s}}
\def\bt{\mathbf{t}}

\def\bv{\mathbf{v}}

\def\bx{\mathbf{x}}
\def\by{\mathbf{y}}

\def\bA{\mathbf{A}}
\def\bB{\mathbf{B}}
\def\bC{\mathbf{C}}

\def\bE{\mathbf{E}}
\def\bF{\mathbf{F}}
\def\bG{\mathbf{G}}
\def\bH{\mathbf{H}}
\def\bI{\mathbf{I}}
\def\bJ{\mathbf{J}}
\def\bK{\mathbf{K}}

\def\bM{\mathbf{M}}

\def\bO{\mathbf{O}}
\def\bP{\mathbf{P}}
\def\bQ{\mathbf{Q}}
\def\bR{\mathbf{R}}
\def\bS{\mathbf{S}}
\def\bT{\mathbf{T}}
\def\bU{\mathbf{U}}
\def\bV{\mathbf{V}}
\def\bW{\mathbf{W}}
\def\bX{\mathbf{X}}
\def\bY{\mathbf{Y}}
\def\bZ{\mathbf{Z}}

\def\bbR{\mathbb{R}}

\def\bbZ{\mathbb{Z}}

\def\bzero{\mathbf{0}}

\def\caL{\mathcal{L}}

\def\caN{\mathcal{N}}
\def\caO{\mathcal{O}}

\newcommand{\done}{\mathds{1}}

\newcommand{\bxi}{\bm{\xi}}

\newcommand{\bGamma}{\bm{\Gamma}}
\newcommand{\bDelta}{\bm{\Delta}}

\newcommand{\bPi}{\bm{\Pi}}

\newcommand{\simmat}{\bPi_{\done}}

\newcommand{\frob}[1]{\left\|#1\right\|_F}
\newcommand{\twonorm}[1]{\left\|#1\right\|_2}
\newcommand{\frobsq}[1]{\left\|#1\right\|_F^2}

\newcommand{\twonormsq}[1]{\left\|#1\right\|_2^2}
\newcommand{\tsim}[1]{\ensuremath{\mathrm{t}_{\mathrm{sim}}(#1)}}
\newcommand{\tdiv}[1]{\ensuremath{\mathrm{t}_{\mathrm{div}}(#1)}}
\newcommand{\inner}[2]{\left\langle #1, #2\right\rangle}

\newcommand{\pardir}[2]{\frac{\partial #1}{\partial #2}}

\def\bDelta{\bm\Delta}

\def\bxi{\bm\xi}

\def\tr{\mathrm{tr}}

\def\dff{d_{\text{ff}}}

\def\ba{\mathbf{a}}

\def\bff{\mathbf{f}}
\def\bg{\mathbf{g}}
\def\bh{\mathbf{h}}

\def\bp{\mathbf{p}}

\def\br{\mathbf{r}}
\def\bs{\mathbf{s}}
\def\bt{\mathbf{t}}

\def\bv{\mathbf{v}}

\def\bx{\mathbf{x}}
\def\by{\mathbf{y}}

\def\bA{\mathbf{A}}
\def\bB{\mathbf{B}}
\def\bC{\mathbf{C}}

\def\bE{\mathbf{E}}
\def\bF{\mathbf{F}}
\def\bG{\mathbf{G}}
\def\bH{\mathbf{H}}
\def\bI{\mathbf{I}}
\def\bJ{\mathbf{J}}
\def\bK{\mathbf{K}}

\def\bM{\mathbf{M}}

\def\bO{\mathbf{O}}
\def\bP{\mathbf{P}}
\def\bQ{\mathbf{Q}}
\def\bR{\mathbf{R}}
\def\bS{\mathbf{S}}
\def\bT{\mathbf{T}}
\def\bU{\mathbf{U}}
\def\bV{\mathbf{V}}
\def\bW{\mathbf{W}}
\def\bX{\mathbf{X}}
\def\bY{\mathbf{Y}}
\def\bZ{\mathbf{Z}}

\def\bzero{\mathbf{0}}

\def\caL{\mathcal{L}}

\def\caN{\mathcal{N}}
\def\caO{\mathcal{O}}

\providecommand{\ldepth}{\tlower{L}{depth}}

\def\smskip{\smallskip}

\def\texitem#1{\par\smskip\noindent\hangindent 25pt
               \hbox to 25pt {\hss #1 ~}\ignorespaces}

\def\abs#1{|#1|}

\providecommand{\BEAS}{\begin{eqnarray*}}
\providecommand{\EEAS}{\end{eqnarray*}}
\providecommand{\BEA}{\begin{eqnarray}}
\providecommand{\EEA}{\end{eqnarray}}
\providecommand{\BEQ}{\begin{eqnarray}}
\providecommand{\EEQ}{\end{eqnarray}}
\providecommand{\BIT}{\begin{itemize}}
\providecommand{\EIT}{\end{itemize}}
\providecommand{\BNUM}{\begin{enumerate}}
\providecommand{\ENUM}{\end{enumerate}}

\providecommand{\BA}{\begin{array}}
\providecommand{\EA}{\end{array}}

\providecommand{\expect}[2]{\mathbb{E}_{#1}\left[#2\right]}

\providecommand\tlower[2]{#1_{\mathrm{#2}}}

\def\red#1{\textcolor{black}{#1}}
\def\blue#1{\textcolor{black}{#1}}

\let\oldref\ref
\renewcommand{\ref}[1]{(\oldref{#1})}

\providecommand{\attn}{\ensuremath{\mathrm{Attn}}}
\providecommand{\rms}{\ensuremath{\mathrm{RMS}}}
\providecommand{\ffn}{\ensuremath{\mathrm{FFN}}}

\providecommand{\proj}{\ensuremath{\mathrm{Proj}}}
\providecommand{\softmax}{\ensuremath{\mathrm{Softmax}}}
\providecommand{\softmaxrow}{\ensuremath{\mathrm{Softmax}_{\mathrm{row}}}}

\providecommand{\silu}{\ensuremath{\mathrm{SiLU}}}
\providecommand{\sigmoid}{\ensuremath{\mathrm{Sigmoid}}}

\providecommand{\Diag}{\ensuremath{\mathrm{Diag}}}

\providecommand{\wlm}{\tlower{\bW}{lm}}

\newtheorem{theorem}{Theorem}[section]
\newtheorem{corollary}[theorem]{Corollary}
\newtheorem{definition}[theorem]{Definition}
\newtheorem{lemma}[theorem]{Lemma}
\newtheorem{assumption}[theorem]{Assumption}

\newtheorem{proposition}[theorem]{Proposition}

\title{Why Post-Norm Transformers Collapse: Attention Amplification and Gradient Repair Failure}
\author{\name Xingjian Wang \email xingjianwang@link.cuhk.edu.cn \\
\addr The Chinese University of Hong Kong, Shenzhen
\AND
\name Qingyu Han \email qingyuhan@link.cuhk.edu.cn \\
\addr The Chinese University of Hong Kong, Shenzhen
\AND
\name Xiaodong Luo$^*$ \email xiaodongluo@cuhk.edu.cn\\ 
\addr The Chinese University of Hong Kong, Shenzhen \\
\addr Shenzhen Research Institute of Big Data
\AND
\name Yin Zhang$^*$ \email yinzhang@cuhk.edu.cn \\
\addr The Chinese University of Hong Kong, Shenzhen
}

\def\month{08}
\def\year{2026}

\makeatletter
\newcommand*\bigcdot{\mathpalette\bigcdot@{.5}}
\newcommand*\bigcdot@[2]{\mathbin{\vcenter{\hbox{\scalebox{#2}{$\m@th#1\bullet$}}}}}
\makeatother

\begin{document}

\newcommand{\SubItem}[1]{
    {\setlength\itemindent{15pt} \item[-] #1}
}

\maketitle

\begin{abstract}
Deep decoder-only Transformers often replace the original Post-Norm architecture with Pre-Norm variants because Post-Norm training is highly sensitive to warmup and learning rate under conventional initialization schemes. Although prior work has identified rank collapse and gradient vanishing as related symptoms, it remains poorly understood how causal attention creates high-similarity representations and why training dynamics fail to repair them. We give a two-stage analysis of Post-Norm rank collapse using token similarity as a scalar state variable. First, at initialization, causal attention acts approximately as a prefix-averaging operator that increases token similarity across depth, while the SwiGLU branch contributes only a smaller damping effect. Second, once training enters a high-similarity regime, growth of pre-normalization residual norms makes the RMSNorm backward factor contractive; under mild conditions, gradients to earlier layers decay geometrically. \red{As a complementary result,  we characterize the properties of a collapsed network: its best predictor is frequency distribution with relatively high loss floor, and gradients in collapsed layers vanish at frequency distribution.} Experiments on 48-layer decoder-only Transformers trained on C4 dataset match the predicted initialization-time similarity growth and collapse-time gradient contraction, \red{and show that collapsed runs stay near the predicted frequency loss.} \red{Together, these results distinguish the forward similarity amplification and backward repair incapacity in Post-Norm collapse,  while also characterizing the behavior of collapsed networks.}

\end{abstract}

\begingroup
\renewcommand{\thefootnote}{\fnsymbol{footnote}}

\footnotetext[1]{Corresponding author.}
\endgroup
\section{Introduction}
The Transformer architecture \citep{vaswani2017attention}, combining self-attention, residual connections \citep{he2016deep}, and layer normalization \citep{ba2016layer}, has become the foundation of modern large-scale models across natural language processing and vision \citep{brown2020language,devlin2019bert,dosovitskiy2020image}. Subsequent modifications such as SwiGLU FFN layers \citep{shazeer2020glu} and RMSNorm \citep{zhang2019root} have further improved its performance. One of the key strengths of Transformers is scalability: deeper and wider models can achieve better performance \citep{kaplan2020scaling,hoffmann2022training}. However, increasing depth in the original Transformer design encounters obstacles. Training typically requires conservative warmup and small initial step sizes, otherwise optimization can converge to suboptimal states. The problem lies in the placement of normalization: the original architecture uses Post-Norm, where normalization is applied after the residual addition \citep{xiong2020layer}. By replacing Post-Norm with Pre-Norm, where normalization is applied before the residual addition, the sensitivity to warmup and step size is mitigated \citep{xiong2020layer,liu2020understanding}.

The key failure mode of Post-Norm is \textbf{rank collapse}: token representations in a sequence become nearly identical, reducing effective dimensionality and harming optimization \citep{noci2022signal,saada2024mind}. Many authors have studied neighboring aspects of this phenomenon. \citet{noci2022signal} analyze initialization-time collapse from a signal-propagation perspective and show that query/key gradients can vanish as collapse develops. \citet{saada2024mind} refine this picture through a spectral characterization of collapse in attention layers. \citet{yu2026classic} give a quantitative initialization analysis for encoder transformers. \citet{chen2026from} study collapse as a training-dynamics phenomenon in a simplified single-layer setting.

\red{Prior work has identified phenomena closely related to Post-Norm architectural failure modes, including gradient vanishing and rank collapse}\citep{noci2022signal,yu2026classic,xiong2020layer,chen2026from}. \red{However, beyond initialization, it remains unclear how these phenomena interact in causal decoder-only Transformers during training. In particular, we lack a clear explanation of how rank collapse depends on normalization placement, and why optimization fails to revert rank collapse once it emerges}.
\red{This leaves a gap between practical fixes and theoretical understanding: we know practical ways to mitigate Post-Norm failures, but not why causal decoder representations enter a collapsed regime or why training struggles to escape it.}

\red{In this paper, our central message is that Post-Norm rank collapse is not caused by a single instability, but by the combination of two mechanisms: attention amplification, which drives token similarity upward at initialization, and RMSNorm-induced gradient shrinkage, which reduces the gradient that reaches earlier layers during collapse.}
To track this process throughout optimization, we use token similarity as a scalar state variable \citep{yu2026classic}.
\red{Our analysis explains how token similarity increases in Post-Norm at initialization and why high token similarity states are hard to repair during training. As a complementary result, we characterize the properties of a collapsed Post-Norm network, showing that it can only obtain suboptimal solutions, with parameter gradients vanish in collapsed layers.} To sum up, our contributions are as follows:

\begin{enumerate}
    \item \textbf{Two-Stage Account.}
    \red{We show that Post-Norm rank collapse can be understood as a two-stage process. First, under equal-correlation and prefix-averaging approximations, we show that causal attention amplifies token similarity at initialization, with the amplification controlled by the amount of attention added to each layer. Second, when residual norms grow during training and sublayer's gradient contribution remains bounded, the RMSNorm backward factor can fall below one, causing gradients to earlier layers to decay exponentially with depth.} Together, these mechanisms give one account of why Post-Norm networks are more prone to collapse than Pre-Norm.

    \item \textbf{Collapsed Network Characterization.}
    We characterize two properties of a collapsed Post-Norm network. First, \red{under rank collapse,} the optimal probability distribution is the frequency distribution, where each token's probability is proportional to its occurrence in the training data. This yields a relatively high loss floor. Second, parameter gradients vanish in the collapsed layers. \red{Combined with gradient vanishing in earlier layers, this makes the collapsed state near a stationary point.}

    \item \textbf{Experimental Verification.}
    Controlled initialization measurements match the predicted token-similarity growth curves across depth, supporting the attention-amplification mechanism. By removing the prefix-averaging component, we suppress the similarity growth at initialization, showing that prefix-averaging component indeed causes similarity growth. \red{Training diagnostics on collapsing 48-layer Post-Norm runs then show residual-norm growth, contraction of the RMSNorm factor, and gradient vanishing in earlier layers, consistent with much weaker gradients near collapse. Finally, collapsed training runs stay near the predicted frequency loss, matching the collapsed-network characterization.}
\end{enumerate}

The paper is organized accordingly. Sections \ref{section_forward_amplification} and \ref{section_backward_repair_bottleneck} develop the two stage analysis.  \red{Section \ref{section_frequency_distribution} then gives a complementary characterization of a collapsed network. Section \ref{section_experiments} tests both the mechanistic analysis and the collapsed-network characterization.}

\section{Preliminaries}\label{section_setup}

We use bold uppercase letters for matrices and bold lowercase letters for vectors. The symbol $\bA \odot \bB$ denotes the Hadamard product. Let $n$ be sequence length, $d$ model dimension, $v$ vocabulary size, $l$ layer index, and $\ldepth$ total number of layers. Matrix entry notation such as $\bA_{i, j}$ denotes the element of matrix $\bA$ at row $i$ and column $j$. The vector $\done_n$ denotes the all-one column vector in $\mathbb{R}^{n\times 1}$, and $\done_n^T$ denotes its row form. We define
\[
\simmat = \frac{1}{n}\done_n \done_n^\top \in \bR^{n\times n}.
\]
as the mean projection matrix, which maps each row of a matrix to the row-mean of that matrix. The Frobenius norm is denoted by $\frob{\cdot}$, and the spectral norm is denoted by $\|\cdot\|_2$. We use $\mathcal{H}(\bp) = -\sum_{i=1}^v p_i \log p_i$ to denote the entropy of a probability distribution $\bp$. For a matrix $\bX$, we write $\bX_{i,:}$ to denote the $i$-th row of $\bX$. The notation $\expect{}{\cdot}$ denotes expectation over the random initialization of the model parameters. $\caL$ is the cross-entropy loss. 

\paragraph{Post-Norm Transformer Architecture.}
The complete forward pass in one Post-Norm Transformer block is
\begin{align}
    \label{eq:X_Y_definition}
    \bX_1^{l} = \rms(\bY_0^{l}), \quad
    \bY_1^{l} = \bX_1^{l} + \attn(\bX_1^{l}), \quad
    \bX_2^{l} = \rms(\bY_1^{l}), \quad
    \bY_2^{l} = \bX_2^{l} + \ffn(\bX_2^{l})
\end{align}
where $\rms(\bx) = \bx / \sqrt{\|\bx\|^2/d}$ is RMSNorm applied row-wise, $\attn$ denotes the attention branch, $\ffn$ denotes the feed-forward branch. \red{For simplicity of analysis, the RMSNorm used here is without learnable affine transformation and numerical stablizer $\epsilon$ at denominator.} Superscripts such as $\bX_1^{l}$ indicate layer index and subscripts denote which sublayer it belongs to. Subscript 1 denote it's in attention and 2 denote it's in FFN. For notation clarity, we omit superscript $l$ when we don't need to emphasize layer index. We mainly discussed a simplified single-head setting for attention and SwiGLU activation for the FFN. So the single-head attention and FFN branches are
\begin{align}
    \label{eq:attention_ffn_definition}
    \attn(\bX) = \bP\bX\bW_V\bW_O := \bP\bX\bW, \qquad \ffn(\bX)=\bigl(\silu(\bX\bW_1)\odot \bX\bW_3\bigr)\bW_2, 
\end{align}
\red{$\text{where } \bP = \softmaxrow((\bX\bW_Q)(\bX\bW_K)^\top / \sqrt{d} + \tlower{\bM}{causal})$, $(\tlower{\bM}{causal})_{i, j} = 0$ for $i \geq j$, and $-\infty$ otherwise, and \ \softmaxrow \ denotes the row-wise softmax.} The function $\silu(z) = z/(1+\exp(-z))$ is applied element-wise. \red{For notational simplicity, we replace the value-output product  $\bW_V\bW_O$ by a single independent Gaussian matrix $\bW \in \mathbb{R}^{d\times d}$ with variance $\sigma_{\bW}^2 = d\sigma_{\bW_V}^2\sigma_{\bW_O}^2$. The stated theoretical results directly extends to multi-head cases.} We use standard initialization throughout: $\bW_V, \bW_O, \bW_Q, \bW_K \in \mathbb{R}^{d\times d}$, $\bW_1, \bW_3 \in \mathbb{R}^{d \times \dff}$ have i.i.d.\ Gaussian entries with variance $1/(3d)$, and $\bW_2 \in \mathbb{R}^{\dff \times d}$ has i.i.d.\ Gaussian entries with variance $1/(3\dff)$, where $\dff$ denotes the feed-forward hidden dimension. The theory is formulated for a single-head Post-Norm block with folded value/output projection, and the experiments track the same mechanism in the full multi-head architecture used for training.

\paragraph{Token Similarity.} \label{section_token_similarity}

To track how close representations are to rank collapse, we use a single scalar observable, token similarity \citep{yu2026classic}, as measurement.

\begin{definition}
    \label{definition_ts}
    For a matrix $\bX \in \mathbb{R}^{n\times d}$, define the \textbf{mean-projection matrix} $\simmat = \frac{1}{n}\done_n\done_n^\top$, which maps each row to the row-mean of $\bX$. The \textbf{token similarity} is
    \begin{align*}
            \tsim{\bX} = \frac{\frob{\simmat\bX}^2}{\frob{\bX}^2}
    \end{align*}
\end{definition}
Here $\simmat\bX$ is the matrix whose every row equals the mean row of $\bX$. Left multiplying $\simmat$ extracts the ``shared component'' across all tokens. A simple observation is that $\tsim{\bX} = 1$ if and only if $\bX = \done \bx^T$ for some $\bx$, i.e.\ all rows are identical. \red{Therefore, we use token similarity to measure how close the hidden representation is to rank collapse, with token similarity being 1 indicate exact collapse.}

\section{Theoretical Results}


\blue{This section develops the theory in three steps, following the timeline of a collapsing run. Section~\ref{section_forward_amplification} studies the forward pass at initialization: causal attention raises token similarity at every layer, so a Post-Norm network starts training already close to collapse. Section~\ref{section_backward_repair_bottleneck} analyzes the backward pass during training: once similarity is high, the growing residual norm shrinks the gradient through RMSNorm, and the gradient that could reduce similarity decays exponentially with depth before it reaches earlier layers. Section~\ref{section_collapsed_properties} describes the properties of a collapsed network: the frequency distribution achieves the lowest loss in a collapsed network, and its parameter gradients vanish in the collapsed layers.}

\subsection{Forward Similarity Amplification}
\label{section_forward_amplification}
\blue{This subsection studies the forward pass at initialization. We show that a single attention sublayer already increases it in expectation, and we compute the expected one-step change in closed form, first for attention and then for the SwiGLU FFN.} The similarity increase is dominated by attention module, with SwiGLU FFN provides mild damping effect. Under equal-correlation assumption, we characterize the change of similarity quantitatively. By introducing the amount of attention as a scalar quantity, we give a unified characterization of similarity growth across Pre-Norm and Post-Norm. We begin by stating the equal-correlation assumption.

\blue{To get a closed-form one-step formula, we make the following approximations. First, we replace the random attention matrix $\bP$ by the prefix-averaging matrix $\bC_n$, with $(\bC_n)_{i,j} = 1/i$ for $i \ge j$ and $0$ otherwise. This is the attention matrix with uniform weights over each causal prefix, and it is close to the mean attention matrix at random initialization (see Figure~\ref{fig:init_prefix_averaging_rel_error_48l_2048} in the appendix). The surrogate forward pass is}
\begin{align}
    \label{equation_surrogate_forward}
    \blue{\bY_1 = \bX_1 + \bC_n\bX_1\bW}
\end{align}
\blue{Second, since the expectation of ratio doesn't have a closed form formula, we approximate the expected similarity by a ratio of expectations, $\expect{}{\tsim{\bY_1}} \approx \expect{}{\frobsq{\simmat\bY_1}}/\expect{}{\frobsq{\bY_1}}$.}

\blue{Third, to keep the computation tractable, we replace the full Gram matrix $\bX_1\bX_1^\top$ by a surrogate $\bR_{\mathrm{eq}}(t_1)$ determined by the single scalar $t_1 = \tsim{\bX_1}$. To see its form, note that $t_1$ depends on the Gram matrix only through the average of its entries, since $\frobsq{\bX_1} = \tr(\bX_1\bX_1^\top)$ and $\frobsq{\simmat\bX_1} = \tr(\simmat\bX_1\bX_1^\top)$ sum over all token positions; permuting tokens or rearranging pairwise correlations at fixed average leaves $t_1$ unchanged. A surrogate determined by $t_1$ alone should therefore assign one common correlation to every token pair, and the symmetric matrices of this form are exactly the linear combinations of $\bI$ and $\done_n\done_n^\top = n\simmat$. The two coefficients are fixed by two constraints: the diagonal entries equal $d$, since every row of $\bX_1$ is RMS-normalized, and $\tr(\simmat\bR_{\mathrm{eq}}(t_1))/\tr(\bR_{\mathrm{eq}}(t_1)) = t_1$. Solving them gives the surrogate below, which we call the \textbf{equal-correlation closure}.}

\begin{assumption}
\label{assumption_equal_correlation_closure}
Let $t_1 := \tsim{\bX_1}$ be the token similarity of the normalized attention input. We approximate the token Gram matrix $\bX_1\bX_1^\top$ by the one-parameter surrogate
\begin{equation}
    \label{eq:pairwise_correlation}
    \bR_{\mathrm{eq}}(t_1):= d\times \left[(1-\rho(t_1))\bI+n\rho(t_1)\,\simmat\right] \qquad \text{where} \ \rho(t_1)=(nt_1-1) / (n-1),
\end{equation}
\red{where $\tr(\simmat\bR_{\mathrm{eq}}(t_1)) / \tr(\bR_{\mathrm{eq}}(t_1)) = t_1$, and $\rho(t_1)$ captures the correlation between different tokens.}
\end{assumption}

\blue{With these approximations, we can compute the expected one-step token-similarity change after one attention sublayer in closed form. The result is summarized in Theorem~\ref{theorem_forward_amplification} below, and the proof is given in Appendix~\ref{app:forward_amplification_details}.}

\begin{theorem}
\label{theorem_forward_amplification}
Under Assumption~\ref{assumption_equal_correlation_closure} and the two approximations above, the one-step token-similarity change after one attention sublayer is
\begin{equation}
\label{eq:attention_scalar_summary}
\tlower{\Delta}{attn}(s,t_1)
:=
\expect{\bW}{\tsim{\bY_1}} - \tsim{\bX_1}
=
\frac{s f_1(t_1)}
{1 + s f_2(t_1)}
\end{equation}
where the amount of attention $s = nd^2\sigma_{\bW}^2 /\|\bX_1\|_F^2$, and 
\[
f_1(t_1):=\frac{1}{n(n-1)}(1-t_1)(n-H_n)(nt_1+1),
\qquad
f_2(t_1):=\frac{1}{n-1}\!\left[(n-H_n)\,t_1+H_n-1\right],
\]
and $H_n:=\sum_{k=1}^n 1/k$. 
\end{theorem}

\paragraph{Remark.}
Equation~\eqref{eq:attention_scalar_summary} separates the residual contribution from the attention contribution. For $t_1\in[1/n,1)$, every factor in $f_1(t_1), f_2(t_1)$ is strictly positive, so $\tlower{\Delta}{attn}(s,t_1)>0$: the one-step attention contribution is positive. Moreover, for fixed $t_1$, 
\[
\frac{\partial}{\partial s}\tlower{\Delta}{attn}(s,t_1)
=
\frac{f_1(t_1)}{(1+s f_2(t_1))^2}
>0,
\]
Therefore, $\Delta(s, t_1)$ increases monotonically with the amount of attention $s$ when $t_1 \in (0, 1)$. The factor $f_1(t_1)$ is a quadratic function of $t_1$ and is maximized at $t_1=(n-1)/(2n)\approx 1/2$. Therefore, if $sf_2(t_1)$ remains a small multiple of $n$, the largest one-step increment occurs at intermediate similarity.

\blue{Equation~\eqref{eq:attention_scalar_summary} suggests a direct test: if prefix averaging causes the similarity increase, then subtracting a multiple of $\bC_n$ from $\bP$ should reduce it. The next corollary makes this precise, and Section~\ref{section_experiments} uses it as an intervention experiment.}

\begin{corollary}
\label{theorem_attention_deescalation}
Under the same assumptions as Theorem~\ref{theorem_forward_amplification}, by replacing $\bP$ by $\bP-\alpha\bC_n$, the similarity increase is
\[
\tlower{\Delta}{de}(s,t_1) := \expect{\bW}{\tsim{\bY_1 - \alpha\bC_n\bX_1\bW}} - \tsim{\bX_1}
=
\frac{(1-\alpha)^2 sf_1(t_1)}
{1 + (1-\alpha)^2s f_2(t_1)}
\]
Hence the expected post-attention similarity is monotonically decreasing in $\alpha$, and at $\alpha=1$ the token similarity increase from attention vanishes.
\end{corollary}

In addition to attention, the transformer module also contains the FFN sublayer. FFN introduces a small damping on similarity at initialization. We adopt similar theoretical tools for analyzing FFN sublayer.

\begin{theorem}
\label{theorem_swiglu_pairwise}
\label{section_swiglu}
Let $\bY_2 = \bX_2 + \bigl(\silu(\bX_2\bW_1)\odot \bX_2\bW_3\bigr)\bW_2$ as in Equation~\ref{eq:attention_ffn_definition}, and let $t_2 := \tsim{\bX_2}$. \blue{Under Assumption~\ref{assumption_equal_correlation_closure} and the ratio-of-expectations approximation used in Theorem~\ref{theorem_forward_amplification}}, the SwiGLU sublayer induces the similarity change
\begin{equation}
\label{eq:swiglu_scalar_summary}
\tlower{\Delta}{FFN} (\xi, t_2) := \expect{\bW_1, \bW_2, \bW_3}{\tsim{\bY_2}} - \tsim{\bX_2} =
\frac{
\xi(t_2 - 1/n)(m(\rho(t_2)) - m(1))
}{
1+\xi m(1)
}
\end{equation}
\red{where $\xi = d\tlower{d}{ff}\sigma_{\bW_2}^2\sigma_{\bW_3}^2$ and $\rho(t_2) = (nt_2 - 1) / (n - 1)$ from Equation~\eqref{eq:pairwise_correlation}.} \red{Here $m(\rho)$ is the pairwise SwiGLU moment}
\[
\red{m(\rho) := \expect{}{\silu(u)\silu(v)}, \qquad (u,v)\sim\caN\!\left(0,\,\sigma_{\bW_1}^2 d\begin{pmatrix}1&\rho \\\rho &1\end{pmatrix}\right).}
\]
\end{theorem}

\paragraph{Remark.}
\red{Under the standard initialization used in this paper, a numerical evaluation of $m(\rho)$ shows that it increases across $\rho\in [0, 1)$ (see Figure~\ref{fig:swiglu_mrho_numerical} in Appendix~\ref{app:forward_amplification_details}). Therefore $m(\rho(t_2)) - m(1) < 0$ for $t_2\in[1/n,1)$ and $\tlower{\Delta}{FFN}(\xi,t_2)<0$, so SwiGLU contributes a damping effect on similarity.} We can upper-bound the magnitude of similarity change by
\[
\abs{\tlower{\Delta}{FFN} (\xi, t_2) }
=
\frac{\xi\left(t_2-\frac{1}{n}\right)\bigl(m(1)-m(\rho(t_2))\bigr)}{1+\xi m(1)}
\le \xi m(1).
\]
In the conventional setting $\xi=1/9$ and $m(1) \approx 0.10$. Thus $\xi m(1)$ is about $0.011$, so the damping effect is small. Appendix~\ref{app:forward_amplification_details} gives the full derivation and also \red{bounds the change in token similarity caused by row-wise RMS rescaling by the spread of the RMS factors across positions.}

\paragraph{Comparison with Pre-Norm.}
\red{Pre-Norm and Post-Norm share the same forward similarity amplification mechanism as Theorem~\ref{theorem_forward_amplification} describes. The difference lies in the amount of attention $s^l = nd^2 \sigma_{\bW}^2 / \frobsq{\bX_1^l}$ at each layer. In Post-Norm network, $\frobsq{\bX_1^l} = nd$ for every $l$, so $s^l = d\sigma_{\bW}^2$, a constant across all layers. In Pre-Norm network, the residual stream is not renormalized after each residual addition. Under the standard initialization-time approximation that the hidden state's Frobenius norm grows linearly with depth \citep{xiong2020layer}, the amount of attention is of order $\caO(d\sigma_{\bW}^2/l)$ and decreases as $l$ increases. Thus for deeper layers, Pre-Norm has a smaller one-step increment in token similarity than Post-Norm, so similarity grows more slowly at initialization.}


\subsection{Backward Repair Incapacity}
\label{section_backward_repair_bottleneck}
In Post-Norm, when $\tsim{\bX_k}=1$, gradient updates of $\bW, \bW_2$ increase the pre-normalization residual norm. As this norm grows, the RMSNorm Jacobian shrinks the backward gradient on the residual path. If the gradients from attention and FFN sublayers do not compensate the shrinkage from RMSNorm, the backward transport across the layer becomes contractive.  \red{For each row, RMSNorm acts independently. The backward gradient from $\bX_{k+1}$ to $\bY_k$ is} 
\begin{align*}
\red{\bigl(\pardir{\caL}{\bY_k}\bigr)_{i,:}
=
\bigl(\pardir{\caL}{\bX_{k+1}}\bigr)_{i,:}
\bJ_{\rms}((\bY_k)_{i, :}),}
\end{align*}
with
\[
\bJ_{\rms}((\bY_k)_{i, :})
:=
\frac{1}{\sqrt{\|(\bY_k)_{i, :}\|^2/d}}
\red{\bigl(\bI-\proj((\bY_k)_{i, :}^\top)\bigr),}
\qquad \text{where} \
\red{\proj(\by) := \frac{\by\by^\top}{\|\by\|^2} \text{ for column vectors } \by.}
\]
and $(\bY_k)_{i, :}$ denote the $i$th row of matrix $\bY_k$. There are two terms that can reduce gradient norm. First, the prefactor $1/\sqrt{\|(\bY_k)_{i, :}\|^2/d}$ decreases as the residual norm grows. Second, the projection $\bI-\proj(\by)$ deletes the component aligned with $\by$. Here, we discuss the prefactor part under the simplified setting of exact collapse ($\tsim{\bY_k} = \tsim{\bX_k} = 1$ for $k = 1, 2$). 

\blue{The RMSNorm Jacobian is only half of the backward picture: the gradient also flows through the attention and FFN branches, which could make it larger and cancel the shrinkage. To compare the two, we need to measure how much each sublayer changes the gradient norm. The factor defined below does this, and Section~\ref{section_experiments} tracks it during collapsing runs.}

\begin{definition}
    \red{
    \label{assumption_sublayer_amplification}
    We define the sublayer gradient contribution factor $\alpha_k^l$ as
    \begin{align*}
        \alpha_k^l = \frob{\pardir{\caL}{\bX_{k}^l}} / \frob{\pardir{\caL}{\bY_{k}^l}}
    \end{align*}
    where subscript and superscript notation follows from Equation~\ref{eq:X_Y_definition}.
    }
\end{definition}
\blue{Using this factor, the next theorem derives the relationship between the norm of the backward gradients.}

\begin{theorem}
\label{theorem_iterative_bound_on_gradient_norm}
\red{
When $\tsim{\bY^l_k} = \tsim{\bX^l_k} = 1$ for $l \in \{1, \ldots, \ldepth\}$ and $k \in \{1, 2\}$, we have $\bY^l_k = \done_n(\by^l_k)^\top$ for some $\by^l_k \in \mathbb{R}^d$. Furthermore, 
\[
\frob{\pardir{\caL}{\bX^l_k}}
\le
c(\by^l_k,\alpha^l_k)
\frob{\pardir{\caL}{\bX^l_{k+1}}},
\qquad \text{where} \ \ c(\by,\alpha) := \alpha\frac{\sqrt{d}}{\twonorm{\by}}.
\]
with $\bX^l_{3} := \bX^{l+1}_{1}$. Hence whenever $\twonormsq{\by^l_k}/d > (\alpha^l_k)^2$, the corresponding factor $c(\by_k^l,\alpha_k^l)$ falls below one, and the backward gradient contracts from $\bX^l_{k+1}$ to $\bX^l_k$.
}
\end{theorem}

\paragraph{Remark.} \red{This bound can be combined for $k = 1, 2$ to obtain} 
\begin{align*}
    \red{\frob{\pardir{\caL}{\bX_1^{l}}} \leq c(\by_1^{l}, \alpha_1^{l})c(\by_2^{l}, \alpha_2^{l})\frob{\pardir{\caL}{\bX_1^{l+1}}}.}
\end{align*}

\red{If $\tsim{\bX_k^{\tilde l}} = \tsim{\bY_k^{\tilde l}}= 1$ for $k \geq 1, \tilde l \geq l$, and each $c(\by, \alpha) < \gamma$ for some constant $\gamma < 1$, then the gradient reaching earlier layers becomes exponentially small. This requires the residual-stream norm $\twonorm{\by}$ to grow under exact collapse, which holds for the $\bW_O$ update in attention and the $\bW_2$ update in SwiGLU FFN(see Appendix~\ref{app:backward_bottleneck_details} for detail). Section~\ref{sec:experiments} measures both $\alpha_k^l$ and the RMSNorm factor $\sqrt{d}/\twonorm{\by_k^l}$.}

\paragraph{Comparison with Pre-Norm.}
\red{For Pre-Norm, the backward gradient from $\bX_{k+1}$ to $\bX_k$ is
\begin{align*}
\left(\frac{\partial \caL}{\partial \bX_k}\right)_{i,:}
=
\left(\frac{\partial \caL}{\partial \bX_{k+1}}\right)_{i,:}
+
\left(
\mathrm{Grad}_S\!\left(\rms(\bX_k), \frac{\partial \caL}{\partial \bX_{k+1}}\right)
\right)_{i,:}
\bJ_{\rms}((\bX_k)_{i,:}).
\end{align*}
where $\mathrm{Grad}_S$ is the gradient through the sublayer. The major difference with Post-Norm lies in the location of $\tlower{\bJ}{\rms}$ term. In Pre-Norm, $\tlower{\bJ}{\rms}$ multiplies only the gradient through the sublayer, while the skip-connection gradient stays unchanged. In Post-Norm, the RMS Jacobian multiplies both the sublayer and the skip-connection gradient. As a result, when the residual norm grows, Post-Norm shrinks the whole backward signal, whereas Pre-Norm shrinks only part of it. Therefore, Pre-Norm doesn't suffer from the gradient contraction mechanism shown in Theorem~\ref{theorem_iterative_bound_on_gradient_norm}.}


\subsection{Properties of a Collapsed Network}
\label{section_collapsed_properties}
\label{section_frequency_distribution}

A collapsed network has two properties that hinder optimization. First, the best achievable probability distribution is the \textbf{frequency distribution}. Second, once the output matches the frequency distribution, \red{the parameter gradients in collapsed layers} completely vanish. \blue{When the output distribution is the same at every position, the cross-entropy loss is minimized by predicting each label in proportion to its count in the target sequence. This optimum is the frequency distribution, defined formally below.}

\begin{definition} \label{definition_frequency_distribution}
    For a sequence of input tokens $\bt \in (\bbZ^+)^n$ with target labels $\mathbf{y}\in (\bbZ^+)^n$, both $\bt$ and $\mathbf{y}$ contain vocabulary indices, i.e., $\bt_i \in [v], \by_i \in [v]$. Let $c_i = \sum_{j=1}^n \mathbb{I}[y_j = i]$ be the count of label $i$ in $\mathbf{y}$. The \textbf{frequency distribution} and \textbf{frequency loss} are
\begin{equation*}
\tlower{\bp}{freq}(\by) =  \left [ \frac{c_i}{n}\right]_{i=1}^v \in \mathbb{R}^v, \qquad \tlower{\caL}{freq}(\mathbf{y}) = \mathcal{H}(\tlower{\bp}{freq}(\by)) = -\sum_{i=1}^v (\tlower{\bp}{freq}(\by))_i\log (\tlower{\bp}{freq}(\by))_i.
\end{equation*}
Here $\tlower{\bp}{freq}(\by)$ is a valid probability distribution because $\sum_{i=1}^v c_i = n$.
\end{definition}

\blue{The next theorem shows that a near-collapsed network cannot achieve a loss much lower than frequency loss, and that once the output matches the frequency distribution, the parameter gradients in all collapsed layers vanish exactly.}

\begin{theorem}\label{theorem_frequency_distribution}
Let $\bX^H \in \bbR^{n\times d}$ denote the hidden representation before the LM-head for input $\bt$, and let $\wlm \in \bbR^{d\times v}$ be the LM-head matrix. Define
\[
\quad \hat\bP = \softmaxrow(F_\theta(\mathbf{t})), \qquad \tlower{\caL}{CE}(F_\theta(\mathbf{t}), \mathbf{y}) = -\frac{1}{n}\sum_{i = 1}^n \log \hat\bP_{i, y_i}.
\]
where $F_{\theta}(\bt) = \bX^H\wlm$ is the neural network parameterized by $\theta$. Then:

\begin{enumerate}[(i)]
    \item \textbf{Lower Bound.} \label{theorem_collapse_lower_bound}\label{theorem_approximate_lower_bound_on_loss} If $\|\bX^H - \simmat\bX^H\|_F \leq \epsilon$, then
\begin{align*}
    \tlower{\caL}{CE}(F_\theta(\mathbf{t}), \mathbf{y})
    \geq \tlower{\caL}{freq}(\mathbf{y}) - \red{\frac{2}{\sqrt{n}}}\|\wlm\|_2 \epsilon.
\end{align*}

    \item \textbf{Gradient Vanishing.} \label{theorem_complete_gradient_vanishing} If the model output equals to $\tlower{\bp}{freq}(\by)$ at every position, then \red{for all sublayer such that $\tsim{\bX^l_k} = 1$,} the gradients with respect to the parameter matrices in these layers are zero.
\end{enumerate}

\end{theorem}

\paragraph{Remark.}
The theorem shows two properties of the collapsed network. Part (i) shows that the network can only produce a suboptimal solution under rank collapse. Part (ii) shows that, \red{once sublayer hidden state $\bX^l_k$ is collapsed} and the output probability matches the frequency distribution, \red{the gradient with respect to every parameter matrix in that sublayer and in every later layer $k>l$ becomes zero. Note that in reality, the frequency distribution may be unattainable when there is 0 count for some labels.} Additionally, if the gradients with respect to earlier layers that are not collapse ($\tsim{\bX} < 1$) also vanish, then the network reaches a near-stationary point. The detailed proof is at Appendix~\ref{app:collapsed_properties_details}.

\section{Experiments}\label{section_experiments}


\label{sec:experiments}
\red{The experiments are organized in the same order as the theory. We first test the initialization-time similarity increase by comparing the measured one-step similarity change with the surrogate predictions from Theorem~\ref{theorem_forward_amplification} and Theorem~\ref{theorem_swiglu_pairwise}, and by removing the prefix-averaging component as in Corollary~\ref{theorem_attention_deescalation}. We then test the training-time quantities in Theorem~\ref{theorem_iterative_bound_on_gradient_norm} by measuring sublayer gradient contributions, residual-stream norm growth, pre/post-RMS gradient contraction, and per-layer gradient norms. Finally, we compare the loss of collapsed runs with the frequency loss in Theorem~\ref{theorem_frequency_distribution}.}

\paragraph{Setup.}
Our base model is a 48-layer Llama-2 \citep{touvron2023llama} style decoder-only Transformer with model dimension $d=512$, feed-forward dimension $d_{\mathrm{ff}}=1536$, 4 attention heads, SwiGLU feed-forward blocks, and RMSNorm normalization, for a total of about 180M parameters. We train on C4 \citep{raffel2020exploring} with sequence length 2048, micro-batch size 4, and gradient accumulation 64, giving an effective batch size of 256. Optimization uses AdamW \citep{loshchilov2017decoupled} with $\beta_1=0.9$, $\beta_2=0.95$, weight decay 0.1, cosine learning-rate decay, 2000 warmup steps, \red{40000 total steps,} and gradient clipping at 1.0. \red{In all training runs, RMSNorm uses its standard implementation with learned gain and numerical stabilizer $\epsilon$.} \blue{The training runs span maximum learning rates from $6\times10^{-4}$ to $1.8\times10^{-3}$.} When comparing Post-Norm and Pre-Norm, all architectural and optimization settings are matched except for normalization placement.

\subsection{Stage I: Forward Similarity Amplification}
\red{Stage I provides a quantitative analysis of forward similarity amplification mechanism at initialization: attention increases token similarity, while SwiGLU contributes only a small damping effect. We validate this analysis in two ways. First, we compare the measured one-step similarity change with the predictions from Theorem~\ref{theorem_forward_amplification} and Theorem~\ref{theorem_swiglu_pairwise}. Second, we test the contribution of the prefix-averaging part of attention by removing it as in Corollary~\ref{theorem_attention_deescalation}.}

\paragraph{Quantitative Measurement of Forward Amplification.}
Figure~\ref{figure_similarity_graph} compares the measured one-step token-similarity increment with the theoretical predictions. To provide a comprehensive evaluation, we sweep the subayer initialization variance, so that the same 48-layer network covers a wider range of similarity growth. 

The left column shows the contribution from attention. For both Post-Norm and Pre-Norm, attention produces a positive one-step increment with strong dependence on layer index. The dashed theorem curves capture the sign and overall scale of this increment, \red{although a visible mismatch remains.}

The right column shows the SwiGLU contribution under the same variance sweep. SwiGLU produces a small damping effect whose magnitude is much smaller than the attention contribution, so the initial similarity increase is dominated by attention rather than FFN. \red{The Post-Norm and Pre-Norm curves have similar overall shape, but Post-Norm saturates more quickly: its token similarity rises rapidly and the one-step increment drops to near zero, whereas in Pre-Norm the growth is more gradual and the increment decays more slowly with depth.}
\begin{figure}[H]
    \centering
    \includegraphics[width=0.65\textwidth]{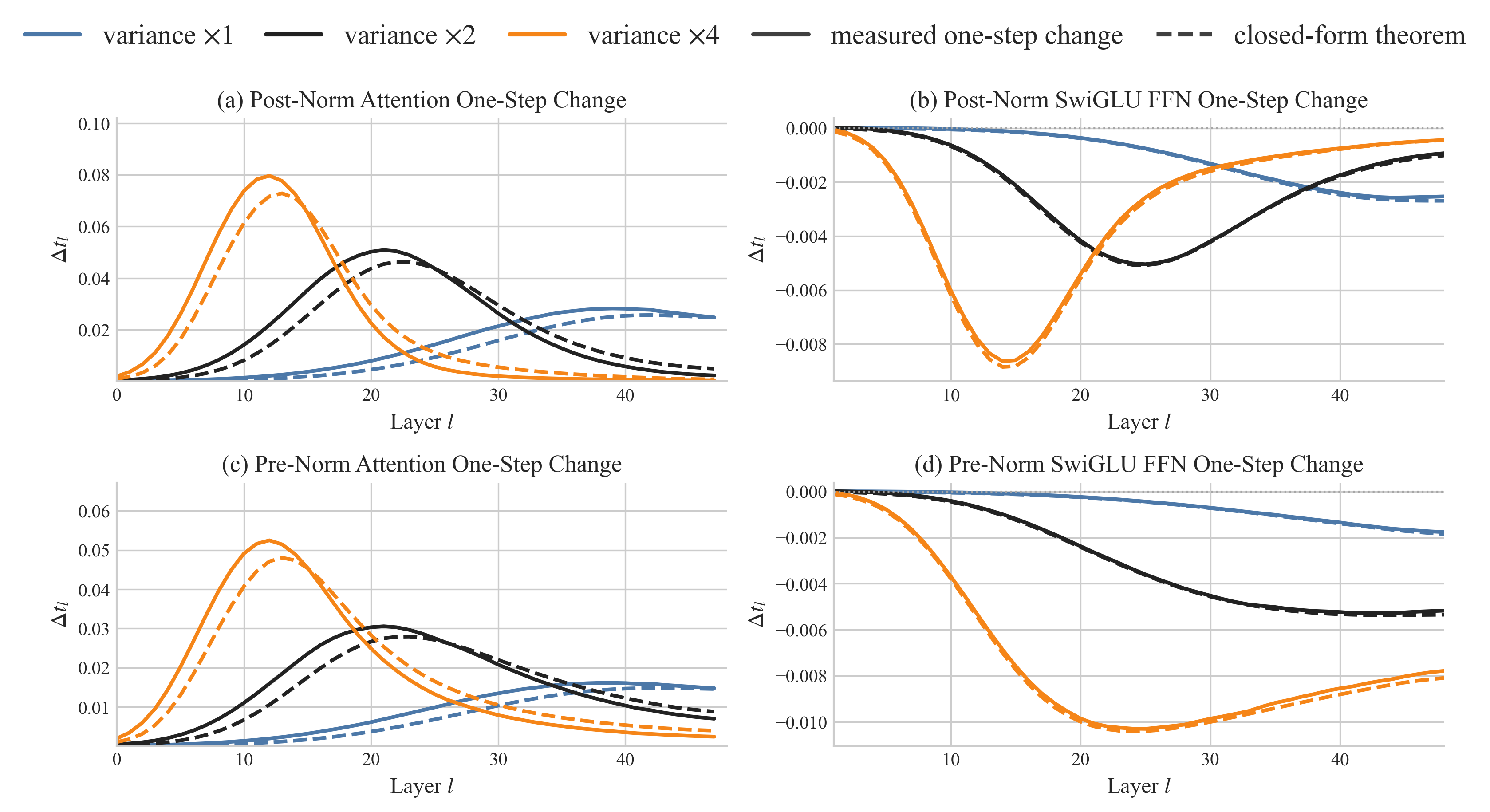}
    \caption{Quantitative evidence for forward similarity amplification on the 48-layer training-matched stack under sublayer initialization variance sweeps ($\times 1$, $\times 2$, $\times 4$).  Top row: Post-Norm. Bottom row: Pre-Norm. Left column: the attention branch produces a positive one-step increment whose magnitude is tracked well by Theorem~\ref{theorem_forward_amplification}. Right column: the SwiGLU feed-forward branch contributes a smaller negative correction tracked by Theorem~\ref{theorem_swiglu_pairwise}.}
    \label{figure_similarity_graph}
\end{figure}

\paragraph{Removing Prefix-Averaging Component.}\begin{wrapfigure}{r}{0.32\textwidth}
    \centering
    \vspace{-8pt}
    \includegraphics[width=0.32\textwidth]{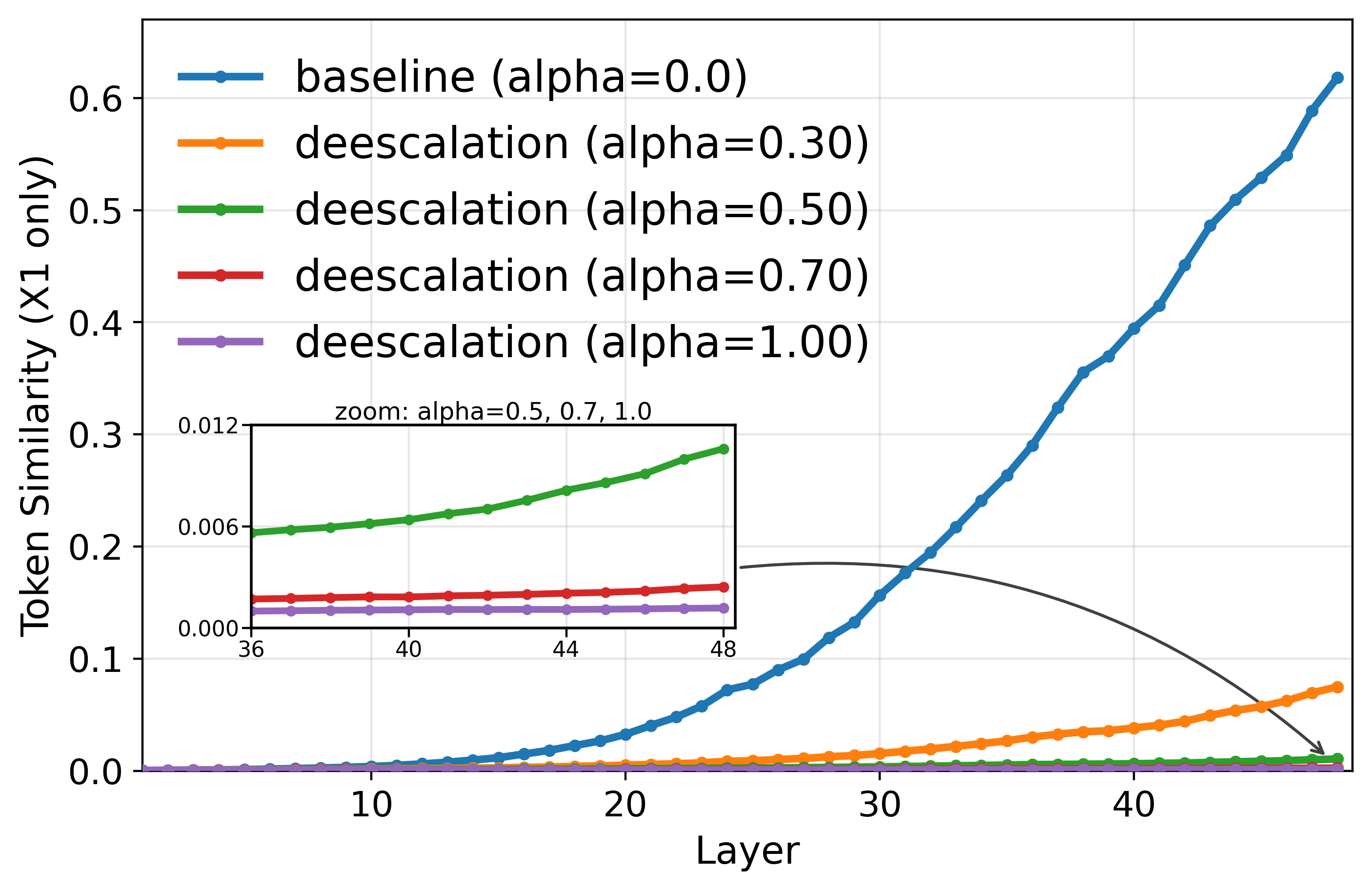}
    \caption{Removing Prefix-Averaging Component suppresses initialization-time token-similarity growth.}
    \label{figure_init_evidence_experiments}
    \vspace{-10pt}
\end{wrapfigure}

\red{Curve matching alone does not identify which part of attention drives the effect. To test this more directly, we subtract the prefix-averaging component $\alpha\bC_n$ from the attention matrix at initialization, where $\bC_n$ is the causal prefix-average matrix with $(\bC_n)_{i, j} = 1/i$ for $i \geq j$, and 0 otherwise.} Figure~\ref{figure_init_evidence_experiments} shows the result. As $\alpha$ increases from $0$ to $1$, token-similarity growth is uniformly suppressed, and $\alpha=1$ removes most of the increment. \red{This result shows that the prefix-averaging component of causal attention is a main driver of the similarity increase.} 

These results show that Post-Norm network is already pushed toward higher token similarity at initialization, but still does not explain why high token similarity state persists during training. If sufficiently strong gradient reach earlier layers, training can still restore token diversity. We next test why the backward signal to earlier layers becomes weak.


\subsection{Stage II: Backward Repair Incapacity}
\red{This section tracks three quantities used in Theorem~\ref{theorem_iterative_bound_on_gradient_norm}: $\alpha_k^l$, $\sqrt{d}/\|\by_k^l\|_2$, and $c(\by_k^l,\alpha_k^l)$, in one collapsed Post-Norm run.}  We call the optimizer steps where last-layer token similarity and training loss rise sharply the \textbf{transition window}. \blue{Post-Norm collapse depends on learning rate: at learning rate $6\times10^{-4}$ the same model trains stably for 40000 steps without collapse (Appendix~\ref{app:non_collapsed_postnorm_control}), while at
   $8\times10^{-4}$ it collapses within a few thousand steps. We therefore diagnose a collapsing run at learning rate $8\times10^{-4}$, whose transition occurs at approximately step 2644. The results for other learning rates can be found at Appendix~\ref{app:masked_lr_sweep_backward_incapacity}} \red{We first measure $\alpha_k^l$ over the full training run, then zoom into the transition window for $\sqrt{d}/\|\by_k^l\|_2$ and $c(\by_k^l,\alpha_k^l)$, and finally measure the gradient norms of every layer.}

\paragraph{Measuring Sublayer Gradient Contribution.}
\red{We first measure the attention and FFN sublayer gradient contribution factors $\alpha_1^l$ and $\alpha_2^l$ over the full training trajectory. Figure~\ref{fig:stage2_sublayer_contribution} plots $\alpha_1^l$ and $\alpha_2^l$ over the full training run. Most of the gradient contribution increase near the transition window, but the contribution factor stays below 2.5. If residual-norm growth is large enough to outweigh this increase in $\alpha_k^l$, then $c(\by_k^l,\alpha_k^l)$ should still decrease, which we test next.}

\begin{figure}[H]
    \centering
    \includegraphics[width=0.8\textwidth]{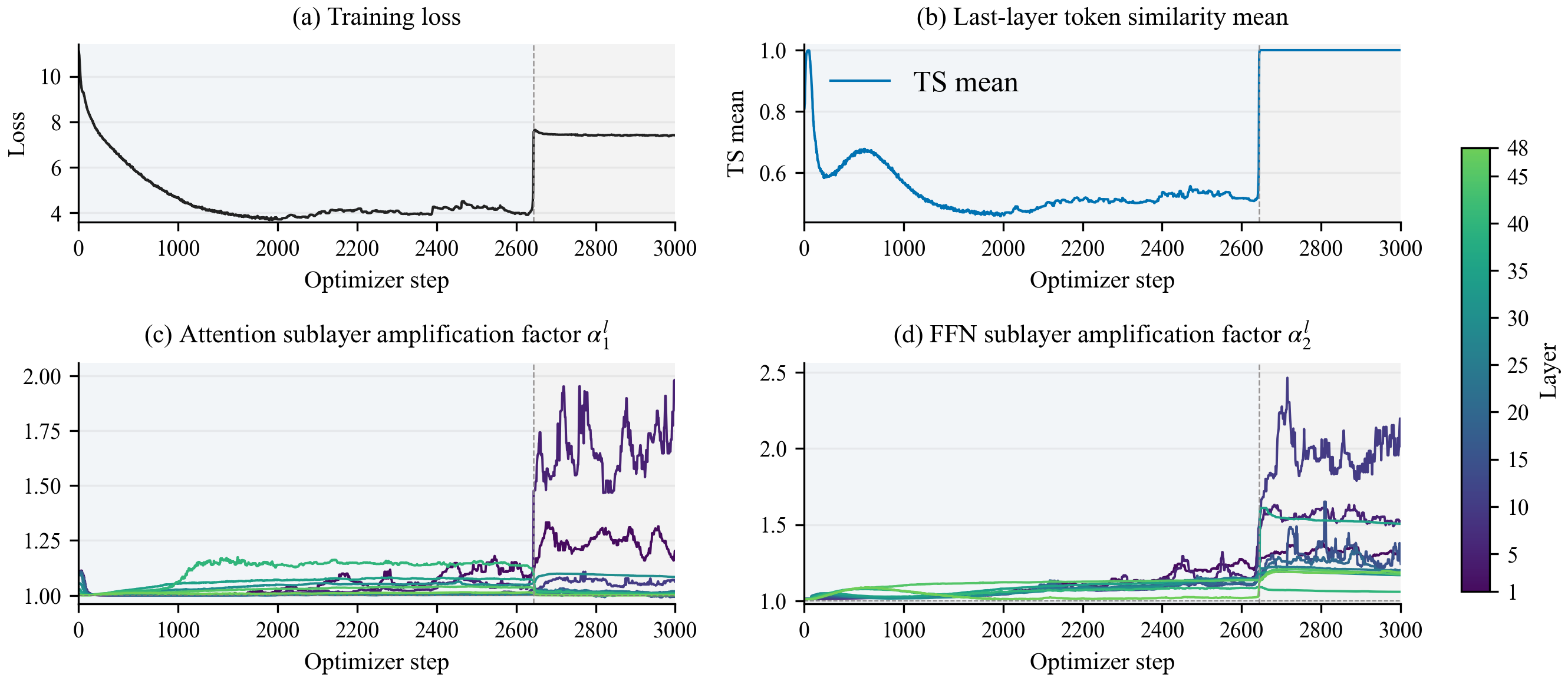}
    \caption{Sublayer gradient contribution in the collapsing Post-Norm run at LR=$8\times10^{-4}$. Dashed vertical line: transition step $\approx 2644$. \textbf{(a)} Training loss. \textbf{(b)} Last-layer token-similarity mean. \textbf{(c)} Attention gradient contribution across layers 1, 5, 10, 15, 20, 25, 30, 35, 40, 45, and 48. \textbf{(d)} FFN gradient contribution across the same layers.}
    \label{fig:stage2_sublayer_contribution}
\end{figure}

\paragraph{Residual Norm Growth and Gradient Contraction.}
\red{We now test the other two terms, $\sqrt{d}/\|\by_k^l\|_2$ and $c(\by_k^l,\alpha_k^l)$ near the transition window. The term $\sqrt{d}/\|\by_k^l\|_2$ measures the shrinkage coming from the pre-RMS residual norm alone. The term $c(\by_k^l,\alpha_k^l)$ combines this RMSNorm factor with the measured sublayer gradient contribution $\alpha_k^l$, and directly determines how much gradient norm changes after backpropagating through a sublayer. Figure~\ref{figure_collapse_gradients_experiments} reports these two quantities near the transition window. In panels (a) and (c), the curves for higher-similarity layers bend downward sharply in the transition window, showing a rapid drop in $\sqrt{d}/\|\by_k^l\|_2$ at both the attention sublayer and the FFN sublayer. In panels (b) and (d), the same layers show a matching drop in $c(\by_k^l,\alpha_k^l)$, and most of the curves cross below 1 shortly after the transition begins. The matching drops in panels (a)--(d) show that the decrease in $\sqrt{d}/\|\by_k^l\|_2$ is large enough to outweigh the increase in $\alpha_k^l$ in most layers. In Theorem~\ref{theorem_iterative_bound_on_gradient_norm}, the condition $c(\by_k^l,\alpha_k^l)<1$ means that the corresponding sublayer shrinks the gradient norm. The attention layer-1 curve is a visible exception: there the rise in $\alpha_1^l$ outweighs the RMSNorm shrinkage, so $c(\by_1^l,\alpha_1^l)$ does not show the same drop.}

\begin{figure}[H]
    \centering
    \includegraphics[width=0.8\textwidth]{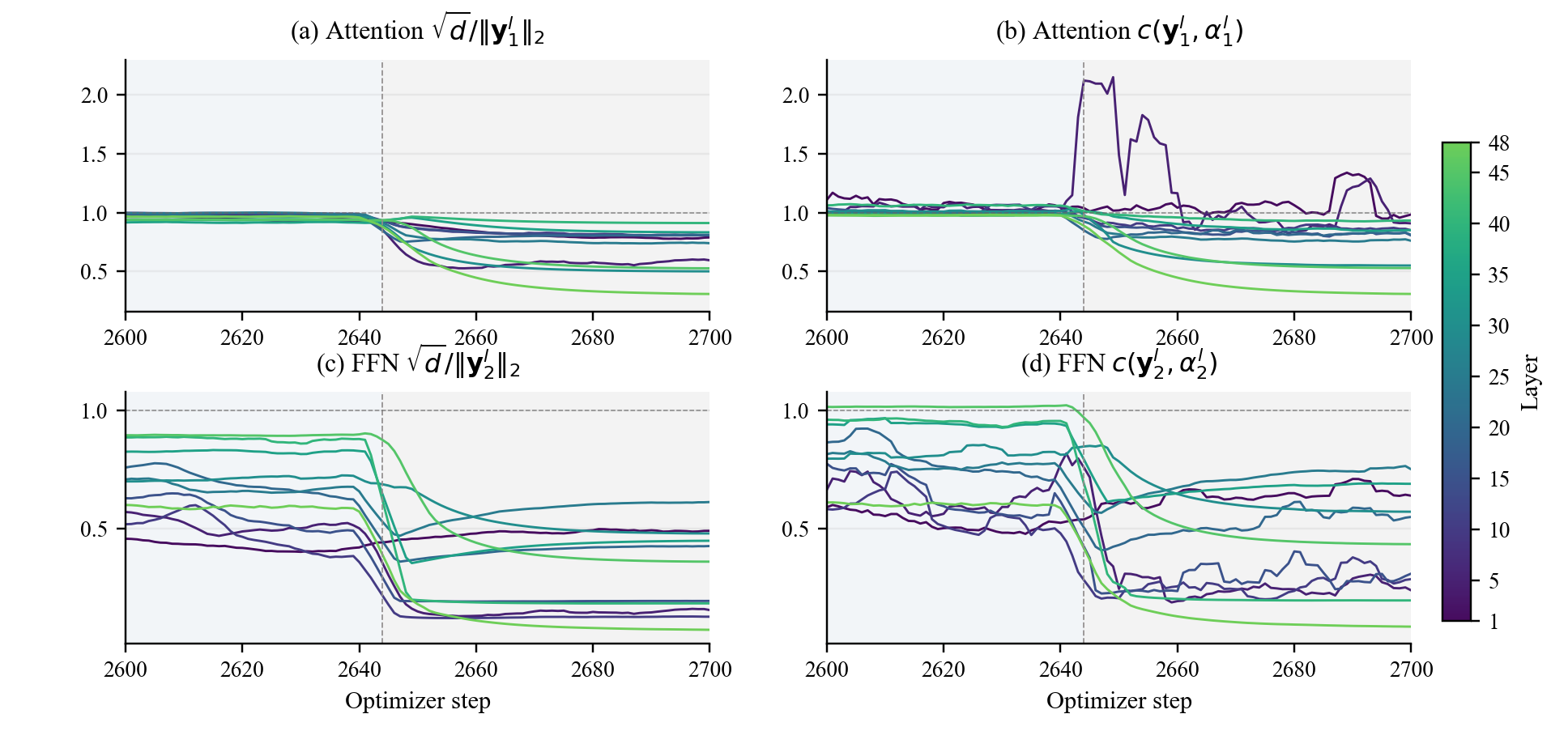}
    \caption{$\sqrt{d}/\|\by_k^l\|_2$ and $c(\by_k^l,\alpha_k^l)$ in the transition window of the Post-Norm run at LR=$8\times10^{-4}$. Dashed vertical line: transition step $\approx 2644$; blue/red shading: pre/post-transition windows. \textbf{(a)} Attention $\sqrt{d}/\|\by_1^l\|_2$. \textbf{(b)} Attention $c(\by_1^l,\alpha_1^l)$. \textbf{(c)} FFN $\sqrt{d}/\|\by_2^l\|_2$. \textbf{(d)} FFN $c(\by_2^l,\alpha_2^l)$. Measured layers are 1, 5, 10, 15, 20, 25, 30, 35, 40, 45, and 48.}
    \label{figure_collapse_gradients_experiments}
\end{figure}

\begin{figure}[H]
    \centering
    \includegraphics[width=0.7\textwidth]{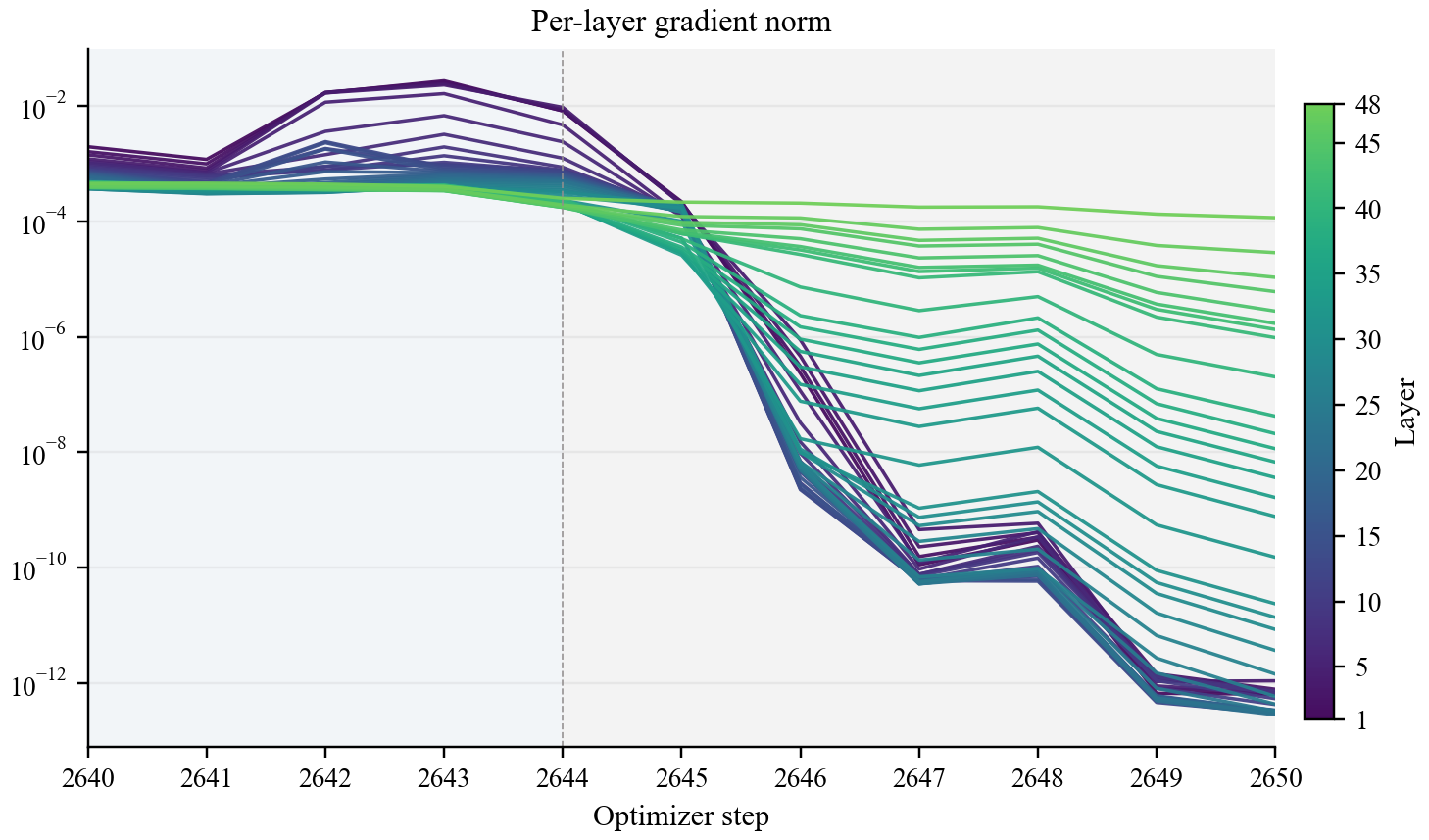}
    \caption{Each curve covers optimizer steps 2640--2650. Gradients to early layers drop sharply, while later-layer gradients remain relatively stable, consistent with geometric compounding of per-block RMSNorm contraction.}
    \label{fig:stage2_gradient_norm_during_collapse}
\end{figure}
\paragraph{Gradient Vanishing in Earlier Layers.} \red{We next test whether $c(\by_k^l,\alpha_k^l)<1$ in many successive layers results in much smaller gradient norms in earlier layers.  Figure~\ref{fig:stage2_gradient_norm_during_collapse} plots the per-layer gradient norms over optimizer steps 2640--2650 and shows that gradient norms become much smaller in earlier layers. During the transition window, gradient norms in early layers decrease by orders of magnitude, while gradient norms in later layers change much less. This leaves little gradient in earlier layers to reduce token similarity, making the high-similarity state hard to reverse.}

\subsection{Collapsed-State Loss Characterization}
\red{
We compare the training loss after transition window with the frequency loss. Theorem~\ref{theorem_frequency_distribution} shows that under exact collapse, the frequency loss is the smallest achievable cross-entropy loss.  Figure~\ref{figure_collapsed_state_loss_verification} tests whether the loss after collapse approaches the frequency loss predicted by Theorem~\ref{theorem_frequency_distribution}.
}
\red{
To extend the frequency-loss reference in Theorem~\ref{theorem_frequency_distribution} to multiple sequences, we consider two approaches. The first approach is to aggregate the labels from multiple sequences into a single label collection and computes the frequency loss from the resulting label frequencies. The second approach is to compute the frequency loss for each sequence separately and then averages these losses. In our experiments, the first one matches the post-collapse training loss more closely, so we use it to compute the frequency-loss curves in Figure~\ref{figure_collapsed_state_loss_verification}.
}

\red{
The left panel shows no-warmup training under three data scopes: full training, a single batch, and a single sequence. We plot each training loss curve together with the frequency loss curve  computed from the corresponding label collection, extending the frequency-distribution construction in Theorem~\ref{theorem_frequency_distribution} with the aggregated-label approach described above. These label collections contain 500M tokens sampled from the dataset for the full-training reference, 0.5M tokens for the single-batch reference, and 2048 tokens for the single-sequence reference. Different label scopes result in different empirical label frequencies and therefore different frequency-loss references. 
}
\red{
The right panel shows three warmup Post-Norm runs with learning rates $1.2\times 10^{-3}$, $1.5\times 10^{-3}$, and $1.8\times 10^{-3}$. The frequency loss reference is obtained using 500M tokens sampled from the dataset as in the left panel. All three runs collapse, and after the transition window the loss in each run stays near the frequency-loss curve. These results match Theorem~\ref{theorem_frequency_distribution}: after collapse, the training loss stays near the frequency loss.}

\begin{figure}[H]
    \centering
    \includegraphics[width=0.8\textwidth]{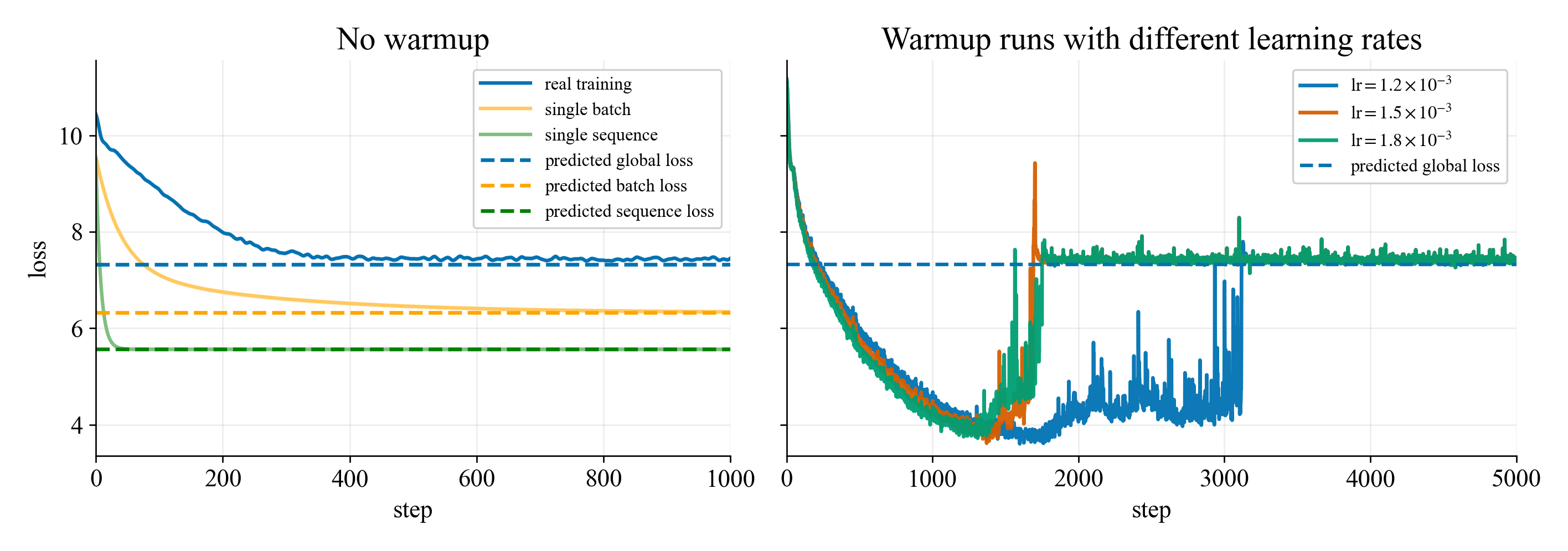}
    \caption{Training loss compared with frequency-loss curves. Left: in the no-warmup setting, real training, a single batch, and a single sequence each approach the frequency loss computed from the same source of labels. Right: in warmup Post-Norm runs with learning rates $1.2\times 10^{-3}$, $1.5\times 10^{-3}$, and $1.8\times 10^{-3}$, the loss rises sharply at different optimizer steps, but all three runs then stay near the frequency loss computed from all target labels in the training stream.}
    \label{figure_collapsed_state_loss_verification}
\end{figure}

\section{Limitations}


\red{The paper explains why token similarity increases at initialization, why gradients to layers with smaller indices become weak during training, and why collapsed runs stay near the frequency loss while gradients vanish in collapsed layers. However, it does not explain what causes the sharp rise in last-layer token similarity and training loss at a particular optimizer step. How this phenomenon occurs remains as future work.}




\section{Conclusion}
\red{The theory and experiments support the same picture of Post-Norm collapse. At initialization, causal attention increases token similarity, while the feed-forward correction is much smaller. During training, growth of the pre-RMS residual norm reduces the gradient that reaches layers with smaller indices, making the high-similarity state difficult to repair. Once collapse occurs, the training loss stays near the frequency loss. Together, these results show how Post-Norm is pushed toward collapse and why training has difficulty reversing it once collapse begins.}

\bibliographystyle{tmlr}
\bibliography{sections/references}

@article{poole2016exponential,
  title={Exponential expressivity in deep neural networks through transient chaos},
  author={Poole, Ben and Lahiri, Subhaneil and Raghu, Maithra and Sohl-Dickstein, Jascha and Ganguli, Surya},
  journal={Advances in Neural Information Processing Systems},
  volume={29},
  year={2016}
}

@article{schoenholz2016deep,
  title={Deep information propagation},
  author={Schoenholz, Samuel S and Gilmer, Justin and Ganguli, Surya and Sohl-Dickstein, Jascha},
  journal={arXiv preprint arXiv:1611.01232},
  year={2016}
}

@inproceedings{glorot2010understanding,
  title={Understanding the difficulty of training deep feedforward neural networks},
  author={Glorot, Xavier and Bengio, Yoshua},
  booktitle={Proceedings of the Thirteenth International Conference on Artificial Intelligence and Statistics},
  pages={249--256},
  year={2010},
  organization={JMLR Workshop and Conference Proceedings}
}

@inproceedings{he2015delving,
  title={Delving deep into rectifiers: Surpassing human-level performance on imagenet classification},
  author={He, Kaiming and Zhang, Xiangyu and Ren, Shaoqing and Sun, Jian},
  booktitle={Proceedings of the IEEE International Conference on Computer Vision},
  pages={1026--1034},
  year={2015}
}

@book{neal2012bayesian,
  title={Bayesian learning for neural networks},
  author={Neal, Radford M},
  volume={118},
  year={2012},
  publisher={Springer Science \& Business Media}
}

@article{lee2018gaussian,
  title={Deep Neural Networks as Gaussian Processes},
  author={Lee, Jaehoon and Bahri, Yasaman and Novak, Roman and Schoenholz, Samuel S. and Pennington, Jeffrey and Sohl-Dickstein, Jascha},
  journal={arXiv preprint arXiv:1711.00165},
  year={2018}
}

@article{matthews2018gaussian,
  title={Gaussian Process Behaviour in Wide Deep Neural Networks},
  author={Matthews, Alexander G. de G. and Rowland, Mark and Hron, Jiri and Turner, Richard E. and Ghahramani, Zoubin},
  journal={arXiv preprint arXiv:1804.11271},
  year={2018}
}

@article{pennington2017resurrecting,
  title={Resurrecting the Sigmoid in Deep Learning Through Dynamical Isometry: Theory and Practice},
  author={Pennington, Jeffrey and Schoenholz, Samuel S. and Ganguli, Surya},
  journal={arXiv preprint arXiv:1711.04735},
  year={2017}
}

@article{xiao2018dynamical,
  title={Dynamical Isometry and a Mean Field Theory of CNNs: How to Train 10,000-Layer Vanilla Convolutional Neural Networks},
  author={Xiao, Lechao and Bahri, Yasaman and Sohl-Dickstein, Jascha and Schoenholz, Samuel S. and Pennington, Jeffrey},
  journal={arXiv preprint arXiv:1806.05393},
  year={2018}
}

@article{chen2018dynamical,
  title={Dynamical Isometry and a Mean Field Theory of RNNs: Gating Enables Signal Propagation in Recurrent Neural Networks},
  author={Chen, Minmin and Pennington, Jeffrey and Schoenholz, Samuel S.},
  journal={arXiv preprint arXiv:1806.05394},
  year={2018}
}

@article{yu2026classic,
  title={Why" Classic" Transformers Are Shallow and A Depth-Enabling Technique},
  author={Yu, Yueyao and Zhang, Yin},
  journal={Journal of Machine Learning Research},
  volume={27},
  number={64},
  pages={1--32},
  year={2026}
}

@article{chen2026from,
  title={From condensation to rank collapse: A two-stage analysis of transformer training dynamics},
  author={Chen, Zheng-An and Luo, Tao},
  journal={Advances in Neural Information Processing Systems},
  volume={38},
  pages={41563--41598},
  year={2026}
}

@inproceedings{dong2021attention,
  title={Attention is not all you need: Pure attention loses rank doubly exponentially with depth},
  author={Dong, Yihe and Cordonnier, Jean-Baptiste and Loukas, Andreas},
  booktitle={International Conference on Machine Learning},
  pages={2793--2803},
  year={2021},
  organization={PMLR}
}

@article{vaswani2017attention,
  title={Attention is all you need},
  author={Vaswani, Ashish and Shazeer, Noam and Parmar, Niki and Uszkoreit, Jakob and Jones, Llion and Gomez, Aidan N and Kaiser, {\L}ukasz and Polosukhin, Illia},
  journal={Advances in Neural Information Processing Systems},
  volume={30},
  year={2017}
}

@article{ba2016layer,
  title={Layer normalization},
  author={Ba, Jimmy Lei and Kiros, Jamie Ryan and Hinton, Geoffrey E},
  journal={arXiv preprint arXiv:1607.06450},
  year={2016}
}

@inproceedings{he2016deep,
  title={Deep residual learning for image recognition},
  author={He, Kaiming and Zhang, Xiangyu and Ren, Shaoqing and Sun, Jian},
  booktitle={Proceedings of the IEEE Conference on Computer Vision and Pattern Recognition},
  pages={770--778},
  year={2016}
}

@article{saada2024mind,
  title={Mind the gap: a spectral analysis of rank collapse and signal propagation in attention layers},
  author={Saada, Thiziri Nait and Naderi, Alireza and Tanner, Jared},
  journal={arXiv preprint arXiv:2410.07799},
  year={2024}
}

@article{noci2022signal,
  title={Signal propagation in transformers: Theoretical perspectives and the role of rank collapse},
  author={Noci, Lorenzo and Anagnostidis, Sotiris and Biggio, Luca and Orvieto, Antonio and Singh, Sidak Pal and Lucchi, Aurelien},
  journal={Advances in Neural Information Processing Systems},
  volume={35},
  pages={27198--27211},
  year={2022}
}

@inproceedings{liu2020understanding,
  title={Understanding the difficulty of training transformers},
  author={Liu, Liyuan and Liu, Xiaodong and Gao, Jianfeng and Chen, Weizhu and Han, Jiawei},
  booktitle={Proceedings of the 2020 Conference on Empirical Methods in Natural Language Processing (EMNLP)},
  pages={5747--5763},
  year={2020}
}

@inproceedings{xiong2020layer,
  title={On layer normalization in the transformer architecture},
  author={Xiong, Ruibin and Yang, Yunchang and He, Di and Zheng, Kai and Zheng, Shuxin and Xing, Chen and Zhang, Huishuai and Lan, Yanyan and Wang, Liwei and Liu, Tieyan},
  booktitle={International Conference on Machine Learning},
  pages={10524--10533},
  year={2020},
  organization={PMLR}
}

@article{nguyen2019transformers,
  title={Transformers without Tears: Improving the Normalization of Self-Attention},
  author={Nguyen, Toan Q. and Salazar, Julian},
  journal={arXiv preprint arXiv:1910.05895},
  year={2019}
}

@inproceedings{yan2022addressing,
  title={Addressing token uniformity in transformers via singular value transformation},
  author={Yan, Hanqi and Gui, Lin and Li, Wenjie and He, Yulan},
  booktitle={Uncertainty in Artificial Intelligence},
  pages={2181--2191},
  year={2022},
  organization={PMLR}
}

@inproceedings{wang2020improving,
  title={Improving neural language generation with spectrum control},
  author={Wang, Lingxiao and Huang, Jing and Huang, Kevin and Hu, Ziniu and Wang, Guangtao and Gu, Quanquan},
  booktitle={International Conference on Learning Representations},
  year={2020}
}

@inproceedings{zhai2023stabilizing,
  title={Stabilizing transformer training by preventing attention entropy collapse},
  author={Zhai, Shuangfei and Likhomanenko, Tatiana and Littwin, Etai and Busbridge, Dan and Ramapuram, Jason and Zhang, Yizhe and Gu, Jiatao and Susskind, Joshua M},
  booktitle={International Conference on Machine Learning},
  pages={40770--40803},
  year={2023},
  organization={PMLR}
}

@article{gao2019representation,
  title={Representation degeneration problem in training natural language generation models},
  author={Gao, Jun and He, Di and Tan, Xu and Qin, Tao and Wang, Liwei and Liu, Tie-Yan},
  journal={arXiv preprint arXiv:1907.12009},
  year={2019}
}

@article{zhang2019fixup,
  title={Fixup initialization: Residual learning without normalization},
  author={Zhang, Hongyi and Dauphin, Yann N and Ma, Tengyu},
  journal={arXiv preprint arXiv:1901.09321},
  year={2019}
}

@misc{shazeer2020glu,
  title={{GLU} Variants Improve Transformer},
  author={Noam Shazeer},
  year={2020},
  eprint={2002.05202},
  archivePrefix={arXiv},
  primaryClass={cs.LG}
}

@article{wang2024deepnet,
  title={Deepnet: Scaling transformers to 1,000 layers},
  author={Wang, Hongyu and Ma, Shuming and Dong, Li and Huang, Shaohan and Zhang, Dongdong and Wei, Furu},
  journal={IEEE Transactions on Pattern Analysis and Machine Intelligence},
  volume={46},
  number={10},
  pages={6761--6774},
  year={2024},
  publisher={IEEE}
}

@article{he2020realformer,
  title={RealFormer: Transformer Likes Residual Attention},
  author={He, Ruining and Ravula, Anirudh and Kanagal, Bhargav and Ainslie, Joshua},
  journal={arXiv preprint arXiv:2012.11747},
  year={2020}
}

@article{shleifer2021normformer,
  title={NormFormer: Improved Transformer Pretraining with Extra Normalization},
  author={Shleifer, Sam and Weston, Jason and Ott, Myle},
  journal={arXiv preprint arXiv:2110.09456},
  year={2021}
}

@inproceedings{bachlechner2021rezero,
  title={Rezero is all you need: Fast convergence at large depth},
  author={Bachlechner, Thomas and Majumder, Bodhisattwa Prasad and Mao, Henry and Cottrell, Gary and McAuley, Julian},
  booktitle={Uncertainty in Artificial Intelligence},
  pages={1352--1361},
  year={2021},
  organization={PMLR}
}

@article{zhang2019root,
  title={Root mean square layer normalization},
  author={Zhang, Biao and Sennrich, Rico},
  journal={Advances in Neural Information Processing Systems},
  volume={32},
  year={2019}
}

@inproceedings{touvron2021going,
  title={Going deeper with image transformers},
  author={Touvron, Hugo and Cord, Matthieu and Sablayrolles, Alexandre and Synnaeve, Gabriel and J{\'e}gou, Herv{\'e}},
  booktitle={Proceedings of the IEEE/CVF International Conference on Computer Vision},
  pages={32--42},
  year={2021}
}

@inproceedings{takase2023b2t,
  title={B2t connection: Serving stability and performance in deep transformers},
  author={Takase, Sho and Kiyono, Shun and Kobayashi, Sosuke and Suzuki, Jun},
  booktitle={Findings of the Association for Computational Linguistics: ACL 2023},
  pages={3078--3095},
  year={2023}
}

@article{xie2023residual,
  title={ResiDual: Transformer with Dual Residual Connections},
  author={Xie, Shufang and Zhang, Huishuai and Guo, Junliang and Tan, Xu and Bian, Jiang and Awadalla, Hany Hassan and Menezes, Arul and Qin, Tao and Yan, Rui},
  journal={arXiv preprint arXiv:2304.14802},
  year={2023}
}

@article{kedia2024transformers,
  title={Transformers get stable: An end-to-end signal propagation theory for language models},
  author={Kedia, Akhil and Zaidi, Mohd Abbas and Khyalia, Sushil and Jung, Jungho and Goka, Harshith and Lee, Haejun},
  journal={arXiv preprint arXiv:2403.09635},
  year={2024}
}

@inproceedings{wang2019learning,
  title={Learning deep transformer models for machine translation},
  author={Wang, Qiang and Li, Bei and Xiao, Tong and Zhu, Jingbo and Li, Changliang and Wong, Derek F and Chao, Lidia S},
  booktitle={Proceedings of the 57th Annual Meeting of the Association for Computational Linguistics},
  pages={1810--1822},
  year={2019}
}

@inproceedings{devlin2019bert,
  title={Bert: Pre-training of deep bidirectional transformers for language understanding},
  author={Devlin, Jacob and Chang, Ming-Wei and Lee, Kenton and Toutanova, Kristina},
  booktitle={Proceedings of the 2019 Conference of the North American Chapter of the Association for Computational Linguistics: Human Language Technologies, Volume 1 (Long and Short Papers)},
  pages={4171--4186},
  year={2019}
}

@article{dosovitskiy2020image,
  title={An image is worth 16x16 words: Transformers for image recognition at scale},
  author={Dosovitskiy, Alexey and Beyer, Lucas and Kolesnikov, Alexander and Weissenborn, Dirk and Zhai, Xiaohua and Unterthiner, Thomas and Dehghani, Mostafa and Minderer, Matthias and Heigold, Georg and Gelly, Sylvain and others},
  journal={arXiv preprint arXiv:2010.11929},
  year={2020}
}

@article{brown2020language,
  title={Language models are few-shot learners},
  author={Brown, Tom and Mann, Benjamin and Ryder, Nick and Subbiah, Melanie and Kaplan, Jared D and Dhariwal, Prafulla and Neelakantan, Arvind and Shyam, Pranav and Sastry, Girish and Askell, Amanda and others},
  journal={Advances in Neural Information Processing Systems},
  volume={33},
  pages={1877--1901},
  year={2020}
}

@article{kaplan2020scaling,
  title={Scaling laws for neural language models},
  author={Kaplan, Jared and McCandlish, Sam and Henighan, Tom and Brown, Tom B and Chess, Benjamin and Child, Rewon and Gray, Scott and Radford, Alec and Wu, Jeffrey and Amodei, Dario},
  journal={arXiv preprint arXiv:2001.08361},
  year={2020}
}

@article{hoffmann2022training,
  title={Training compute-optimal large language models},
  author={Hoffmann, Jordan and Borgeaud, Sebastian and Mensch, Arthur and Buchatskaya, Elena and Cai, Trevor and Rutherford, Eliza and Casas, DDL and Hendricks, Lisa Anne and Welbl, Johannes and Clark, Aidan and others},
  journal={arXiv preprint arXiv:2203.15556},
  volume={10},
  year={2022}
}

@article{loshchilov2017decoupled,
  title={Decoupled weight decay regularization},
  author={Loshchilov, Ilya and Hutter, Frank},
  journal={arXiv preprint arXiv:1711.05101},
  year={2017}
}

@article{raffel2020exploring,
  title={Exploring the limits of transfer learning with a unified text-to-text transformer},
  author={Raffel, Colin and Shazeer, Noam and Roberts, Adam and Lee, Katherine and Narang, Sharan and Matena, Michael and Zhou, Yanqi and Li, Wei and Liu, Peter J},
  journal={Journal of Machine Learning Research},
  volume={21},
  number={140},
  pages={1--67},
  year={2020}
}

@article{touvron2023llama,
  title={Llama 2: Open foundation and fine-tuned chat models},
  author={Touvron, Hugo and Martin, Louis and Stone, Kevin and Albert, Peter and Almahairi, Amjad and Babaei, Yasmine and Bashlykov, Nikolay and Batra, Soumya and Bhargava, Prajjwal and Bhosale, Shruti and others},
  journal={arXiv preprint arXiv:2307.09288},
  year={2023}
}

@article{emadi2026exact,
  title={Exact Attention Sensitivity and the Geometry of Transformer Stability},
  author={Emadi, Seyed Morteza},
  journal={arXiv preprint arXiv:2602.18849},
  year={2026}
}

@article{chen2026post,
  title={Post-LayerNorm Is Back: Stable, ExpressivE, and Deep},
  author={Chen, Chen and Wei, Lai},
  journal={arXiv preprint arXiv:2601.19895},
  year={2026}
}

@article{wu2024role,
  title={On the role of attention masks and layernorm in transformers},
  author={Wu, Xinyi and Ajorlou, Amir and Wang, Yifei and Jegelka, Stefanie and Jadbabaie, Ali},
  journal={Advances in Neural Information Processing Systems},
  volume={37},
  pages={14774--14809},
  year={2024}
}

@inproceedings{zhang2019improving,
  title={Improving deep transformer with depth-scaled initialization and merged attention},
  author={Zhang, Biao and Titov, Ivan and Sennrich, Rico},
  booktitle={Proceedings of the 2019 Conference on Empirical Methods in Natural Language Processing and the 9th International Joint Conference on Natural Language Processing (EMNLP-IJCNLP)},
  pages={898--909},
  year={2019}
}

\appendix
\section*{Appendix}
The appendix follows the same order as the theoretical results in the main text. Appendix~\ref{app:forward_amplification_details} gives the technical assumptions and derivations for the forward analysis. Appendix~\ref{app:backward_bottleneck_details} proves the residual-norm growth result and the gradient bound used in Theorem~\ref{theorem_iterative_bound_on_gradient_norm}. Appendix~\ref{app:collapsed_properties_details} then gives the supporting proofs for the complementary characterization of a collapsed network. Appendix~\ref{app:diagnostics} collects empirical checks of the approximations used in these derivations.

\section{Related Works}

\paragraph{Deep Post-Norm Transformer Stability.}
Deep Post-Norm Transformers are known to be difficult to optimize at large depth and are highly sensitive to warmup and learning-rate scale \citep{xiong2020layer,liu2020understanding,wang2019learning,nguyen2019transformers}. Specifically, prior work has identified two characteristic failure modes in deep Post-Norm models: \textbf{gradient vanishing} makes optimization increasingly difficult \citep{xiong2020layer,liu2020understanding,emadi2026exact,chen2026post}, while \textbf{rank collapse} reduces representation diversity as depth increases \citep{dong2021attention,noci2022signal,saada2024mind,yu2026classic}. These instabilities motivate both practical stabilizers and theoretical analyses.

\paragraph{Stabilization Methods.}
Existing methods modify scale, residual routing, or normalization to mitigate the instability. The most widely used stabilization method is Pre-Norm, which has become the standard choice for training deep Transformers in practice \citep{xiong2020layer,wang2019learning,raffel2020exploring,touvron2023llama}. Other representative examples include residual or parameter rescaling such as Fixup, ReZero, Going Deeper, DeepNet, and B2T \citep{zhang2019fixup,bachlechner2021rezero,touvron2021going,wang2024deepnet,takase2023b2t}; residual-attention or hybrid residual designs such as RealFormer and ResiDual \citep{he2020realformer,xie2023residual}; normalization-centric changes such as ScaleNorm/FixNorm, and NormFormer\citep{nguyen2019transformers,shleifer2021normformer}; and objective-side approaches such as spectrum control, anti-degeneration penalties, and attention entropy regularization \citep{wang2020improving,gao2019representation,zhai2023stabilizing}. Our paper is complementary to this line. By analyzing the rank collapse mechanism in deep Post-Norm training, we pinpoint two concrete mechanisms behind Pre-Norm's stability over Post-Norm, providing insight into how normalization placement affects training stability.

\paragraph{Signal Propagation.}
Classical signal propagation theory studies how activations and gradients evolve with depth in randomly initialized neural networks. Early work studied training difficulty in models using ReLU or linear activations \citep{glorot2010understanding,he2015delving}. Later, researchers adopted multiple theoretical tools to analyze properties of the forward-backward process in fully-connected neural networks, including mean-field theory \citep{poole2016exponential,schoenholz2016deep}, Jacobian conditioning \citep{pennington2017resurrecting}, and Gaussian processes \citep{neal2012bayesian,lee2018gaussian,matthews2018gaussian}. Similar tools were later extended beyond fully connected networks to convolutional and recurrent architectures \citep{xiao2018dynamical,chen2018dynamical}. Our forward and backward analyses follow the signal-propagation viewpoint: we track layerwise quantities and study how they evolve across depth in deep Post-Norm Transformers.

\paragraph{Rank Collapse.}
Rank collapse in deep attention and Transformer stacks has been studied from several complementary angles \citep{dong2021attention,noci2022signal,saada2024mind,yu2026classic,yan2022addressing,wu2024role}. \citet{dong2021attention} first showed in pure attention networks that repeated attention drives token uniformity and doubly exponential convergence to rank-one representations. \citet{noci2022signal} then connected rank collapse to signal propagation at initialization. \citet{saada2024mind} revisited the same phenomenon from a spectral viewpoint, distinguishing collapse in depth from collapse in width. \citet{yu2026classic} provided a quantitative analysis of rank collapse in Post-Norm encoder Transformers using token similarity as the core metric. Our forward amplification analysis extends the work of \citet{yu2026classic} to decoder networks with additional analysis on SwiGLU FFN and RMSNorm.

\paragraph{Gradient Vanishing.}
Gradient vanishing and related imbalance issues in deep Transformers have been analyzed from multiple perspectives \citep{xiong2020layer,liu2020understanding,wang2019learning,zhang2019improving,kedia2024transformers,emadi2026exact,chen2026post}. \citet{xiong2020layer} showed that Post-LN has imbalanced gradient magnitudes at initialization, and more recent analyses attributed gradient vanishing in deep Post-LN networks to repeated normalization Jacobians along the residual path \citep{emadi2026exact,chen2026post}. Our backward analysis is complementary to these results. In addition to showing that Post-Norm backward propagation contracts, we identify a concrete mechanism for why this contraction becomes stronger after collapse.

\newpage
 \section{Appendix Notation Table}
\centerline{\bf Numbers}
\bgroup
\def\arraystretch{1.5}
\begin{longtable}{p{1.35in}p{4.35in}}
$\displaystyle n$ & Sequence length\\
$\displaystyle d$ & Model dimension\\
$\displaystyle \dff$ & Feed-forward hidden dimension\\
$\displaystyle v$ & Vocabulary size\\
$\displaystyle \ldepth$ & Total number of Transformer layers\\
$\displaystyle l$ & Layer index\\
$\displaystyle k \in \{1,2\}$ & Sublayer index, with $k=1$ for attention and $k=2$ for FFN\\
$\displaystyle H$ & Number of attention heads\\
$\displaystyle B$ & Batch size\\
$\displaystyle t_1,t_2$ & Token similarity of $\bX_1$ and $\bX_2$\\
$\displaystyle \rho_1,\rho_2$ & Pairwise row correlation coefficient of $\bX_1$ and $\bX_2$\\
$\displaystyle H_n$ & Harmonic number $\sum_{i=1}^n \frac{1}{i}$\\
$\displaystyle H_n^{(k)}$ & Generalized harmonic number $\sum_{i=1}^n \frac{1}{i^k}$\\
$\displaystyle \gamma_{\mathrm{euler}}$ & Euler--Mascheroni constant\\
$\displaystyle \xi$ & Scalar coefficient multiplying the pairwise SwiGLU moment\\
$\displaystyle \beta_i$ & Number of occurrences of label $i$ in the sequence\\
$\displaystyle \sigma_{\bG_k}^2$ & Gradient variance\\
$\displaystyle \alpha_k^l$ & Sublayer gradient contribution factor in layer $l$\\
$\displaystyle c_0$ & Exact-collapse RMSNorm shrinkage factor\\
\end{longtable}
\egroup
\vspace{0.25cm}

\centerline{\bf Vectors and Matrices}
\bgroup
\def\arraystretch{1.5}
\begin{longtable}{p{1.35in}p{4.35in}}
$\displaystyle \done_n$ & All-one column vector in $\mathbb{R}^{n\times 1}$\\
$\displaystyle \simmat$ & Mean-projection matrix $\frac{1}{n}\done_n\done_n^\top$\\
$\displaystyle \bC_n$ & Causal prefix-averaging matrix\\
$\displaystyle \tlower{\bM}{causal}$ & Additive causal mask used in attention logits\\
$\displaystyle \bX_k^l$ & Input of sublayer $k$ in layer $l$\\
$\displaystyle \bY_k^l$ & Output of sublayer $k$ in layer $l$, including the residual addition\\
$\displaystyle \bP$ & Attention weight matrix\\
$\displaystyle \hat{\bp}$ & Shared output distribution under exact collapse\\
$\displaystyle \hat\bP$ & Output probability matrix under exact collapse\\
$\displaystyle \bE(\by)$ & One-hot label matrix for the label sequence $\by$\\
$\displaystyle \bar\bX$ & Mean-projected component $\simmat \bX$ of $\bX$\\
$\displaystyle \bar\by_k$ & Mean token vector in the near-collapse decomposition of $\bY_k$\\
$\displaystyle \bG_k$ & Backpropagated gradient $\partial \caL / \partial \bY_k$\\
$\displaystyle \bR_s$ & SwiGLU hidden branch feature\\
$\displaystyle \bP_{de}$ & Deescalated attention matrix, $\bP-\alpha\bC_n$\\
$\displaystyle \bW_Q,\bW_K,\bW_V,\bW_O$ & Query, key, value, and output projection matrices\\
$\displaystyle \bW$ & Folded attention output matrix in the single-head theory\\
$\displaystyle \bW_1,\bW_2,\bW_3$ & SwiGLU parameter matrices\\
$\displaystyle \wlm$ & LM-head matrix\\
$\displaystyle \bX^H$ & Hidden representation before the LM head\\
\end{longtable}
\egroup
\vspace{0.25cm}

\centerline{\bf Functions}
\bgroup
\def\arraystretch{1.5}
\begin{longtable}{p{1.35in}p{4.35in}}
$\displaystyle \softmaxrow(\cdot)$ & Row-wise softmax\\
$\displaystyle \rms(\cdot)$ & RMSNorm applied row-wise\\
$\displaystyle \attn(\cdot)$ & Attention branch\\
$\displaystyle \ffn(\cdot)$ & Feed-forward branch\\
$\displaystyle \tsim{\bX}$ & Token similarity of $\bX$\\
$\displaystyle \tdiv{\bX}$ & Token diversity of $\bX$\\
$\displaystyle m(\rho)$ & Pairwise SwiGLU moment\\
$\displaystyle c(\by_k^l,\alpha_k^l)$ & One-step gradient contraction factor in layer $l$\\
\end{longtable}
\egroup
\vspace{0.25cm}

\section{Forward Similarity Amplification Details and Proofs}\label{app:forward_amplification_details}

This section develops the forward amplification part of the paper. It states the initialization-time assumptions, derives the one-step attention and SwiGLU updates, quantifies the RMSNorm correction, and explains why the ratio-of-expectations surrogate is accurate for the one-step change in token similarity.
\subsection{Technical Assumptions for the Forward Analysis}
This subsection collects the initialization-time assumptions used throughout the forward appendix. Under these assumptions, the attention and SwiGLU residual updates can be reduced to expected second moments, which leads to the closed-form one-step token-similarity formulas proved later in this section.

\begin{assumption}[Initialization-Time Structure for the Forward Analysis]
\label{assumption_on_X}
We use the following initialization-time structure in the forward analysis.
\begin{enumerate}[(i)]
    \item \textbf{Initialization Scheme.}
    All weight matrices $\bW, \bW_Q, \bW_K, \bW_1, \bW_2, \bW_3$ have i.i.d.\ Gaussian entries with mean $0$.

    \item \textbf{Input Structure.}
    For $k\in\{1,2\}$, each row of $\bX_k$ has squared norm $d$, namely
    \[
    \|(\bX_k)_{i,:}\|_2^2=d,\qquad \forall i\in\{1,\dots,n\}.
    \]
    Hence $\|\bX_k\|_F^2=nd$.

    \item \textbf{Equal-Correlation Closure.}
    For $k\in\{1,2\}$, let
    \[
    t_k:=\tsim{\bX_k},
    \qquad
    \rho_k:=\rho(t_k)=\frac{nt_k-1}{n-1}.
    \]
    We approximate the token Gram matrix by
    \begin{equation}
    \label{expect_xxt}
    \bX_k\bX_k^\top
    \approx
    d\bigl(\rho_k\,n\simmat + (1-\rho_k)\bI\bigr).
    \end{equation}
    Combined with part (ii), this means that each diagonal entry equals $d$, while each off-diagonal entry is approximated by $d\,\rho_k$.

    \item \textbf{Ratio-of-Expectations Surrogate.}
    For the sublayer outputs $\bY_k$ analyzed below, we approximate
    \[
    \expect{}{\tsim{\bY_k}}
    =
    \expect{}{\frac{\frobsq{\simmat\bY_k}}{\frobsq{\bY_k}}}
    \approx
    \frac{\expect{}{\frobsq{\simmat\bY_k}}}{\expect{}{\frobsq{\bY_k}}}.
    \]
\end{enumerate}
\end{assumption}

\paragraph{Remark.}
Parts (i)--(iii) specify the initialization-time structure used in the forward analysis. In particular, part (iii) gives a one-parameter equal-correlation closure for the token Gram matrix, with the parameter determined by the token similarity $t_k=\tsim{\bX_k}$.

\subsection{Proof of Theorem~\ref{theorem_forward_amplification}}
\label{app:stage1_one_step_expectations}

We now derive the attention-side one-step formula in Theorem~\ref{theorem_forward_amplification}. The derivation computes the expected numerator and denominator after one attention residual update, and applying the ratio-of-expectations surrogate gives the one-step token-similarity change. In addition to Assumption~\ref{assumption_on_X}, the attention derivation uses one extra approximation on the attention weights.
\begin{assumption}[Prefix-Averaging Approximation]
\label{assumption_prefix_averaging}
At initialization, we approximate the causal attention matrix $\bP$ by the prefix-averaging matrix $\bC_n$, where
\[
(\bC_n)_{i,j}
=
\begin{cases}
    1/i, & \text{if } i \geq j,\\
    0,   & \text{otherwise}.
\end{cases}
\]
\end{assumption}
\paragraph{Remark.}
At initialization, the prefix-averaging matrix $\bC_n$ serves as the deterministic surrogate used for the causal attention matrix in the forward appendix.

The attention layer with residual connection is:
\begin{align*}
    \bP &= \softmaxrow\left(\frac{(\bX_1\bW_Q)(\bX_1\bW_K)^\top}{\sqrt{d}} + \tlower{\bM}{causal}\right)\\
    \bY_1 &= \bX_1 + \bP\bX_1\bW
\end{align*}
where $(\tlower{\bM}{causal})_{i, j} = 0$ for $i \geq j$, and $-\infty$ otherwise. The matrix $\bP$ is the attention weight matrix. We can expand the denominator and numerator part of $\bY_1$ as :
\begin{align*}
    \frob{\simmat\bY_1}^2&= \frob{\simmat\bX_1}^2 + 2\inner{\simmat\bX_1}{\simmat\bP\bX_1\bW} + \frob{\simmat\bP\bX_1\bW}^2 \\
    \frob{\bY_1}^2 &= \frob{\bX_1}^2 + 2\inner{\bX_1}{\bP\bX_1\bW} + \frob{\bP\bX_1\bW}^2
\end{align*}
Using the zero-mean initialization of $\bW$, the cross terms vanish after taking expectation over $\bW$. We then evaluate the remaining quadratic terms with $\bX_1$ treated as fixed. Under Assumption~\ref{assumption_on_X}(iii), we use the closure
\[
\bX_1\bX_1^\top
\approx
d\bigl(\rho_1\,n\simmat + (1-\rho_1)\bI\bigr),
\qquad
\rho_1:=\frac{n\tsim{\bX_1}-1}{n-1},
\]
and under the prefix-averaging approximation we replace $\bP$ by $\bC_n$. This yields the closed-form evolution:
\begin{theorem}
\label{theorem_ts_attn}
    Under Assumptions~\ref{assumption_on_X} and~\ref{assumption_prefix_averaging}, we have
    \begin{align*}
        \expect{}{\frob{\simmat\bY_1}^2} &= nd\tsim{\bX_1} + d^2\sigma_{\bW}^2\left(\rho_1 n + (1 - \rho_1)(2 - H_n/n)\right) \\
        \expect{}{\frob{\bY_1}^2} &= nd + d^2\sigma_{\bW}^2\left(\rho_1 n + (1 - \rho_1)H_n\right)
    \end{align*}
    where $\rho_1 = \frac{n\tsim{\bX_1}-1}{n-1}$ and the harmonic series $H_n = \sum_{i = 1}^n\frac{1}{i}$. 
\end{theorem}
\begin{proof}
     We first consider the more complicated term, $\expect{}{\frobsq{\simmat\bY}}$, by expanding it directly
\begin{align*}
    \expect{}{\frob{\simmat\bY}^2} &= \expect{}{\frob{\simmat\bX + \simmat\bP\bX\bW}^2} \\
    &= \expect{}{\frob{\simmat\bX}^2} + 2\expect{}{\inner{\simmat\bX}{\simmat\bP\bX\bW}} + \expect{}{\frob{\simmat\bP\bX\bW}^2}\\
    &= \expect{}{\frob{\simmat\bX}^2} + \expect{}{\frob{\simmat\bP\bX\bW}^2}
\end{align*}
The last equality follows from 
\begin{align*}
    \expect{}{\inner{\simmat\bX}{\simmat\bP\bX\bW}} &= \expect{}{\tr(\bX^T\simmat\bP\bX\bW)}&& \\
    &=\tr(\expect{}{\bX^T\simmat\bP\bX}\expect{}{\bW}) =0 \qquad \text{(By (i) of Assumption \ref{assumption_on_X}, $\expect{}{\bW} = 0$})
\end{align*}
We already have $\frob{\simmat\bX}^2 = \tsim{\bX}\frob{\bX}^2 = nd \tsim{\bX}$ from the definition of $\tsim{\bX}$. Let $\rho := \frac{n\tsim{\bX}-1}{n-1}$. For the term $\expect{}{\frob{\simmat\bP\bX\bW}^2}$, using $\expect{}{\bW\bW^T}=d\sigma_{\bW}^2\bI$, we obtain
    \begin{align*}
    \expect{\bW_Q, \bW_K, \bW}{\frob{\simmat\bP\bX\bW}^2} &= \expect{\bW_Q, \bW_K, \bW}{\tr(\simmat\bP\bX\bW\bW^T\bX^T\bP^T\simmat)} \\
    &= \expect{\bW_Q, \bW_K}{\tr(\simmat\bP\bX\expect{\bW}{\bW\bW^T}\bX^T\bP^T\simmat)} \\
    &= d\sigma_{\bW}^2\expect{\bW_Q, \bW_K}{\tr(\simmat\bP\bX\bX^T\bP^T\simmat)} \\
    &= d\sigma_{\bW}^2\tr(\simmat\expect{\bW_Q, \bW_K}{\bP\bX\bX^T\bP^T}\simmat)
\end{align*}

We then need to compute $\expect{}{\bP\bX\bX^T\bP^T}$. Under Assumption~\ref{assumption_prefix_averaging}, we replace $\bP$ by $\bC_n$ and obtain
\[
\expect{}{\bP\bX\bX^T\bP^T}
\approx
\bC_n\,\bX\bX^T\,\bC_n^T.
\]
Applying Assumption~\ref{assumption_on_X}(iii) then gives
\begin{align*}
    \expect{}{\bP\bX\bX^T\bP^T}
    &\approx
    \bC_n\,\bX\bX^T\,\bC_n^T\\
    &\approx
    \bC_n \bigl(d(\rho\, n \simmat + (1-\rho)\bI)\bigr)\bC_n^T \\
    &= d\rho\, n \simmat + d(1-\rho)\bC_n\bC_n^T
\end{align*}
and
\begin{align*}
    \tr(\simmat\expect{}{\bP\bX\bX^T\bP^T}\simmat) &= \tr(d\rho\, n \simmat + d(1-\rho)\simmat\bC_n\bC_n^T\simmat) \\
    &= d\rho\, n + d(1-\rho) \frob{\simmat \bC_n}^2
\end{align*}

Denote the harmonic series as $H_n = \sum_{k = 1}^n \frac{1}{k}$, we can express $\frob{\simmat\bC_n}$ as follows
\begin{align*}
    \frob{\simmat\bC_n}^2
    &= \frac{1}{n}\sum_{i=1}^n\left(\sum_{j=i}^n \frac{1}{j}\right)^2 \\
    &= \frac{1}{n}\sum_{i=1}^n
    \left(
        \sum_{j=i}^n\frac{1}{j^2}
        +2\sum_{i\le j<k\le n}\frac{1}{jk}
    \right) \\
    &= \frac{1}{n}
    \left(
        \sum_{j=1}^n \frac{j}{j^2}
        +2\sum_{1\le j<k\le n}\frac{j}{jk}
    \right) \\
    &= \frac{1}{n}\left(H_n+2\sum_{k=1}^n\frac{k-1}{k}\right) \\
    &= 2 - \frac{1}{n}\sum_{i = 1}^n \frac{1}{i} \\
    &= 2 - \frac{1}{n}H_n
\end{align*}
Plug $\frob{\simmat\bC_n}^2 = 2 - \frac{1}{n}H_n$ back to previous formula, we get 
\begin{align*}
    \expect{}{\frob{\simmat\bY}^2} &= nd\tsim{\bX} + d^2\sigma_{\bW}^2\left(\rho n + (1 - \rho)(2 - H_n/n)\right)
\end{align*}
Following similar line of argument, we have the following formulas for $\expect{}{\frob{\bY}^2}$

\begin{align*}
    \expect{}{\frob{\bY}^2} &= \expect{}{\frob{\bX}^2} + \expect{}{\frob{\bP\bX\bW}^2} \\
    \expect{\bW_Q, \bW_K, \bW}{\frob{\bP\bX\bW}^2} &= d\sigma_{\bW}^2\tr(\expect{\bW_Q, \bW_K}{\bP\bX\bX^T\bP^T}) \\
    \tr(\expect{\bW_Q, \bW_K}{\bP\bX\bX^T\bP^T}) &= d\rho\, n + d(1 - \rho) \frob{\bC_n}^2
\end{align*}
We can express $\frob{\bC_n}^2$ as
\begin{align*}
    \frob{\bC_n}^2 &= \sum_{i = 1}^n i \times \frac{1}{i^2} = H_n
\end{align*}
Plug $\frob{\bC_n}^2 = H_n$ back into previous derivation we get
\begin{align*}
    \expect{}{\frob{\bY}^2} &= nd + d^2\sigma_{\bW}^2\left(\rho n + (1 - \rho)H_n\right)
\end{align*}
For $\expect{}{\tsim{\bY}}$, plug in our derivation for $\expect{}{\frob{\bY}^2}, \expect{}{\frob{\simmat\bY}^2}$ respectively and we get the final formula. The harmonic series can be approximated by
\begin{align*}
    H_n \approx \log n + \gamma_{\mathrm{euler}} + \frac{1}{2n} - \frac{1}{12n^2}
\end{align*}
where $\gamma_{\mathrm{euler}} \approx 0.5772$ is the Euler–Mascheroni constant. For neural network $n$ is generally greater than 100, which makes $H_n \approx  \log n + \gamma_{\mathrm{euler}}$ a good approximation.
\end{proof}

\subsection{Proof of Corollary~\ref{theorem_attention_deescalation}}
\label{app:proof_attention_deescalation}
 We now prove Corollary~\ref{theorem_attention_deescalation}, which isolates the role of the prefix-averaging component $\alpha\bC_n$. The calculation shows how removing this component weakens the one-step similarity increase for the attention sublayer.
\begin{proof}
Define
\[
\bP_{de}:=\bP-\alpha\bC_n,\qquad
\bY_{de}=\bX_1+\bP_{de}\bX_1\bW,\qquad
t:=\tsim{\bX_1},
    \qquad
    \rho_1:=\frac{nt-1}{n-1}.
\]
As in the proof of the attention forward theorem, cross terms vanish in expectation because $\expect{}{\bW}=0$, so it suffices to compute second moments of the attention branch.

Using the same prefix-averaging replacement,
\[
\expect{}{\bP_{de}\bX_1\bX_1^\top\bP_{de}^\top}
\approx
(1-\alpha)^2 \bC_n \bX_1\bX_1^\top \bC_n^\top,
\]
as above,
\begin{align*}
\expect{}{\bP_{de}\bX_1\bX_1^\top\bP_{de}^\top}
&\approx
(1-\alpha)^2 \bC_n \bX_1\bX_1^\top \bC_n^\top \\
&=
(1-\alpha)^2 d\,\bC_n\bigl(\rho_1\done_n\done_n^\top+(1-\rho_1)\bI\bigr)\bC_n^\top.
\end{align*}
Let $\lambda_\alpha:=(1-\alpha)^2$. Repeating the trace computations from Theorem \ref{theorem_ts_attn} gives
\begin{align*}
\expect{}{\frobsq{\simmat\bY_{de}}}
&= nd\,t+d\sigma_{\bW}^2\lambda_\alpha\!\left(\rho_1 n+(1-\rho_1)\left(2-\frac{H_n}{n}\right)\right),\\
\expect{}{\frobsq{\bY_{de}}}
&= nd+d\sigma_{\bW}^2\lambda_\alpha\!\left(\rho_1 n+(1-\rho_1)H_n\right).
\end{align*}
Dividing numerator and denominator by $nd$, and using the Post-Norm identity $\frobsq{\bX_1}=nd$ so that the main-text attention-strength scalar becomes
\[
s=\frac{nd^2\sigma_{\bW}^2}{\frobsq{\bX_1}}=d\sigma_{\bW}^2,
\]
we get
\[
\expect{}{\tsim{\bY_{de}}}
\approx
\frac{
t+\lambda_\alpha s\left[\rho_1+\frac{1-\rho_1}{n}\left(2-\frac{H_n}{n}\right)\right]
}{
1+\lambda_\alpha s\left[\rho_1+\frac{1-\rho_1}{n}H_n\right]
}.
\]
Using $\frac{a+b}{1+c}=a+\frac{b-ac}{1+c}$, this becomes
\[
\expect{}{\tsim{\bY_{de}}}
\approx
t+\frac{
\lambda_\alpha s\left(
\rho_1(1-t)+\frac{1-\rho_1}{n}\left(2-\frac{H_n}{n}-tH_n\right)
\right)
}{
1+\lambda_\alpha s\left[\rho_1+\frac{1-\rho_1}{n}H_n\right]
}.
\]
Substituting $\rho_1=\frac{nt-1}{n-1}$ gives
\begin{align*}
\rho_1(1-t)+\frac{1-\rho_1}{n}\left(2-\frac{H_n}{n}-tH_n\right)
&=
\frac{(1-t)(n-H_n)(nt+1)}{n(n-1)}
=
f_1(t), \\
\rho_1+\frac{1-\rho_1}{n}H_n
&=
\frac{(n-H_n)t+H_n-1}{n-1}
=
f_2(t),
\end{align*}
which are exactly the two functions defined in Theorem~\ref{theorem_forward_amplification}. Therefore,
\[
\expect{}{\tsim{\bY_{de}}}
\approx
R(\lambda_\alpha)
:=
t+\frac{\lambda_\alpha s f_1(t)}{1+\lambda_\alpha s f_2(t)},
\]
To derive monotonicity, differentiate w.r.t. $\lambda_\alpha$:
\[
\frac{dR}{d\lambda_\alpha}
=
\frac{s f_1(t)}{(1+\lambda_\alpha s f_2(t))^2}.
\]
For $t\in[1/n,1)$, we have $f_1(t)>0$, so $\frac{dR}{d\lambda_\alpha}>0$. Since
\[
\frac{d\lambda_\alpha}{d\alpha}=-2(1-\alpha)\le 0,
\]
we obtain $\frac{dR}{d\alpha}<0$ for $\alpha\in[0,1)$ in this regime, so larger $\alpha$ reduces expected similarity. Also,
\[
\lim_{\alpha\to 1}\expect{}{\tsim{\bY_{de}}}
\approx
\lim_{\lambda_\alpha\to 0}R(\lambda_\alpha)=t.
\]
\end{proof}

\subsection{Proof of Theorem~\ref{theorem_swiglu_pairwise}}
\label{app:proof_swiglu_forward}

\begin{proof}
We next derive the SwiGLU one-step correction in Theorem~\ref{theorem_swiglu_pairwise}. Let
\[
\bR_s:=\silu(\bX_2\bW_1)\odot(\bX_2\bW_3),
\qquad
\bY_2=\bX_2+\bR_s\bW_2.
\]
As in the attention case, we compute the expected numerator and denominator separately so the FFN contribution can be compared directly with the attention-side increment. Since $\bW_1,\bW_3,\bW_2$ are all mean-zero and independent of $\bX_2$, we have
\begin{align*}
\expect{}{\inner{\bX_2}{\bR_s\bW_2}}
&= \expect{}{\tr(\bX_2^\top\bR_s\bW_2)}
= \tr\!\bigl(\expect{}{\bX_2^\top\bR_s}\expect{}{\bW_2}\bigr) = 0.
\end{align*}
Hence,
\begin{align*}
\expect{}{\frobsq{\simmat\bY_2}} &= \expect{}{\frobsq{\simmat\bX_2}}
  + \expect{}{\frobsq{\simmat\bR_s\bW_2}}, \\
\expect{}{\frobsq{\bY_2}} &= \expect{}{\frobsq{\bX_2}}
  + \expect{}{\frobsq{\bR_s\bW_2}}.
\end{align*}

It therefore remains to compute the branch second moments. Since $\bW_2$ is independent of $(\bX_2,\bW_1,\bW_3)$ and $\expect{}{\bW_2\bW_2^\top}=d\sigma_{\bW_2}^2\bI$,
\begin{align*}
\expect{}{\frobsq{\simmat\bR_s\bW_2}}
&= d\sigma_{\bW_2}^2\,\tr\!\bigl(\simmat\expect{}{\bR_s\bR_s^\top}\simmat\bigr), \\
\expect{}{\frobsq{\bR_s\bW_2}}
&= d\sigma_{\bW_2}^2\,\tr\!\bigl(\expect{}{\bR_s\bR_s^\top}\bigr).
\end{align*}

Let $\bA=\bX_2\bW_1$ and $\bB=\bX_2\bW_3$. Since $\bW_1\perp\bW_3$, for any $k\in[\dff]$:
\begin{align*}
(\bR_s)_{i,k} &= \silu(\bA_{i,k})\cdot \bB_{i,k}.
\end{align*}
For $i\neq j$, using independence of $\bW_1$ and $\bW_3$:
\begin{align*}
\expect{}{(\bR_s)_{i,k}(\bR_s)_{j,k}}
&= \expect{}{\silu(\bA_{i,k})\silu(\bA_{j,k})}\cdot\expect{}{\bB_{i,k}\bB_{j,k}}.
\end{align*}
Under Assumption \ref{assumption_on_X}, $(\bA_{i,k},\bA_{j,k})$ is bivariate Gaussian with marginal
variance $\sigma_{\bW_1}^2 d$ and correlation $\rho_2$, and
$\expect{}{\bB_{i,k}\bB_{j,k}}=d\sigma_{\bW_3}^2 \rho_2$ for $i\neq j$, $\sigma_{\bW_3}^2 d$ for $i=j$.
Define
\[
m(\rho) := \expect{}{\silu(u)\silu(v)},
\qquad (u,v)\sim\caN\!\left(0,\sigma_{\bW_1}^2 d\begin{pmatrix}1&\rho\\\rho&1\end{pmatrix}\right).
\]
Summing over $k$ under the equal-correlation closure gives
\begin{align*}
\expect{}{\bR_s\bR_s^\top}_{ij}
&= \dff\cdot m(\rho_2)\cdot\sigma_{\bW_3}^2 d\,\rho_2,
\quad i\neq j, \\
\expect{}{\bR_s\bR_s^\top}_{ii}
&= \dff\cdot m(1)\cdot\sigma_{\bW_3}^2 d. \\
\expect{}{\bR_s\bR_s^\top}
&= \dff\sigma_{\bW_3}^2 d
\bigl(n\,m(\rho_2)\rho_2\,\simmat
  + (m(1)-m(\rho_2)\rho_2)\bI\bigr).
\end{align*}

Using $\tr(\simmat\bM\simmat)=\frac{1}{n}\sum_{i,j}M_{ij}$ and $\tr(\bM)=\sum_i M_{ii}$, we obtain
\begin{align*}
\tr\!\bigl(\simmat\expect{}{\bR_s\bR_s^\top}\simmat\bigr)
&= \dff\sigma_{\bW_3}^2 d
\bigl(m(1)+(n-1)m(\rho_2)\rho_2\bigr) \\
&= \dff\sigma_{\bW_3}^2 d\,h_n(\rho_2), \\
\tr\!\bigl(\expect{}{\bR_s\bR_s^\top}\bigr)
&= n\dff\sigma_{\bW_3}^2 d\,m(1),
\end{align*}
where
\[
h_n(\rho_2):=m(1)+(n-1)m(\rho_2)\rho_2.
\]
The important asymmetry is that the denominator trace depends only on diagonal terms, whereas the equal-correlation closure only enters through the off-diagonal contribution to the projected numerator trace. Substituting these traces back into the numerator and denominator expressions, and using $\expect{}{\frobsq{\simmat\bX_2}}=nd\,\tsim{\bX_2}$ and $\expect{}{\frobsq{\bX_2}}=nd$, we obtain
\begin{align*}
	\expect{}{\frobsq{\simmat\bY_2}}
	&= nd\,\tsim{\bX_2}
    + \dff d^2\sigma_{\bW_3}^2\sigma_{\bW_2}^2\,h_n(\rho_2), \\
	\expect{}{\frobsq{\bY_2}}
	&= nd + n\dff d^2\sigma_{\bW_3}^2\sigma_{\bW_2}^2\,m(1).
	\end{align*}
Therefore, using the ratio-of-expectations surrogate and writing $t_2:=\tsim{\bX_2}$, we obtain
\[
\expect{}{\tsim{\bY_2}}
\approx
\frac{
t_2
+
\frac{d\tlower{d}{ff}\sigma_{\bW_2}^2\sigma_{\bW_3}^2}{n}
\bigl(m(1)+(nt_2-1)m(\rho_2)\bigr)
}
{1+d\tlower{d}{ff}\sigma_{\bW_2}^2\sigma_{\bW_3}^2\,m(1)},
\]
where the coefficient is denoted $\xi := d\tlower{d}{ff}\sigma_{\bW_2}^2\sigma_{\bW_3}^2$ in the main text.

\begin{figure}
    \centering
    \includegraphics[width=0.4\textwidth]{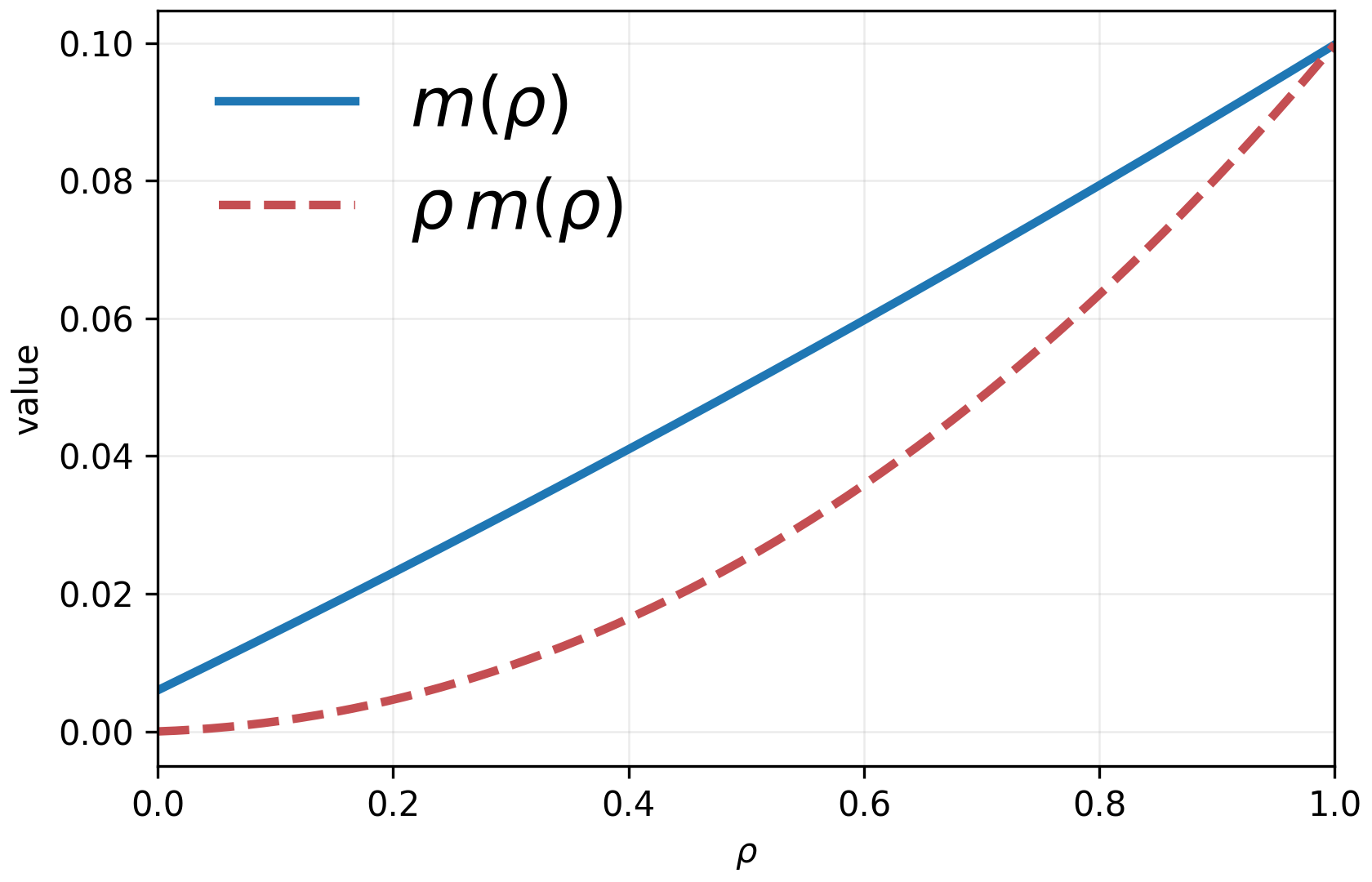}
    \caption{$m(\rho)$ under standard initialization.}
    \label{fig:swiglu_mrho_numerical}
\end{figure}
Rewriting the same expression in terms of $t_2$ gives Eq.~\eqref{eq:swiglu_scalar_summary}. Without that specialization, the same derivation already shows that only the projected numerator term changes, while the denominator term stays the same.
\end{proof}

\paragraph{Numerical evaluation of $m(\rho)$.}
To give a better understanding of the numerical properties of the SwiGLU correction, we evaluate $m(\rho)$ numerically. We can see that $m(\rho)$ is approximately linear and increasing on $\rho\in[0,1]$, and $\rho m(\rho)$ is also increasing on $\rho\in[0,1]$. Numerically, $m(0)\approx 0.006$ and $m(1)\approx 0.100$ under current variance setting



\subsection{Deterministic Control of the RMSNorm Forward Correction}
We now bound the change from $\tsim{\bY}$ to $\tsim{\bT\bY}$ for a fixed diagonal matrix $\bT$. The bound is controlled by
\[
\eta := \frac{\|\bT-\mu\bI\|_2}{\mu},
\]
which measures how far the RMS factors are from their mean scalar $\mu$. Let $\bT := \bGamma(\bY)$ be diagonal, with $\bT_{i,i} = 1/\sqrt{\frac{1}{d}\|\bY_{i,:}\|_2^2 + \epsilon}$. Then RMS normalization is $\bZ = \bT\bY$. In this subsection only, we keep the numerical stabilizer $\epsilon$ so that the row-wise scaling factors remain well-defined even when some row $\bY_{i,:}$ is zero.
On the FFN branch, $\mathbb{E}\|\bY_{i,:}\|_2^2$ is the same for every position $i$. On the attention branch, the next lemma writes $\mathbb{E}\|\bY_{i,:}\|_2^2$ as $\alpha+\beta/i$, so the expected squared row norm varies with $i$ through the $1/i$ term.

\paragraph{Token Similarity after RMS.}

We begin with the SwiGLU branch, where $\mathbb{E}\|\bY_{i,:}\|_2^2$ is the same for every position $i$.

\begin{lemma} \label{lemma_ffn_row_homogenity}
Under Assumption \ref{assumption_on_X}, the expected squared row norms are constant across positions:
\begin{align*}
    \expect{}{\|\bY_{i,:}\|_2^2}
    =
    d + \dff d^2 \sigma_{\bW_3}^2\sigma_{\bW_2}^2\,m(1),
    \qquad \text{for all } i \in [n],
\end{align*}
\end{lemma}

Lemma~\ref{lemma_ffn_row_homogenity} shows that $\mathbb{E}\|\bY_{i,:}\|_2^2$ does not depend on the position $i$ on the SwiGLU branch. We next turn to the attention branch, where causal masking makes the row norms position-dependent.

Self-attention outputs have position-dependent row norms because later tokens attend to more previous tokens. The next lemma gives the form of $\mathbb{E}\|\bY_{i,:}\|_2^2$ across positions.

\begin{lemma} \label{lem:attn-row-norm}
For $\bY = \bX + \bP\bX\bW$ under Assumption \ref{assumption_on_X}, 
the expected squared row norms follow:
\[
\mathbb{E}\bigl[\|\bY_{i,:}\|_2^2\bigr] = \alpha + \frac{\beta}{i}, 
\quad \text{where} \quad 
\alpha := d + d^2\rho\,\sigma_{\bW}^2, \quad \beta := d^2\sigma_{\bW}^2(1-\rho), \quad \rho:=\frac{n\tsim{\bX}-1}{n-1}.
\]

\end{lemma}
Lemma~\ref{lemma_ffn_row_homogenity} and Lemma~\ref{lem:attn-row-norm} give formulas for $\mathbb{E}\|\bY_{i,:}\|_2^2$. Theorem~\ref{theorem_attention_rms} then bounds $\tsim{\bT\bY}$ directly in terms of the diagonal matrix $\bT$.

\begin{theorem}
\label{theorem_attention_rms}
Let $\bT$ be any diagonal matrix with positive diagonal entries, and define
\[
\mu := \frac{1}{n}\tr(\bT),
\qquad
\eta := \frac{\|\bT-\mu\bI\|_2}{\mu}.
\]
For any nonzero $\bY \in \mathbb{R}^{n\times d}$, let $t := \tsim{\bY}$. If $\eta < 1$, then
\[
\frac{\bigl(\max\{\sqrt{t}-\eta,0\}\bigr)^2}{(1+\eta)^2}
\le
\tsim{\bT\bY}
\le
\frac{(\sqrt{t}+\eta)^2}{(1-\eta)^2}.
\]
In particular, if $\eta \ll \sqrt{t}$, then the bounds above force $\tsim{\bT\bY}/\tsim{\bY}$ to stay close to $1$.
\end{theorem}
\subsection{Supporting Proof of Theorem~\ref{theorem_attention_rms}}

\subsubsection{Proof of Lemma~\ref{lemma_ffn_row_homogenity}}

\begin{proof}
For SwiGLU FFN, $\bY=\bX+\bR_s\bW_2$ with
\[
\bR_s=\silu(\bX\bW_1)\odot(\bX\bW_3).
\]
For any row $i$,
\begin{align*}
\expect{}{\| \bY_{i,:}\|_2^2}
&=
\expect{}{\|\bX_{i,:}\|_2^2}
 + \expect{}{\|\bR_{s,i,:}\bW_2\|_2^2}
 + 2\expect{}{\langle \bX_{i,:}, \bR_{s,i,:}\bW_2\rangle}.
\end{align*}
The cross term is zero because $\expect{}{\bW_2}=0$ and $\bW_2$ is independent of $(\bX,\bW_1,\bW_3)$.
By Assumption \ref{assumption_on_X}(ii), $\expect{}{\|\bX_{i,:}\|_2^2}=d$.
As computed in the proof of Theorem \ref{theorem_swiglu_pairwise}, the diagonal branch second moment is row-independent:
\[
\expect{}{\|\bR_{s,i,:}\|_2^2}
=
\expect{}{\bigl(\bR_s\bR_s^\top\bigr)_{ii}}
=
\dff d\sigma_{\bW_3}^2\,m(1).
\]
Finally,
\[
\expect{}{\|\bR_{s,i,:}\bW_2\|_2^2}
=
d\sigma_{\bW_2}^2\,\expect{}{\|\bR_{s,i,:}\|_2^2}
=
\dff d^2\sigma_{\bW_3}^2\sigma_{\bW_2}^2\,m(1).
\]
Hence
\[
\expect{}{\| \bY_{i,:}\|_2^2}
=
d+\dff d^2\sigma_{\bW_3}^2\sigma_{\bW_2}^2\,m(1),
\]
independent of $i$.
\end{proof}

\subsubsection{Proof of Lemma~\ref{lem:attn-row-norm}}
\begin{proof}
For the attention module, expand the row norm as
\begin{align*}
    \bY &= \bX + \bP\bX\bW, \\
    \expect{}{\|\bY_{i,:}\|_2^2}
    &= \expect{}{\|\bX_{i,:}\|_2^2}
     + \expect{}{\|\bP_{i,:}\bX\bW\|_2^2}
     + 2\expect{}{\inner{\bX_{i,:}}{\bP_{i,:}\bX\bW}}.
\end{align*}
The cross term vanishes because $\expect{}{\bW}=0$. By Assumption~\ref{assumption_on_X}(ii), $\expect{}{\|\bX_{i,:}\|_2^2}=d$. Using the prefix-averaging profile of the attention row and the equal-correlation closure, we obtain
\[
\expect{}{\|\bP_{i,:}\bX\bW\|_2^2}
=
d^2\sigma_{\bW}^2\left(\frac{1}{i}+\frac{i-1}{i}\rho\right),
\qquad
\rho=\frac{n\tsim{\bX}-1}{n-1}.
\]
Rearranging gives
\[
\expect{}{\|\bY_{i,:}\|_2^2}
=
d+d^2\rho\sigma_{\bW}^2+\frac{d^2\sigma_{\bW}^2(1-\rho)}{i}
=
\alpha+\frac{\beta}{i},
\]
with $\alpha := d + d^2\rho\,\sigma_{\bW}^2$ and $\beta := d^2\sigma_{\bW}^2(1-\rho)$.
\end{proof}

\subsubsection{Proof of Theorem~\ref{theorem_attention_rms}}
\begin{proof}
Write
\[
\bT = \mu \bI + \bDelta,
\qquad
\|\bDelta\|_2 = \eta \mu,
\]
and let $\bP := \simmat$. Then $\|\bP\|_2 = 1$ and
\[
t = \frac{\|\bP\bY\|_F^2}{\|\bY\|_F^2}.
\]
Since $\eta < 1$, the diagonal entries of $\bT$ stay positive and
\begin{align*}
\|\bT\bY\|_F
&\le
\mu\|\bY\|_F + \|\bDelta\bY\|_F
\le
\mu(1+\eta)\|\bY\|_F, \\
\|\bT\bY\|_F
&\ge
\mu\|\bY\|_F - \|\bDelta\bY\|_F
\ge
\mu(1-\eta)\|\bY\|_F.
\end{align*}
For the numerator $\|\bP\bT\bY\|_F$,
\begin{align*}
\|\bP\bT\bY\|_F
&=
\|\mu \bP\bY + \bP\bDelta\bY\|_F \\
&\le
\mu\|\bP\bY\|_F + \|\bP\|_2\|\bDelta\|_2\|\bY\|_F \\
&\le
\mu(\sqrt{t}+\eta)\|\bY\|_F,
\end{align*}
and similarly
\begin{align*}
\|\bP\bT\bY\|_F
&\ge
\mu\|\bP\bY\|_F - \|\bP\|_2\|\bDelta\|_2\|\bY\|_F \\
&\ge
\mu(\sqrt{t}-\eta)\|\bY\|_F.
\end{align*}
Since $\|\bP\bT\bY\|_F \ge 0$, we can rewrite the lower bound as
\[
\|\bP\bT\bY\|_F
\ge
\mu \max\{\sqrt{t}-\eta,0\}\|\bY\|_F.
\]
Squaring the numerator bounds and dividing by the squared denominator bounds gives
\[
\frac{\bigl(\max\{\sqrt{t}-\eta,0\}\bigr)^2}{(1+\eta)^2}
\le
\frac{\|\bP\bT\bY\|_F^2}{\|\bT\bY\|_F^2}
\le
\frac{(\sqrt{t}+\eta)^2}{(1-\eta)^2},
\]
which is exactly the claimed estimate for $\tsim{\bT\bY}$.
\end{proof}

\paragraph{Interpretation for RMSNorm.}
For RMSNorm, $\bT$ is the diagonal matrix with entries
\[
\bT_{i,i} = \frac{1}{\sqrt{\|\bY_{i,:}\|_2^2/d+\epsilon}}.
\]
Therefore
\[
\eta = \frac{\|\bT-\mu\bI\|_2}{\mu}
\]
measures the spread of the RMS factors around their mean scalar $\mu$. When this spread satisfies $\eta \ll \sqrt{\tsim{\bY}}$, Theorem~\ref{theorem_attention_rms} keeps $\tsim{\bT\bY}$ close to $\tsim{\bY}$.


\subsection{Justification of Ratio-of-Expectations Surrogate}\label{subsection_concentration}
This subsection checks the ratio-of-expectations approximation after the matrices
before the last Gaussian projection are fixed. For the attention update, we
condition on $(\bX,\bP)$ and take probability only over $\bW$. For the SwiGLU
update, we condition on $(\bX,\bR_s)$ and take probability only over $\bW_2$.
The closed-form centers used in the forward theorem are obtained by applying the
same equal-correlation, prefix-averaging, and SwiGLU moment closures used in the
preceding derivations.

For matrix updates of the form $\bB = \bA + \bM(\bA)$, define
\[
\xi_1 := \frac{\frobsq{\simmat\bB}}{\frobsq{\simmat\bA}},
\qquad
\xi_2 := \frac{\frobsq{\bB}}{\frobsq{\bA}}.
\]
Then
\[
\tsim{\bB} = \frac{\xi_1}{\xi_2}\,\tsim{\bA}.
\]
The next lemma controls the difference between $\xi_1/\xi_2$ and the ratio of
conditional expectations after conditioning on $(\bX,\bP)$.

\begin{lemma}\label{lemma_eta_bound}
Condition on $(\bX,\bP)$ and let
\[
m_i := \mathbb{E}[\xi_i\mid \bX,\bP],\qquad i\in\{1,2\}.
\]
Assume $m_1,m_2 > 0$. Then for every $\varepsilon > 0$,
\[
\mathbb{P}\!\left(
\left|\frac{\xi_1}{\xi_2} - \frac{m_1}{m_2}\right| \ge \varepsilon
\;\middle|\; \bX,\bP
\right)
\le
\mathbb{P}\!\left(
|\xi_2 - m_2| \ge \frac{m_2}{2}
\;\middle|\; \bX,\bP
\right)
+
\mathbb{P}\!\left(
\max_{i=1,2} |\xi_i - m_i| \ge \gamma \varepsilon
\;\middle|\; \bX,\bP
\right),
\]
where
\[
\gamma := \min\left\{\frac{m_2}{4}, \frac{m_2^2}{4m_1}\right\}.
\]
\end{lemma}

\begin{proof}
Let
\[
\mathcal{E}_0 := \left\{ |\xi_2 - m_2| \le \frac{m_2}{2} \right\}.
\]
On $\mathcal{E}_0$, we have $\xi_2 \ge m_2/2$, so
\begin{align*}
\left|\frac{\xi_1}{\xi_2} - \frac{m_1}{m_2}\right|
&=
\left|
\frac{m_2(\xi_1 - m_1) - m_1(\xi_2 - m_2)}{\xi_2 m_2}
\right| \\
&\le
\frac{2}{m_2} |\xi_1 - m_1|
+
\frac{2m_1}{m_2^2} |\xi_2 - m_2|.
\end{align*}
Therefore,
\[
\left\{
\left|\frac{\xi_1}{\xi_2} - \frac{m_1}{m_2}\right| \ge \varepsilon
\right\}
\cap \mathcal{E}_0
\subseteq
\left\{
|\xi_1 - m_1| \ge \frac{\varepsilon m_2}{4}
\right\}
\cup
\left\{
|\xi_2 - m_2| \ge \frac{\varepsilon m_2^2}{4m_1}
\right\}.
\]
Taking conditional probabilities given $(\bX,\bP)$ and using the definition of
$\gamma$ gives the claim.
\end{proof}

Therefore, after conditioning on $(\bX,\bP)$, it suffices to show concentration
for $\xi_1$ and $\xi_2$ around
$\mathbb{E}[\xi_1\mid\bX,\bP]$ and $\mathbb{E}[\xi_2\mid\bX,\bP]$. We analyze
\[
\xi_1 := \frac{\|\simmat\bY\|_F^2}{\|\simmat\bX\|_F^2},
\qquad
\xi_2 := \frac{\|\bY\|_F^2}{\|\bX\|_F^2}.
\]
Under Assumption~\ref{assumption_on_X}, $\|\simmat\bX\|_F^2 = t\,nd$ and $\|\bX\|_F^2 = nd$. The following argument uses three standard concentration inequalities.

\begin{lemma}
    (General Hoeffding’s Inequality) Let $X_1, \dots, X_n$ be independent, zero-mean random variables such that each $X_i$ is \textbf{sub-Gaussian} with norm $\|X_i\|_{\psi_2} \leq K_i$. Let $S = \sum_{i=1}^n X_i$ and $v^2 = \sum_{i=1}^n K_i^2$.

Then for all $\epsilon \geq 0$:
$$\mathbb{P}(|S| \geq \epsilon) \leq 2\exp\left(-\frac{c\epsilon^2}{v^2}\right)$$
where $c > 0$ is an absolute constant (one can take $c = 1/2$ for Gaussians).
\end{lemma}

\begin{lemma}
    (Bernstein’s Inequality) Let $X_1, \dots, X_n$ be independent, zero-mean random variables such that each $X_i$ is \textbf{sub-exponential} with norm $\|X_i\|_{\psi_1} \leq K$. Let $S = \sum_{i=1}^n X_i$.

Then for all $\epsilon \geq 0$:
\begin{align*}
    \mathbb{P}(|S| \geq \epsilon) \leq 2\exp\left(-c \cdot \min\left\{\frac{\epsilon^2}{nK^2}, \frac{\epsilon}{K}\right\}\right)
\end{align*}
\end{lemma}

\begin{lemma}
    (Hanson-Wright Inequality) Let $\bx = (x_1, \dots, x_n) \in \mathbb{R}^n$ be a random vector with independent, zero-mean, \textbf{sub-Gaussian} components satisfying $\max_i \|x_i\|_{\psi_2} \leq K$. Let $\bA \in \mathbb{R}^{n \times n}$ be a fixed matrix (not necessarily symmetric or PSD).

    Then for all $\epsilon \geq 0$:
    \begin{align*}
        \mathbb{P}\left(\left|\bx^\top \bA \bx - \mathbb{E}[\bx^\top \bA \bx]\right| \geq \epsilon\right) \leq 2\exp\left(-c \cdot \min\left\{\frac{\epsilon^2}{K^4\|\bA\|_F^2}, \frac{\epsilon}{K^2\|\bA\|_2}\right\}\right)
    \end{align*}
\end{lemma}

\textbf{Remark}: Unlike Hoeffding/Bernstein, which apply to sums of independent terms, Hanson-Wright applies to quadratic forms $\sum_{i,j} A_{ij} X_i X_j$ in which the summands are dependent.

\begin{lemma}\label{lemma_concentration_attention}
Let $\bY = \bX + \bP\bX\bW$, where the entries of $\bW$ are i.i.d.\ Gaussian random variables with mean $0$ and variance $\sigma_{\bW}^2$. Define
\[
\mu_1 := \frac{\|\simmat\bP\bX\|_F}{\|\simmat\bX\|_F}.
\]
Conditioned on $(\bX,\bP)$, for every $\varepsilon > 0$,
\begin{align}
\mathbb{P}\!\left(
\left|\xi_2 - \mathbb{E}[\xi_2 \mid \bX,\bP]\right| \ge \varepsilon
\;\middle|\; \bX,\bP
\right)
&\le
4\exp\!\left(
-c \min\left\{
\frac{\varepsilon^2\|\bX\|_F^2}{\sigma_{\bW}^2\|\bP\bX\|_F^2},
\frac{\varepsilon^2 \|\bX\|_F^4}{d \sigma_{\bW}^4 \|\bP\bX\|_F^4},
\frac{\varepsilon \|\bX\|_F^2}{\sigma_{\bW}^2 \|\bP\bX\|_F^2}
\right\}
\right), \label{eq:attn-xi2-cond} \\
\mathbb{P}\!\left(
\left|\xi_1 - \mathbb{E}[\xi_1 \mid \bX,\bP]\right| \ge \varepsilon
\;\middle|\; \bX,\bP
\right)
&\le
4\exp\!\left(
-c \min\left\{
\frac{\varepsilon^2\|\simmat\bX\|_F^2}{\sigma_{\bW}^2\|\simmat\bP\bX\|_F^2},
\frac{\varepsilon^2 \|\simmat\bX\|_F^4}{d \sigma_{\bW}^4 \|\simmat\bP\bX\|_F^4},
\frac{\varepsilon \|\simmat\bX\|_F^2}{\sigma_{\bW}^2 \|\simmat\bP\bX\|_F^2}
\right\}
\right). \label{eq:attn-xi1-cond}
\end{align}
Moreover,
\[
\mathbb{E}[\xi_2 \mid \bX,\bP]
=
1 + \sigma_{\bW}^2 d \frac{\|\bP\bX\|_F^2}{\|\bX\|_F^2},
\qquad
\mathbb{E}[\xi_1 \mid \bX,\bP]
=
1 + \sigma_{\bW}^2 d\,\mu_1^2.
\]
\end{lemma}

\begin{proof}
For $\xi_2$, write
\[
\xi_2
=
\frac{\|\bX + \bP\bX\bW\|_F^2}{\|\bX\|_F^2}
=
1 + L_2 + Q_2,
\]
where
\[
L_2 := \frac{2\langle \bX,\bP\bX\bW\rangle}{\|\bX\|_F^2},
\qquad
Q_2 := \frac{\|\bP\bX\bW\|_F^2}{\|\bX\|_F^2}.
\]
Conditioned on $(\bX,\bP)$, the term $L_2$ is centered Gaussian with
\[
\mathrm{Var}(L_2 \mid \bX,\bP)
=
\frac{4\sigma_{\bW}^2\|(\bP\bX)^\top \bX\|_F^2}{\|\bX\|_F^4}
\le
\frac{4\sigma_{\bW}^2\|\bP\bX\|_F^2}{\|\bX\|_F^2},
\]
so Gaussian tails give
\[
\mathbb{P}\!\left(
|L_2| \ge \varepsilon
\;\middle|\; \bX,\bP
\right)
\le
2\exp\!\left(
-c \frac{\varepsilon^2\|\bX\|_F^2}{\sigma_{\bW}^2\|\bP\bX\|_F^2}
\right).
\]
Also conditioned on $(\bX,\bP)$, the $d$ columns of $\bW$ are independent Gaussian vectors, so $Q_2$ is a sum of $d$ independent quadratic forms. Hanson-Wright and Bernstein therefore imply
\[
\mathbb{P}\!\left(
\left|Q_2 - \mathbb{E}[Q_2 \mid \bX,\bP]\right| \ge \varepsilon
\;\middle|\; \bX,\bP
\right)
\le
2\exp\!\left(
-c \min\left\{
\frac{\varepsilon^2 \|\bX\|_F^4}{d \sigma_{\bW}^4 \|\bP\bX\|_F^4},
\frac{\varepsilon \|\bX\|_F^2}{\sigma_{\bW}^2 \|\bP\bX\|_F^2}
\right\}
\right).
\]
The conditional expectation of $\xi_2$ is
\[
\mathbb{E}[\xi_2 \mid \bX,\bP]
=
1 + \mathbb{E}[Q_2 \mid \bX,\bP],
\]
so the deviation from this conditional expectation is
\[
\xi_2-\mathbb{E}[\xi_2\mid \bX,\bP]
=
L_2+\bigl(Q_2-\mathbb{E}[Q_2\mid \bX,\bP]\bigr).
\]
The event
\[
\left\{|\xi_2-\mathbb{E}[\xi_2\mid \bX,\bP]|\ge \varepsilon\right\}
\]
is contained in
\[
\left\{|L_2|\ge \varepsilon/2\right\}
\cup
\left\{|Q_2-\mathbb{E}[Q_2\mid \bX,\bP]|\ge \varepsilon/2\right\}.
\]
The one-dimensional Gaussian tail inequality controls $L_2$, and the
Hanson-Wright/Bernstein inequality controls
$Q_2-\mathbb{E}[Q_2\mid \bX,\bP]$. Combining these two probability bounds gives
Eq.~\eqref{eq:attn-xi2-cond}, after replacing the numerical constant $c$ in the
exponent by a smaller numerical constant.

The argument for $\xi_1$ is identical after replacing $\bX$ by $\simmat\bX$ and $\bP\bX$ by $\simmat\bP\bX$. This gives Eq.~\eqref{eq:attn-xi1-cond} and the formula
\[
\mathbb{E}[\xi_1 \mid \bX,\bP]
=
1 + \sigma_{\bW}^2 d \frac{\|\simmat\bP\bX\|_F^2}{\|\simmat\bX\|_F^2}
=
1 + \sigma_{\bW}^2 d\,\mu_1^2.
\]
\end{proof}

Under Assumption~\ref{assumption_on_X}, $\|\bX\|_F^2 = nd$ and $\|\simmat\bX\|_F^2 = t\,nd$. Under the prefix-averaging surrogate $\bP \approx \bC_n$,
\[
\frac{\|\bP\bX\|_F^2}{\|\bX\|_F^2}
=
\rho + \frac{1-\rho}{n} H_n,
\]
which is $O(1)$ in the initialization regime. If in addition $t \ge t_0 > 0$ and $\mu_1 \le C_1$, then both factors concentrate at width scale.

For the SwiGLU update, we use the same ratio lemma after replacing the
conditioning variables $(\bX,\bP)$ by $(\bX,\bR_s)$.

\begin{lemma}\label{lem:xi-swiglu-v2}
For the SwiGLU update
\[
\bY = \bX + \bigl(\silu(\bX\bW_1)\odot \bX\bW_3\bigr)\bW_2
= \bX + \bR_s\bW_2,
\]
assume  $\bR_s$ satisfies
\[
\|\bR_s\|_F^2 \le 2\mathbb{E}\|\bR_s\|_F^2,
\qquad
\|\simmat\bR_s\|_F^2 \le 2\mathbb{E}\|\simmat\bR_s\|_F^2.
\]
Then, conditioned on any fixed pair $(\bX,\bR_s)$ satisfying the two displayed
bounds, for every $\varepsilon > 0$,
\begin{align}
\mathbb{P}\!\left(
\left|\xi_2 - \mathbb{E}[\xi_2 \mid \bX,\bR_s]\right| \ge \varepsilon
\;\middle|\; \bX,\bR_s
\right)
&\le
4\exp\!\left(
-c \min\left\{
\frac{\varepsilon^2\|\bX\|_F^2}{\sigma_{\bW_2}^2\|\bR_s\|_F^2},
\frac{\varepsilon^2 \|\bX\|_F^4}{d \sigma_{\bW_2}^4 \|\bR_s\|_F^4},
\frac{\varepsilon \|\bX\|_F^2}{\sigma_{\bW_2}^2 \|\bR_s\|_F^2}
\right\}
\right), \label{eq:ffn-xi2-cond} \\
\mathbb{P}\!\left(
\left|\xi_1 - \mathbb{E}[\xi_1 \mid \bX,\bR_s]\right| \ge \varepsilon
\;\middle|\; \bX,\bR_s
\right)
&\le
4\exp\!\left(
-c \min\left\{
\frac{\varepsilon^2\|\simmat\bX\|_F^2}{\sigma_{\bW_2}^2\|\simmat\bR_s\|_F^2},
\frac{\varepsilon^2 \|\simmat\bX\|_F^4}{d \sigma_{\bW_2}^4 \|\simmat\bR_s\|_F^4},
\frac{\varepsilon \|\simmat\bX\|_F^2}{\sigma_{\bW_2}^2 \|\simmat\bR_s\|_F^2}
\right\}
\right). \label{eq:ffn-xi1-cond}
\end{align}
\end{lemma}

\begin{proof}
Condition on a fixed pair $(\bX,\bR_s)$ satisfying the two displayed bounds.
Expanding the squared Frobenius norm in $\xi_2$ gives
\[
\xi_2-\mathbb{E}[\xi_2\mid \bX,\bR_s]
=
\frac{2\langle \bX,\bR_s\bW_2\rangle}{\|\bX\|_F^2}
+
\left(
\frac{\|\bR_s\bW_2\|_F^2}{\|\bX\|_F^2}
-
\mathbb{E}\left[
\frac{\|\bR_s\bW_2\|_F^2}{\|\bX\|_F^2}
\mid \bX,\bR_s
\right]
\right).
\]
Expanding the squared Frobenius norm in $\xi_1$ gives
\[
\xi_1-\mathbb{E}[\xi_1\mid \bX,\bR_s]
=
\frac{2\langle \simmat\bX,\simmat\bR_s\bW_2\rangle}{\|\simmat\bX\|_F^2}
+
\left(
\frac{\|\simmat\bR_s\bW_2\|_F^2}{\|\simmat\bX\|_F^2}
-
\mathbb{E}\left[
\frac{\|\simmat\bR_s\bW_2\|_F^2}{\|\simmat\bX\|_F^2}
\mid \bX,\bR_s
\right]
\right).
\]
Conditioned on $(\bX,\bR_s)$, the two inner-product terms are one-dimensional
Gaussian random variables. The two quadratic terms are sums of Gaussian quadratic
forms. Applying the one-dimensional Gaussian tail inequality to the two
inner-product terms, applying the Hanson-Wright/Bernstein inequality to the two
quadratic terms minus their conditional expectations, and then applying a union
bound gives Eqs.~\eqref{eq:ffn-xi2-cond} and \eqref{eq:ffn-xi1-cond}, after
replacing the numerical constant $c$ in the exponent by a smaller numerical
constant.
\end{proof}

Under the standard initialization and $\dff=\Theta(d)$, the moment calculations
in Appendix~\ref{app:proof_swiglu_forward} give the closed-form conditional
centers used for the SwiGLU branch. Lemma~\ref{lem:xi-swiglu-v2} then controls
the remaining randomness from the output projection $\bW_2$ around these centers.

\section{Backward Repair Incapacity Details and Proofs}\label{app:backward_bottleneck_details}

This section supplies the technical steps behind the backward analysis in the main text. It first proves the gradient-contraction bound, then shows that the same contraction persists near collapse, and finally proves the exact-collapse residual-norm growth statement that drives the RMSNorm factor below one.

\subsection{Proof of Theorem~\ref{theorem_iterative_bound_on_gradient_norm}}
This subsection proves the gradient bound in Theorem~\ref{theorem_iterative_bound_on_gradient_norm}. Its role is to describe how RMSNorm shrinkage factor and the sublayer contribution factors $\alpha_1,\alpha_2$ affect the gradient norm. 

\begin{proof}
\begin{align*}
    \twonorm{\pardir{\caL}{(\bY_1)_{i, :}}}
    &=
    \twonorm{\frac{1}{\sqrt{\frac{1}{d} \frobsq{\by_1}}} \times \pardir{\caL}{(\bX_2)_{i, :}} \left(\bI - \frac{\by_1\by_1^\top}{\frobsq{\by_1}}\right)} \\
    &\leq \frac{1}{\sqrt{\frac{1}{d} \frobsq{\by_1}}}\twonorm{\pardir{\caL}{(\bX_2)_{i, :}}}\times \twonorm{\bI - \frac{\by_1\by_1^\top}{\frobsq{\by_1}}} \\
    &= \frac{1}{\sqrt{\frac{1}{d} \frobsq{\by_1}}}\twonorm{\pardir{\caL}{(\bX_2)_{i, :}}}.
\end{align*}
Squaring and summing over rows yields
\begin{align*}
    \frobsq{\pardir{\caL}{\bY_1}}
    &\leq \frac{1}{\frac{1}{d} \frobsq{\by_1}}\frobsq{\pardir{\caL}{\bX_2}}.
\end{align*}
Taking square roots gives
\begin{align*}
    \frob{\pardir{\caL}{\bY_1}}
    &\leq \frac{1}{\sqrt{\frac{1}{d} \frobsq{\by_1}}}\frob{\pardir{\caL}{\bX_2}}
    = \frac{\sqrt{d}}{\twonorm{\by_1}}\frob{\pardir{\caL}{\bX_2}}.
\end{align*}
By Definition~\ref{assumption_sublayer_amplification}, $\frob{\pardir{\caL}{\bX_1}} \leq \alpha_1 \frob{\pardir{\caL}{\bY_1}}$, so
\[
\frob{\pardir{\caL}{\bX_1}} \leq \alpha_1\frac{\sqrt{d}}{\twonorm{\by_1}}\frob{\pardir{\caL}{\bX_2}}.
\]
The same argument for the FFN sublayer gives
\[
\frob{\pardir{\caL}{\bX_2}} \leq \alpha_2\frac{\sqrt{d}}{\twonorm{\by_2}}\frob{\pardir{\caL}{\bX_1^{l+1}}}.
\]
\end{proof}

\subsection{Near-Collapse Extension of Theorem~\ref{theorem_iterative_bound_on_gradient_norm}}
\label{app:backward_repair_bottleneck_near_collapse}
This subsection extends the exact-collapse contraction statement to the case $\tsim{\bY_k}=1-\delta$ with small $\delta$. Its role is to show that the exact-collapse threshold used in the backward theorem still describes a neighborhood of the collapsed state rather than a single isolated point.

\begin{proposition}\label{prop_backward_repair_incapacity_near_collapse}
Let $\tsim{\bY_k} = 1 - \delta$ with $0<\delta\le 1/(4n+1)$. Decompose
\[
\bY_k = \done_n\bar\by_k^\top + \bF,
\qquad
\bar\by_k := \frac{1}{n}\bY_k^\top\done_n,
\qquad
\bF := (\bI - \simmat)\bY_k.
\]
Let $c_0 = \sqrt{d}/\|\bar\by_k\|_2$ be the exact-collapse RMSNorm shrinkage factor. If $\delta \leq 1/(4n+1)$, then the per-row shrinkage satisfies
\[
\frac{\sqrt{d}}{\|\by_{k,i}\|_2}
\leq
c_0\!\left(1 + 2\sqrt{\frac{n\delta}{1-\delta}}\right)
\qquad \text{for every } i.
\]
Consequently, if the corresponding sublayer gradient contribution factor is bounded by $\alpha_k$, then
\[
\left\|\frac{\partial \caL}{\partial \bX_k}\right\|_F \leq c_0\!\left(1 + O(\sqrt{n\delta})\right)\alpha_k\left\|\frac{\partial \caL}{\partial \bX_{k+1}}\right\|_F.
\]
If $c_0\alpha_k \leq 1-\eta$ for some $\eta>0$, then for sufficiently small
$\delta>0$,
\[
c_0(1+O(\sqrt{n\delta}))\alpha_k < 1,
\]
so the gradient norm contracts from $\bX_{k+1}$ to $\bX_k$ in the
near-collapse regime.
\end{proposition}

\begin{proof}
Since $\simmat\bY_k = \done_n\bar\by_k^\top$ and $(\bI-\simmat)\bY_k=\bF$, the token-similarity definition gives
\[
1-\delta
=
\tsim{\bY_k}
=
\frac{\|\simmat\bY_k\|_F^2}{\|\bY_k\|_F^2}
=
\frac{n\|\bar\by_k\|_2^2}{n\|\bar\by_k\|_2^2 + \|\bF\|_F^2}.
\]
Therefore
\[
\|\bF\|_F^2 = \frac{n\delta}{1-\delta}\|\bar\by_k\|_2^2.
\]
Let $\bff_i^\top$ denote the $i$-th row of $\bF$. Then
\[
\|\bff_i\|_2 \leq \|\bF\|_F = \|\bar\by_k\|_2\sqrt{\frac{n\delta}{1-\delta}}.
\]
By the reverse triangle inequality,
\[
\|\by_{k,i}\|_2 = \|\bar\by_k + \bff_i\|_2
\geq
\|\bar\by_k\|_2 - \|\bff_i\|_2
\geq
\|\bar\by_k\|_2\left(1-\sqrt{\frac{n\delta}{1-\delta}}\right).
\]
Under $\delta \leq 1/(4n+1)$, we have $\sqrt{\frac{n\delta}{1-\delta}} \leq 1/2$, so the lower bound is at least $\|\bar\by_k\|_2/2 > 0$. Hence
\[
\frac{\sqrt{d}}{\|\by_{k,i}\|_2}
\leq
\frac{\sqrt{d}}{\|\bar\by_k\|_2\left(1-\sqrt{\frac{n\delta}{1-\delta}}\right)}
=
c_0\cdot \frac{1}{1-u}
\leq
c_0(1+2u),
\]
where $u := \sqrt{\frac{n\delta}{1-\delta}}$ and we used $1/(1-u) \leq 1+2u$ for $u \leq 1/2$. The RMSNorm Jacobian at row $i$ has spectral norm $\sqrt{d}/\|\by_{k,i}\|_2$ (the projection factor $\bI - \by_{k,i}\by_{k,i}^\top/\|\by_{k,i}\|_2^2$ has spectral norm $1$), so the full block-diagonal Jacobian has operator norm
\[
\max_i \frac{\sqrt{d}}{\|\by_{k,i}\|_2}
\leq
c_0\!\left(1 + O(\sqrt{n\delta})\right).
\]
The gradient bound follows by the chain rule and submultiplicativity.
\end{proof}

\subsection{Residual-Stream Norm Growth at Exact Collapse}
\label{app:residual_norm_growth_details}
\label{app:lemma3a_norm_growth}

This subsection proves the exact-collapse norm-growth step used in the backward analysis. We first show it for the attention output matrix $\bW_O$, then for the FFN output matrix $\bW_2$.

\begin{proposition}[Residual-Stream Norm Growth at Exact Collapse]
\label{prop_residual_stream_norm_growth}
Consider one Post-Norm block as defined in Equation~\ref{eq:attention_ffn_definition}. Let
\[
\bG_1 := \pardir{\caL}{\bY_1},
\qquad
\bG_2 := \pardir{\caL}{\bY_2},
\qquad
\bH_1 := \bP\bX_1\bW_V,
\qquad
\bR_s := \silu(\bX_2\bW_1)\odot(\bX_2\bW_3).
\]
Suppose exact collapse holds at the attention and FFN sites, so
\[
\bX_1=\done_n\bx_1^\top,
\bY_1=\done_n\by_1^\top,
\qquad
\bX_2=\done_n\bx_2^\top,
\bY_2=\done_n\by_2^\top
\]
for some $\bx_1,\by_1,\bx_2,\by_2\in\bbR^d$. After one gradient step of size $\eta$ on the attention output matrix $\bW_O$, the residual stream becomes
\[
\bY_{1,W_O}(\eta)=\bY_1-\eta\bU_1,
\qquad
\bU_1:=\bH_1\bH_1^\top\bG_1.
\]
After one gradient step of size $\eta$ on the FFN output matrix $\bW_2$, the residual stream becomes
\[
\bY_{2,W_2}(\eta)=\bY_2-\eta\bU_2,
\qquad
\bU_2:=\bR_s\bR_s^\top\bG_2.
\]
Moreover,
\[
\frobsq{\bY_{1,W_O}(\eta)}=\frobsq{\bY_1}+\eta^2\frobsq{\bU_1},
\qquad
\frobsq{\bY_{2,W_2}(\eta)}=\frobsq{\bY_2}+\eta^2\frobsq{\bU_2}.
\]
In particular, whenever $\bU_1\neq 0$ or $\bU_2\neq 0$, the corresponding pre-RMS residual norm grows strictly for every $\eta>0$.
\end{proposition}

\begin{proof}
We first derive the RMSNorm orthogonality identity used below. If $\bX = \rms(\bY) = \bT(\bY)\bY$ with
\[
\bT_{i,i} = \frac{1}{\sqrt{\frac{1}{d}\|\bY_{i,:}\|_2^2}},
\]
then the row-wise Jacobian gives
\[
\pardir{\caL}{\bY_{i,:}}
=
\bT_{i,i}
\pardir{\caL}{\bX_{i,:}}
\left(\bI - \frac{\bY_{i,:}^\top\bY_{i,:}}{\|\bY_{i,:}\|_2^2}\right).
\]
Therefore each gradient row is orthogonal to the corresponding residual row:
\[
\pardir{\caL}{\bY_{i,:}}\bY_{i,:}^\top = 0.
\]
Under exact collapse, all residual rows are identical, so this implies
\[
\pardir{\caL}{\bY}\bY^\top = \bzero.
\]

Now we consider the attention output matrix $\bW_O$.
At exact collapse $\tsim{\bX_1}=1$, write $\bX_1=\done_n\bx_1^\top$. Since $\bP\done_n=\done_n$, the attention feature before $\bW_O$ is
\[
\bH_1 :=\bP\bX_1\bW_V=\done_n\bh_1^\top
\]
for some $\bh_1\in\bbR^d$. One gradient step on $\bW_O$ gives
\[
\bW_O(\eta)=\bW_O-\eta\bH_1^\top\bG_1,
\]
so the residual stream update is
\[
\bY_{1,W_O}(\eta)=\bY_1-\eta\bH_1\bH_1^\top\bG_1=\bY_1-\eta\bU_1.
\]
Therefore
\[
\frobsq{\bY_{1,W_O}(\eta)} = \frobsq{\bY_1} - 2\eta\inner{\bY_1}{\bU_1} + \eta^2\frobsq{\bU_1}.
\]
Because $\bH_1=\done_n\bh_1^\top$, the matrix $\bU_1=\bH_1\bH_1^\top\bG_1$ also has identical rows. Applying the orthogonality identity above with $\bY=\bY_1$ gives $\bG_1\bY_1^\top=\bzero$, so $\inner{\bY_1}{\bU_1}=0$. Therefore
\[
\frobsq{\bY_{1,W_O}(\eta)} = \frobsq{\bY_1} + \eta^2\frobsq{\bU_1}.
\]

Second, consider the FFN output matrix $\bW_2$.
At exact collapse $\tsim{\bX_2}=1$, write $\bX_2=\done_n\bx_2^\top$. Then
\[
\bR_s=\silu(\bX_2\bW_1)\odot(\bX_2\bW_3)=\done_n\br_s^\top
\]
for some $\br_s\in\bbR^{d_{\mathrm{ff}}}$. One gradient step on $\bW_2$ gives
\[
\bW_2(\eta)=\bW_2-\eta\bR_s^\top\bG_2,
\]
so the residual stream update is
\[
\bY_{2,W_2}(\eta)=\bY_2-\eta\bR_s\bR_s^\top\bG_2=\bY_2-\eta\bU_2.
\]
The same expansion gives
\[
\frobsq{\bY_{2,W_2}(\eta)} = \frobsq{\bY_2} - 2\eta\inner{\bY_2}{\bU_2} + \eta^2\frobsq{\bU_2}.
\]
Because $\bR_s=\done_n\br_s^\top$, the matrix $\bU_2=\bR_s\bR_s^\top\bG_2$ also has identical rows. Applying the same orthogonality identity with $\bY=\bY_2$ gives $\bG_2\bY_2^\top=\bzero$, so $\inner{\bY_2}{\bU_2}=0$, hence
\[
\frobsq{\bY_{2,W_2}(\eta)}=\frobsq{\bY_2}+\eta^2\frobsq{\bU_2}.
\]
The strict-growth statement follows immediately.
\end{proof}

\paragraph{Remark.}
The residual stream norm appears in the contraction factor is $c(\by,\alpha)=\alpha\sqrt{d}/\|\by\|_2$ in Theorem~\ref{theorem_iterative_bound_on_gradient_norm}, which is less than 1 when $\|\by\|_2^2>d\alpha^2$. The exact-collapse calculation above shows a monotone direction: for any nonzero attention-path update, one gradient step strictly increases the pre-normalization residual norm, making the RMS factor smaller along that local path.

\section{Properties of a Collapsed Network: Details and Proofs}\label{app:collapsed_properties_details}

This section proves two properties of collapsed state. It first identifies the frequency distribution as the best prediction under exact collapse, then extends the loss bound beyond the exact rank-one setting, and finally shows why gradients vanish when the hidden state is rank one and the output matches that shared distribution.

\subsection{Proof of Theorem~\ref{theorem_frequency_distribution} (Lower Bound)}
This subsection proves the collapsed loss lower bound in Theorem~\ref{theorem_frequency_distribution}. Its role is to show that once all token representations coincide, the best achievable output distribution is the empirical frequency distribution of the labels.
\begin{proof}
Under exact collapse, the output probability matrix has identical rows. Let $\hat{\bp}$ be the common output distribution, so
\[
\hat\bP = \done \hat{\bp}^\top.
\]
If label $i$ appears $\beta_i$ times in $\by$, the cross-entropy loss becomes
\begin{equation*}
    -\frac{1}{n}\sum_{i = 1}^n \log \hat\bP_{i, y_i} = -\frac{1}{n}\sum_{i = 1}^v \beta_i\log \hat p_i.
\end{equation*}
This is only a regrouping by label counts.
To find the optimum, solve
\begin{align*}
    \min_{\hat{\bp}} \quad & -\sum_{i = 1}^v \beta_i\log \hat p_i \\
    \text{s.t.} \quad & \sum_{i = 1}^v \hat p_i = 1
\end{align*}
We define the Lagrangian with multiplier $\mu$:
\begin{equation*}
    \lambda (\hat{\bp}, \mu) = -\sum_{i = 1}^v \beta_i\log \hat p_i + \mu\left(1 - \sum_{i = 1}^v \hat p_i \right)
\end{equation*}
Applying the KKT conditions ($\partial \lambda/\partial \hat p_i = 0$):
\begin{equation*}
    -\frac{\beta_i}{\hat p_i} - \mu = 0 \implies \hat p_i = -\frac{\beta_i}{\mu}
\end{equation*}
Summing over $i$ and enforcing the constraint $\sum \hat p_i = 1$:
\begin{equation*}
    \sum_{i=1}^v \hat p_i = -\frac{1}{\mu}\sum_{i=1}^v \beta_i = 1 \implies \mu = -\sum_{i=1}^v \beta_i
\end{equation*}
Substituting $\mu$ back yields the optimal $\hat{\bp}$:
\begin{equation*}
    \hat p_i = \frac{\beta_i}{\sum_{j = 1}^v \beta_j}
\end{equation*}
This is exactly the frequency distribution of $\by$. 
\end{proof}

\subsection{Near-collapse extension of Theorem~\ref{theorem_frequency_distribution} (Lower Bound)}
We now consider the near-collapsed case and compare the resulting loss to the frequency-distribution optimum. While $\caL(\bX,\wlm,\by)$ is not jointly convex, it is convex in $\bX$ and in $\wlm$ separately. To see this: cross-entropy is convex in the logit matrix $\bZ$, and since $\bZ = \bX\wlm$ is linear in each factor when the other is fixed, the loss is convex in $\bX$ for any fixed $\wlm$, and convex in $\wlm$ for any fixed $\bX$. This follows because cross-entropy is convex in logits:
\begin{align*}
    \tlower{\caL}{CE}(\bZ, \by) &= -\frac{1}{n}\sum_{i = 1}^n (\bZ_{i, y_i} - \log(\sum_{j = 1}^v  \exp(\bZ_{i,j}))) \\
    &= \frac{1}{n}\sum_{i = 1}^n (-\bZ_{i, y_i} + \log(\sum_{j = 1}^v  \exp(\bZ_{i,j})))
\end{align*}
The log-sum-exp term is convex and the remaining term is linear, so the objective is convex in $\bZ$. Since $\bZ=\bX\wlm$ is linear in each variable when the other is fixed, the objective is convex in $\bX$ (resp. $\wlm$) conditionally.

The gradients of $\caL$ w.r.t. $\wlm$, $\bX$ are
\begin{align*}
    \frac{\partial \caL}{\partial \wlm} &= \frac{1}{n}\bX^T(\softmax(\bX\wlm) - \bE(\by)) =  \bX^T\frac{\partial \caL}{\partial \bZ}\\
    \frac{\partial \caL}{\partial \bX} &= \frac{1}{n}(\softmax(\bX\wlm) - \bE(\by))\wlm^T = \frac{\partial \caL}{\partial \bZ}\wlm^T
\end{align*}
where $\bE(\by) \in \bR^{n \times v}$ is the one-hot matrix for labels $\by$
\begin{align*}
    \bE(\by)_{i, j} = \begin{cases}
        1, \quad \text{If } j = y_i \\
        0, \quad \text{Otherwise }
    \end{cases}
\end{align*}
We have the following inequality
\begin{align*}
    \|\frac{\partial \caL}{\partial \bX}\|_F &= \frac{1}{n}\|(\softmax(\bX\wlm) - \bE(\by))\wlm^T\|_F \\
     &\leq \frac{1}{n}\|(\softmax(\bX\wlm) - \bE(\by))\|_F\|\wlm\|_2 \\
     &\leq \frac{1}{n}(\|\softmax(\bX\wlm)\|_F + \|\bE(\by)\|_F)\|\wlm\|_2 \\
     &\leq \frac{1}{n}(\sqrt{n} + \sqrt{n})\|\wlm\|_2 \\
     &= \frac{2}{\sqrt{n}}\|\wlm\|_2 
\end{align*}
The fourth line uses that $\bE(\by)$ has exactly $n$ unit entries and each row of $\softmax(\bX\wlm)$ is a probability vector, hence row norm at most 1.
Let $\bar \bX = \simmat \bX$. 
Since the function is convex w.r.t. $\bX$ when fixing $\wlm$ and $\by$, by subgradient inequality, we have the following inequalities
\begin{align*}
     \caL(\bX, \wlm, \by) &\geq \caL(\bar \bX, \wlm, \by) + \langle\frac{\partial \caL}{\partial \bX}(\bar \bX, \wlm, \by), \bX - \bar \bX\rangle \\
    &\geq \caL(\bar \bX, \wlm, \by)  - \|\frac{\partial \caL}{\partial \bX}(\bar \bX, \wlm, \by)\|_F \|\bX - \bar \bX\|_F \\
    &\geq \caL(\bar \bX, \wlm, \by)  -  \frac{2}{\sqrt{n}}\|\wlm\|_2 \epsilon
\end{align*}
The second inequality is Cauchy-Schwarz. The last uses $\|\frac{\partial \caL}{\partial \bX}\|_F \leq \frac{2}{\sqrt{n}}\|\wlm\|_2$ and $\|\bX - \bar \bX\|_F \leq \epsilon$.

\subsection{Proof of Theorem~\ref{theorem_frequency_distribution} (Gradient Vanishing)}
We prove this result in a more realistic setting with multi-head attention and SwiGLU. The attention computation is
\begin{align*}
        \bQ_h &= \bX_1(\bW_Q)_h \\
        \bK_h &= \bX_1(\bW_K)_h\\
        \bV_h &= \bX_1(\bW_V)_h \\
        \bO_h &= \softmaxrow(\frac{\bQ_h\bK_h^T}{\sqrt{d}} + \tlower{\bM}{causal}) \bV_h \\
        \attn(\bX) &= \begin{bmatrix}
            \bO_1 & \bO_2 & ... & \bO_H
        \end{bmatrix}
        \bW_O \\
        \bY_1&= \bX_1 + \attn(\bX)
\end{align*}
The SwiGLU FFN is
\begin{align*}
    \ffn(\bX_2) &= (\silu(\bX_2\bW_1) \odot (\bX_2\bW_3))\bW_2 \\
    \bY_2 &= \bX_2 + \ffn(\bX_2)
\end{align*}
where $(\tlower{\bM}{causal})_{i, j} = 0$ for $i \geq j$, and $-\infty$ otherwise, and $\silu(\bX_2) = (\bX_2 \odot \sigmoid(\bX_2))$. RMSNorm is unchanged.

\begin{lemma}[Forward closure of row-constant states]
\label{lemma_row_constant_forward_closure}
If the input to an attention or FFN sublayer is row-constant, then its residual output is row-constant. Row-wise RMSNorm also maps a row-constant matrix to a row-constant matrix. Consequently, once a post-norm state is row-constant, all downstream hidden states are row-constant.
\end{lemma}
\begin{proof}
Let $\bX=\done_n\bx^\top$. For attention head $h$, $\bV_h=\done_n\bv_h^\top$ for some $\bv_h$. Its causal softmax matrix $\bP_h$ is row-stochastic, so $\bP_h\done_n=\done_n$ and hence $\bP_h\bV_h=\done_n\bv_h^\top$. Concatenation, output projection, and residual addition therefore preserve identical rows. For SwiGLU, both $\bX\bW_1$ and $\bX\bW_3$ have identical rows; the same is true after applying SiLU, the Hadamard product, the output projection, and the residual addition. Finally, RMSNorm applies the same row-wise map to identical rows.
\end{proof}

\paragraph{Backward Propagation of Zero-Column-Sum Gradients}
The proof starts at the LM head and then moves backward through the collapsed blocks. The key invariant is that once the gradient has zero column sum at one interface, each collapsed module preserves that property, and the FFN and attention blocks also force their own parameter gradients to vanish. The main proof is divided into the following steps
\begin{enumerate}
    \item If outputs equal the frequency distribution at all positions, then for the pre-LM-head embedding $\bX^H$:
    \begin{align*}
        \done_n^T\frac{\partial \caL}{\partial \bX^H} = 0
    \end{align*}
    Denoting the LM-head matrix by $\wlm$, if also $\bX^H = \done_n(\bs^H)^\top$, then
    \begin{align*}
        \frac{\partial \caL}{\partial \wlm} = 0
    \end{align*}
    \item If $\done_n^T\frac{\partial \caL}{\partial \bY_2^{l}} = 0$ and $\bX_2^{l} = \done_n(\bx_2^{l})^\top$, then all FFN parameter gradients are zero, and $\done_n^T\frac{\partial \caL}{\partial \bX_2^{l}} = 0$.
    \item If $\done_n^T\frac{\partial \caL}{\partial \bY_1^{l}} = 0$ and $\bX_1^{l} = \done_n(\bx_1^{l})^\top$, then all attention parameter gradients are zero, and $\done_n^T\frac{\partial \caL}{\partial \bX_1^{l}} = 0$.
    \item If $\done_n^T\frac{\partial \caL}{\partial \bX_2^{l}} = 0$ and $\bY_1^{l} = \done_n(\by_1^{l})^\top$, then RMSNorm preserves zero column sum, i.e. $\done_n^T\frac{\partial \caL}{\partial \bY_1^{l}} = 0$. The same argument applies to the RMSNorm between $\bY_2^{l}$ and $\bX_1^{l+1}$.
\end{enumerate}
We now begin to state our lemma formally
\begin{lemma}
\label{lm_head_gradient}
    Let $\bX^H$ be the embedding before the LM-head. If the output probability equals the frequency distribution, then the gradient w.r.t. $\bX^H$ has zero column sum:
    \begin{align*}
        \done_n^T\frac{\partial \caL}{\partial \bX^H} = 0
    \end{align*}
    Denote the matrix of LM-head as $\wlm$, then if also $\bX^H = \done_n(\bx^H)^\top$, then
    \begin{align*}
        \frac{\partial \caL}{\partial \wlm} = 0
    \end{align*}
\end{lemma}
This lemma establishes the base case: if the output equals the frequency distribution and the last hidden state is rank-one, all gradients vanish at the LM head. The following lemmas show that the zero-column-sum property of the gradient propagates backward through each module.
\begin{theorem} \label{theorem_gradient_vanishing}
Now we prove the gradient-vanishing part of Theorem~\ref{theorem_frequency_distribution}.
\begin{enumerate}[(a)]
    \item \label{item_ffn_gradient} (Gradient vanishing of FFN) If \ $\done_n^T\frac{\partial \caL}{\bY_2^{l}} = 0$, and $\bX_2^{l} = \done_n(\bx_2^{l})^\top$, then $\frac{\partial \caL}{\partial \bW} = 0$ for every parameter matrix $\bW$ in Feed-Forward module, and $\done_n^T\frac{\partial \caL}{\bX_2^{l}} = 0$. 
    \item \label{item_attn_gradient} (Gradient vanishing of Attention)
    If \ $\done_n^T\frac{\partial \caL}{\bY_1^{l}} = 0$, and $\bX_1^{l} = \done_n(\bx_1^{l})^\top$, then $\frac{\partial \caL}{\partial \bW} = 0$ for every parameter matrix $\bW$ in Attention module, and $\done_n^T\frac{\partial \caL}{\bX_1^{l}} = 0$.
\end{enumerate}
\end{theorem}

\begin{lemma}
\label{rms_gradient}
    (RMS norm preserves zero column sum of gradient) If \ $\done_n^T\frac{\partial \caL}{\bX_2^{l}} = 0$, and $\bY_1^{l} = \done_n(\by_1^{l})^\top$, then $\done_n^T\frac{\partial \caL}{\bY_1^{l}} = 0$. If \ $\done_n^T\frac{\partial \caL}{\bX_1^{l+1}} = 0$, and $\bY_2^{l} = \done_n(\by_2^{l})^\top$, then $\done_n^T\frac{\partial \caL}{\bY_2^{l}} = 0$. 
\end{lemma}

\subsubsection{Proof of Lemma~\ref{lm_head_gradient}}
\begin{proof}
Let $\hat\bP$ be the output probability matrix after softmax, and let $\bE(\by)$ be the one-hot label matrix (row $i$ has 1 at index $y_i$). Denote the LM-head input by $\bX$. Then
\begin{align*}
    \frac{\partial \caL}{\partial \bX} &= (\hat\bP - \bE(\by))\wlm^T, \qquad 
    \frac{\partial \caL}{\partial \wlm} &= \bX^T(\hat\bP - \bE(\by))
\end{align*}
Suppose label $y_i$ appears $k_i$ times. Then if $\hat\bP_{ij} = \frac{k_i}{n}$, we have $\done_n^T\hat\bP = \done_n^T\bE(\by)$. Therefore,
\begin{align*}
    \frac{\partial \caL}{\partial \wlm} &= \bX^T(\hat\bP - \bE(\by))= \bx^T\done_n^T(\hat\bP - \bE(\by)) = 0 \\
    \done_n^T\frac{\partial \caL}{\partial \bX} &= \done_n^T(\hat\bP - \bE(\by))\wlm^T  = 0
\end{align*}
\end{proof}

\subsubsection{Proof of Theorem~\ref{theorem_gradient_vanishing} (\ref{item_ffn_gradient})}
\begin{proof}
For the FFN layer we consider SwiGLU, where
\begin{align*}
    \ffn(\bX) &= (\silu(\bX\bW_1) \odot (\bX\bW_3))\bW_2
\end{align*}
where $\silu(\bX) = (\bX \odot \sigmoid(\bX))$ \\
Let 
\begin{align*}
    \bA &= \bX\bW_1\\
    \bG &= \silu(\bA) \\
    \bU&= \bX \bW_3 \\
    \bH &= \bG \odot \bU \\
    \bY &= \bX + \ffn(\bX)
\end{align*}
The two induction hypotheses are
\begin{itemize}
    \item Let $\done_n^T \in \mathbb{R}^n$ be the all-one row vector; then $\done_n^T\frac{\partial \caL}{\partial \bY} = 0$.
    \item Input $\bX = \done_n\bx^\top$ for some $\bx$ 
\end{itemize}
Then 
\begin{align*}
    \frac{\partial \caL}{\partial \bW_2} &= \bH^T \frac{\partial \caL}{\partial \bY}\\
    \frac{\partial \caL}{\partial \bH} &= \frac{\partial \caL}{\partial \bY}\bW_2^T\\
    \frac{\partial \caL}{\partial \bG} &= \frac{\partial \caL}{\partial \bH} \odot \bU \\
     \frac{\partial \caL}{\partial \bU} &= \frac{\partial \caL}{\partial \bH} \odot \bG \\
     \frac{\partial \caL}{\partial \bW_3} &= \bX^T \frac{\partial \caL}{\partial \bU} \\
     \frac{\partial \caL}{\partial \bA} &= \frac{\partial \caL}{\partial \bG} \odot (\sigma(\bA) + \bA \odot \sigma(\bA) \odot (1 - \sigma(\bA))) \\
     \frac{\partial \caL}{\partial \bW_1} &= \bX^T \frac{\partial \caL}{\partial \bA} \\
     \frac{\partial \caL}{\partial \bX} &= \frac{\partial \caL}{\partial \bY} + \frac{\partial \caL}{\partial \bA} \bW_1^T + \frac{\partial \caL}{\partial \bU} \bW_3^T
\end{align*}
Here we need to use the fact that
\begin{align*}
    \bM \odot (\done_n\bv^\top) &= \bM \Diag(\bv)
\end{align*}
Since $\bA = \bX \bW_1 = \done_n \bx^\top \bW_1 := \done_n\ba^\top$, we have $\sigma(\bA) = \done_n \bxi^\top$, $\bG := \done_n \bg^\top$ for some $\bg$, then
\begin{align*}
    \bH &= \bG \odot \bU = (\done_n \bg^\top) \odot (\done_n \bx^\top \bW_3) = (\done_n \bg^\top)\Diag(\bW_3^\top\bx)
\end{align*}
So
\begin{align*}
    \frac{\partial \caL}{\partial \bW_2} &= \bH^T \frac{\partial \caL}{\partial \bY} = \Diag(\bW_3^\top\bx)\bg(\done_n^T\frac{\partial \caL}{\partial \bY}) = 0
\end{align*}
and
\begin{align*}
    \frac{\partial \caL}{\partial \bW_3} &= \bX^T ((\frac{\partial \caL}{\partial \bY}\bW_2^T) \odot \bG) = \bx\done_n^T((\frac{\partial \caL}{\partial \bY}\bW_2^T) \Diag(\bg)) = \bx(\done_n^T\frac{\partial \caL}{\partial \bY}) \bW_2^T\Diag(\bg) = 0
\end{align*}
and
\begin{align*}
    \frac{\partial \caL}{\partial \bA} &= \frac{\partial \caL}{\partial \bH} \odot \bU \odot (\sigma(\bA) + \bA \odot \sigma(\bA) \odot (1 - \sigma(\bA))) \\
    &= \frac{\partial \caL}{\partial \bY}\bW_2^T \Diag(\bW_3^\top\bx)\Diag(\bxi + \ba \odot \bxi \odot (1 - \bxi)) =: \frac{\partial \caL}{\partial \bY} \bC
\end{align*}
for some matrix $\bC = \bW_2^T \Diag(\bW_3^\top\bx)\Diag(\bxi + \ba \odot \bxi \odot (1 - \bxi))$, and 
\begin{align*}
    \frac{\partial \caL}{\partial \bW_1} &= \bX^T \frac{\partial \caL}{\partial \bA} = \bx(\done_n^T \frac{\partial \caL}{\partial \bY}) \bC = 0
\end{align*}
and 
\begin{align*}
    \frac{\partial \caL}{\partial \bU} &= (\frac{\partial \caL}{\partial \bY} \bW_2^T) \odot \bG = \frac{\partial \caL}{\partial \bY} \bW_2^T \Diag (\bg)
\end{align*}
So the gradient $\frac{\partial \caL}{\partial \bX}$ satisfies
\begin{align*}
    \done_n^T\frac{\partial \caL}{\partial \bX} &= \done_n^T\frac{\partial \caL}{\partial \bY} + (\done_n^T \frac{\partial \caL}{\partial \bY}) \bC\bW_1^T + (\done_n^T \frac{\partial \caL}{\partial \bY})\bW_2^T\Diag(\bg)\bW_3^T = 0.
\end{align*}
\end{proof}

\subsubsection{Proof of Theorem~\ref{theorem_gradient_vanishing} (\ref{item_attn_gradient})}
\begin{proof}
Denote the attention operation $\attn(\bX)$ as computing the following
\begin{align*}
    \bQ_h &= \bX(\bW_Q)_h \\
    \bK_h &= \bX(\bW_K)_h\\
    \bV_h &= \bX(\bW_V)_h \\
    \bO_h &= \softmaxrow(\frac{\bQ_h\bK_h^T}{\sqrt{d}} + \tlower{\bM}{causal}) \bV_h \\
    \attn(\bX) &= \begin{bmatrix}
        \bO_1 & \bO_2 & ... & \bO_H
    \end{bmatrix}
    \bW_O \\
    \bY&= \bX + \attn(\bX)
\end{align*}
where $(\tlower{\bM}{causal})_{i, j} = 0$ for $i \geq j$, and $-\infty$ otherwise. 
Besides the aforementioned notations, we need to add some extra notations. Let 
\begin{align*}
    \bP_h &= \softmaxrow(\bS_h + \tlower{\bM}{causal}) \\
    \bS_h&= \frac{\bQ_h\bK_h^T}{\sqrt{d}}\\
    \hat \bS_h&= \bS_h + \tlower{\bM}{causal}\\
    \bO &= \begin{bmatrix}
        \bO_1 & \bO_2 & ... & \bO_H
    \end{bmatrix} \\
    \bW_O &= \begin{bmatrix}
        (\bW_O)^T_1 &  (\bW_O)^T_2 & ... &  (\bW_O)^T_H
    \end{bmatrix}
\end{align*}
Then
\begin{align*}
    \frac{\partial \caL}{\partial \bW_O} &= \bO^T\frac{\partial \caL}{\partial \bY} = \begin{bmatrix}
        (\frac{\partial \caL}{\partial \bY})^T\bO_1 & ... & (\frac{\partial \caL}{\partial \bY})^T\bO_H
    \end{bmatrix}^T
\end{align*}
and since each row of $\bP_h$ sums to 1 (i.e. $\bP_h\done_n = \done_n$), we have
\begin{align*}
    \bO_h^T\frac{\partial \caL}{\partial \bY} &= (\bP_h\bX(\bW_V)_h)^T\frac{\partial \caL}{\partial \bY} \\
    &= ((\bP_h\done_n)\bx^\top(\bW_V)_h)^T\frac{\partial \caL}{\partial \bY} \\
    &=  (\done_n\bx^\top(\bW_V)_h)^T\frac{\partial \caL}{\partial \bY} \\
    &= ((\bW_V)_h^\top\bx)(\done_n^T\frac{\partial \caL}{\partial \bY}) \\
    &= 0
\end{align*}
So $\frac{\partial \caL}{\partial \bW_O} = 0$.
\begin{align*}
    \frac{\partial \caL}{\partial \bO_h} &= \frac{\partial \caL}{\partial \bY}(\bW_O)_h \\
    \frac{\partial \caL}{\partial \bV_h} &= \bP_h^T \frac{\partial \caL}{\partial \bO_h} \\
    \frac{\partial \caL}{\partial (\bW_V)_h} &= \bX^T\frac{\partial \caL}{\partial \bV_h} \\
    \frac{\partial \caL}{\partial \bP_h} &=  \frac{\partial \caL}{\partial \bO_h}\bV_h^T \\
\end{align*}
where $\bW_O = \begin{bmatrix}
    (\bW_O)_1 & (\bW_O)_2 & ... & (\bW_O)_H
\end{bmatrix}^T$.
Since the causal mask is constant, the unmasked derivative is the usual row-wise softmax derivative, with masked entries equal to zero. Therefore
\begin{align*}
    \frac{\partial \caL}{\partial \bS_h}
    &= \bP_h \odot \frac{\partial \caL}{\partial \bP_h}
    - \Diag\!\left(\left(\bP_h \odot \frac{\partial \caL}{\partial \bP_h}\right)\done_n\right)\bP_h.
\end{align*}
Using $\bP_h\done_n=\done_n$, we obtain
\begin{align*}
    \frac{\partial \caL}{\partial \bS_h}\done_n
    &= \left(\bP_h \odot \frac{\partial \caL}{\partial \bP_h}\right)\done_n
    - \Diag\!\left(\left(\bP_h \odot \frac{\partial \caL}{\partial \bP_h}\right)\done_n\right)\done_n
    = \bzero.
\end{align*}
Consequently,
\begin{align*}
    \frac{\partial \caL}{\partial \bQ_h}  &= \frac{1}{\sqrt{d}}\frac{\partial \caL}{\partial \bS_h} \bK_h, \\
    \frac{\partial \caL}{\partial \bK_h}  &= \frac{1}{\sqrt{d}}(\frac{\partial \caL}{\partial \bS_h})^T \bQ_h, \\
    \frac{\partial \caL}{\partial (\bW_Q)_h}  &= \frac{1}{\sqrt{d}}\bX^T\frac{\partial \caL}{\partial \bS_h} \bK_h, \\
    \frac{\partial \caL}{\partial (\bW_K)_h}  &= \frac{1}{\sqrt{d}}\bX^T(\frac{\partial \caL}{\partial \bS_h})^T \bQ_h.
\end{align*}
By writing $\bX = \done_n\bx^\top$, we have
\begin{align*}
    \frac{\partial \caL}{\partial \bQ_h} &= \frac{1}{\sqrt{d}}\frac{\partial \caL}{\partial \bS_h} \bK_h  
    = \frac{1}{\sqrt{d}}\frac{\partial \caL}{\partial \bS_h} \done_n\bx^\top(\bW_K)_h 
    = 0 \\
    \frac{\partial \caL}{\partial (\bW_K)_h}  &= \frac{1}{\sqrt{d}}\bX^T(\frac{\partial \caL}{\partial \bS_h})^T \bQ_h
    = \frac{1}{\sqrt{d}}\bx\left((\frac{\partial \caL}{\partial \bS_h})\done_n\right)^\top \bQ_h = 0 \\
    \frac{\partial \caL}{\partial (\bW_Q)_h}  &= \bX^T \frac{\partial \caL}{\partial \bQ_h} = 0
\end{align*}
For the gradient w.r.t. input, we have
\begin{align*}
    \frac{\partial \caL}{\partial \bX} &= \sum_{h}\frac{\partial \caL}{\partial \bQ_h} (\bW_Q)_h^T + \sum_{h}\frac{\partial \caL}{\partial \bK_h}(\bW_K)_h^T + \sum_{h}\frac{\partial \caL}{\partial \bV_h} (\bW_V)_h^T + \frac{\partial \caL}{\partial \bY} \\
    &= \sum_{h}\frac{\partial \caL}{\partial \bK_h} (\bW_K)_h^T + \sum_{h}\frac{\partial \caL}{\partial \bV_h} (\bW_V)_h^T + \frac{\partial \caL}{\partial \bY}
\end{align*}
and
\begin{align*}
    \done_n^T\frac{\partial \caL}{\partial \bK_h} &= \frac{1}{\sqrt{d}}\done_n^T(\frac{\partial \caL}{\partial \bS_h})^T \bQ_h 
    = \frac{1}{\sqrt{d}}(\frac{\partial \caL}{\partial \bS_h}\done_n)^T \bQ_h = 0 \\
    \done_n^T\frac{\partial \caL}{\partial \bV_h} &= \done_n^T\bP_h^T  \frac{\partial \caL}{\partial \bY}(\bW_O)_h = (\done_n^T \frac{\partial \caL}{\partial \bY})(\bW_O)_h = 0
\end{align*}
The derivation for $\done_n^T\frac{\partial \caL}{\partial \bV_h}$ uses the facts that $\bP_h \done_n = \done_n$ and $\done_n^T \frac{\partial \caL}{\partial \bY} = 0$.
\begin{align*}
    \done_n^T\frac{\partial \caL}{\partial \bX} &= \sum_{h}\done_n^T\frac{\partial \caL}{\partial \bK_h} (\bW_K)_h^T + \sum_{h}\done_n^T\frac{\partial \caL}{\partial \bV_h} (\bW_V)_h^T + \done_n^T\frac{\partial \caL}{\partial \bY} = 0
\end{align*}
Therefore, after the backward pass of the attention layer, the gradient still satisfies $\done_n^T\frac{\partial \caL}{\partial \bX} = 0$.
\end{proof}

\subsubsection{Proof of Lemma~\ref{rms_gradient}}
\begin{proof}
Under the ideal RMSNorm used in the main text, suppose $\bY=\done_n\by^\top$. Every row then has the same RMSNorm Jacobian, so the row-wise backward derivative gives
\begin{align*}
    \frac{\partial \caL}{\partial \bY}
    = \frac{\partial \caL}{\partial \bX}
    \frac{\sqrt{d}}{\|\by\|_2}
    \left(\bI-\frac{\by\by^\top}{\|\by\|_2^2}\right).
\end{align*}
Therefore,
\begin{align*}
    \done_n^T \frac{\partial \caL}{\partial \bY}
    = \left(\done_n^T\frac{\partial \caL}{\partial \bX}\right)
    \frac{\sqrt{d}}{\|\by\|_2}
    \left(\bI-\frac{\by\by^\top}{\|\by\|_2^2}\right)
    = \bzero.
\end{align*}
\end{proof}

\section{Validation and Supplementary Diagnostics}\label{app:diagnostics}

This section checks the approximations used in Appendix~\ref{app:forward_amplification_details} and adds one repeated-run diagnostic for the backward transition.

The first subsection records a supporting expectation-level check for the attention second moment:
\[
\mathbb{E}\!\left[\bP\bX_1\bX_1^\top\bP^\top\right]
\approx
\mathbb{E}[\bP]\,\mathbb{E}[\bX_1\bX_1^\top]\,\mathbb{E}[\bP^\top]
\]
The second subsection checks the separate prefix-averaging approximation, and the last subsection shows that the backward-side transition pattern also appears on another collapsing run. These experiments are all done on the same architecture and optimization setup as the main text ($d=512$, $d_{\mathrm{ff}}=1536$, 4 heads, $n=2048$).

\subsection{Layerwise Validation of an Unconditional Proxy for the Attention Second Moment}
\label{app:validation_attention_expectation_product}
The forward appendix directly replaces $\bP^l$ by $\bC_n$. In parallel, this subsection records the initialization-time behavior of the unconditional product replacement
\[
\mathbb{E}[\bP^l]\,\mathbb{E}[\bX_1^l(\bX_1^l)^\top]\,\mathbb{E}[(\bP^l)^\top].
\]
We record the resulting errors layer by layer on the same setup in the main text. We measure the relative Frobenius error of this replacement
\[
\delta_{\mathrm{attn}}^{l}
:=
\frac{
\left\|
\mathbb{E}\!\left[\bP^l \bX_1^l(\bX_1^l)^\top (\bP^l)^\top\right]
-
\mathbb{E}[\bP^l]\,
\mathbb{E}[\bX_1^l(\bX_1^l)^\top]\,
\mathbb{E}[(\bP^l)^\top]
\right\|_F
}{
\left\|
\mathbb{E}\!\left[\bP^l \bX_1^l(\bX_1^l)^\top (\bP^l)^\top\right]
\right\|_F
}.
\]

Because the forward derivation also uses the projected quantity, we separately measure
\[
\delta_{\mathrm{proj}}^{l}
:=
\frac{
\left\|
\simmat \mathbb{E}\!\left[\bP^l \bX_1^l(\bX_1^l)^\top (\bP^l)^\top\right]\simmat
-
\simmat \mathbb{E}[\bP^l]\,\mathbb{E}[\bX_1^l(\bX_1^l)^\top]\,\mathbb{E}[(\bP^l)^\top]\simmat
\right\|_F
}{
\left\|
\simmat \mathbb{E}\!\left[\bP^l \bX_1^l(\bX_1^l)^\top (\bP^l)^\top\right]\simmat
\right\|_F
}.
\]
Figure~\ref{fig:attention_expectation_product_validation_layerwise} plots $\delta_{\mathrm{attn}}^{l}$ and
$\delta_{\mathrm{proj}}^{l}$ against layer index $l$. The relative error is small throughout the stack, which is consistent with the initialization-time attention surrogate used in the forward appendix.

\begin{figure}[H]
    \centering
    \includegraphics[width=0.9\linewidth]{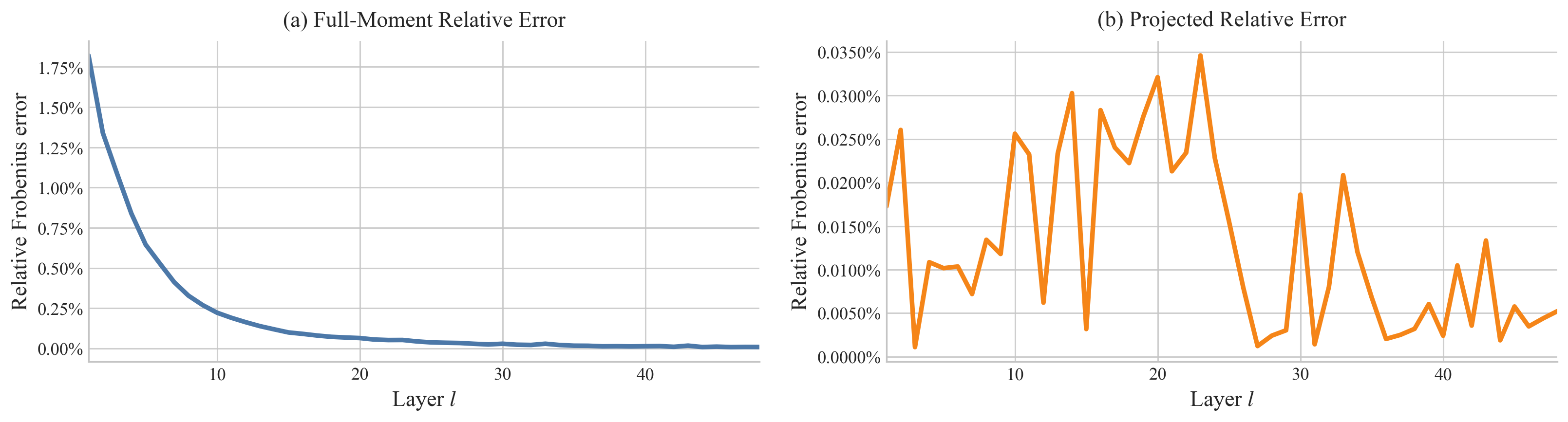}
    \caption{Layerwise validation of the unconditional proxy $\mathbb{E}[\bP^l \bX_1^l(\bX_1^l)^\top (\bP^l)^\top]\approx\mathbb{E}[\bP^l]\mathbb{E}[\bX_1^l(\bX_1^l)^\top]\mathbb{E}[(\bP^l)^\top]$. The horizontal axis is layer index in the 48-layer Post-Norm stack ($d=512$, $d_{\mathrm{ff}}=1536$, 4 heads, $n=2048$). Left: $\delta_{\mathrm{attn}}^{l}$ for different layer, Right: $\delta_{\mathrm{proj}}^{l}$ for different layers}
    \label{fig:attention_expectation_product_validation_layerwise}
\end{figure}

\subsection{Initialization-Time Prefix-Averaging Approximation}
The forward similarity increment theorem also uses the separate initialization-time proxy $\expect{}{\bP}\approx \bC_n$ directly. To verify that $\expect{}{\bP^l}\approx \bC_n$ holds at initialization in the actual experimental architecture, evaluate it at step~0, and estimate each layer's mean attention matrix $\mathbb{E}[\bP^l]$ by averaging over 32 random initializations. Figure~\ref{fig:init_prefix_averaging_rel_error_48l_2048} reports
\[
\frac{\|\mathbb{E}[\bP^l]-\bC_n\|_F}{\|\bC_n\|_F}
\]
across depth. The relative Frobenius error stays small throughout the stack (about $2.7\%$ in the first layer and $1.2\%$ by layer~48; mean over layers $\approx 2.1\%$), directly supporting the initialization-time prefix-averaging approximation used in Theorem~\ref{theorem_forward_amplification}.

\begin{figure}[H]
    \centering
    \includegraphics[width=0.6\linewidth]{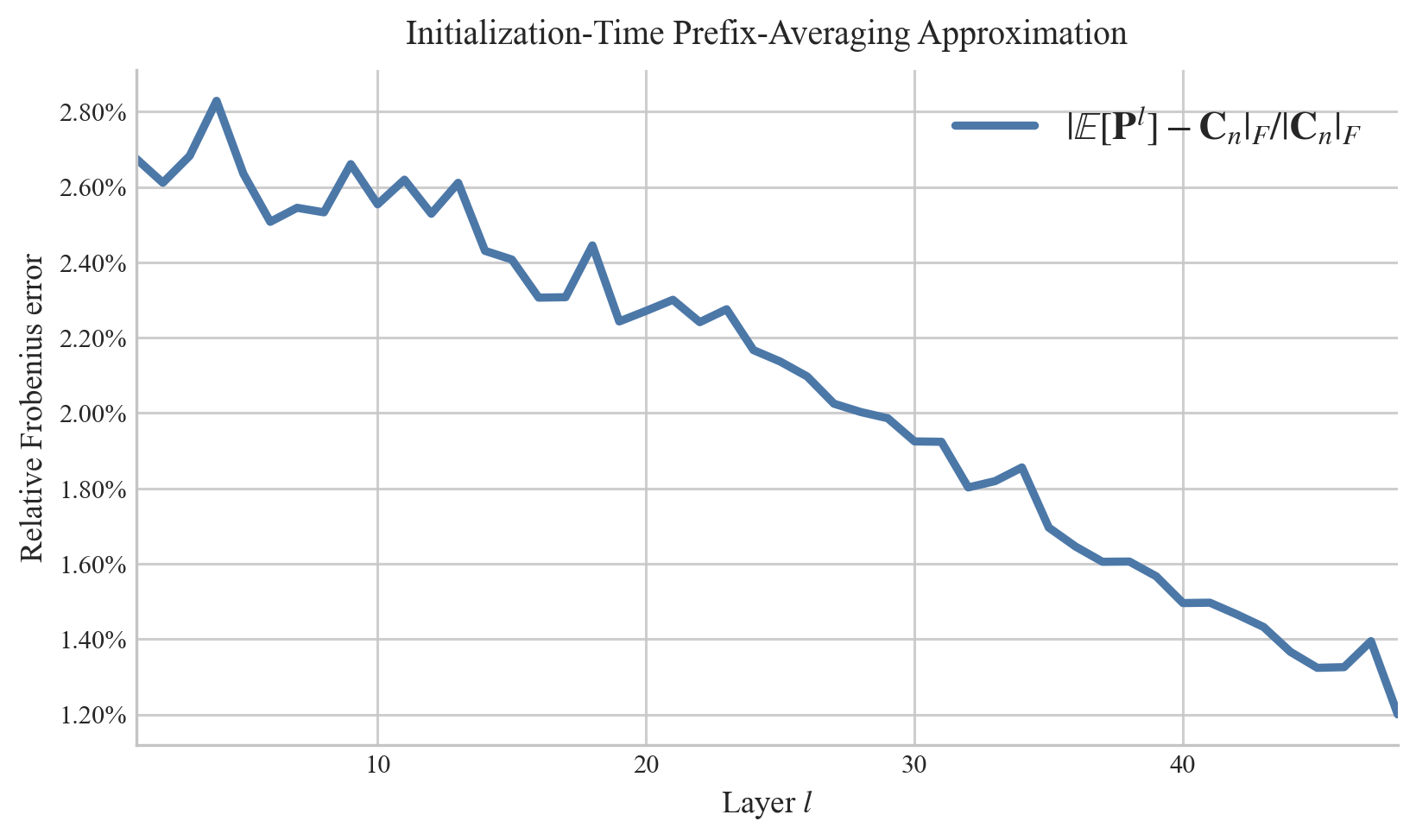}
    \caption{Initialization-time validation of the prefix-averaging approximation. }
    \label{fig:init_prefix_averaging_rel_error_48l_2048}
\end{figure}

\subsection{Supplementary Experiments for Another Collapsing Run}
\label{app:collapse_family_diagnostics}
This subsection repeats the backward diagnostics on another collapsing run under the same architecture and optimization setup. Its role is to show that the transition pattern highlighted in the main text is not specific to a single random seed.

\paragraph{Repeated Run on Another Seed.}
Figures~\ref{fig:appendix_vii_sublayer_amplification} and~\ref{fig:appendix_vii_native_bridge} repeat the same backward-side measurements from the main text on another collapsing run.

\begin{figure}[H]
    \centering
    \includegraphics[width=0.8\textwidth]{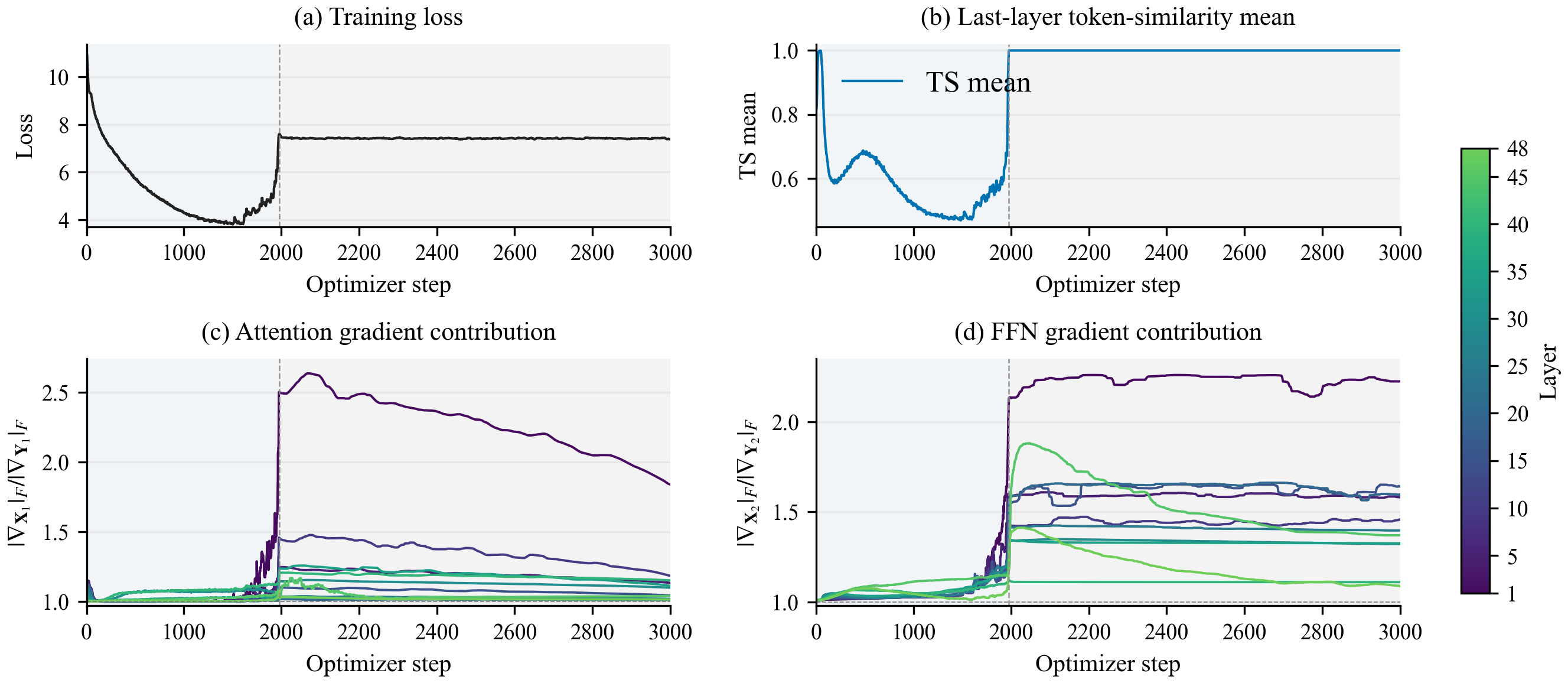}
    \caption{Sublayer gradient contribution in another collapsing Post-Norm run. Dashed vertical line: transition step $\approx 1979$. \textbf{(a)} Training loss. \textbf{(b)} Last-layer token-similarity mean. \textbf{(c)} Attention gradient contribution across layers 1, 5, 10, 15, 20, 25, 30, 35, 40, 45, and 48. \textbf{(d)} FFN gradient contribution across the same layers.}
    \label{fig:appendix_vii_sublayer_amplification}
\end{figure}

\begin{figure}[H]
    \centering
    \includegraphics[width=0.8\textwidth]{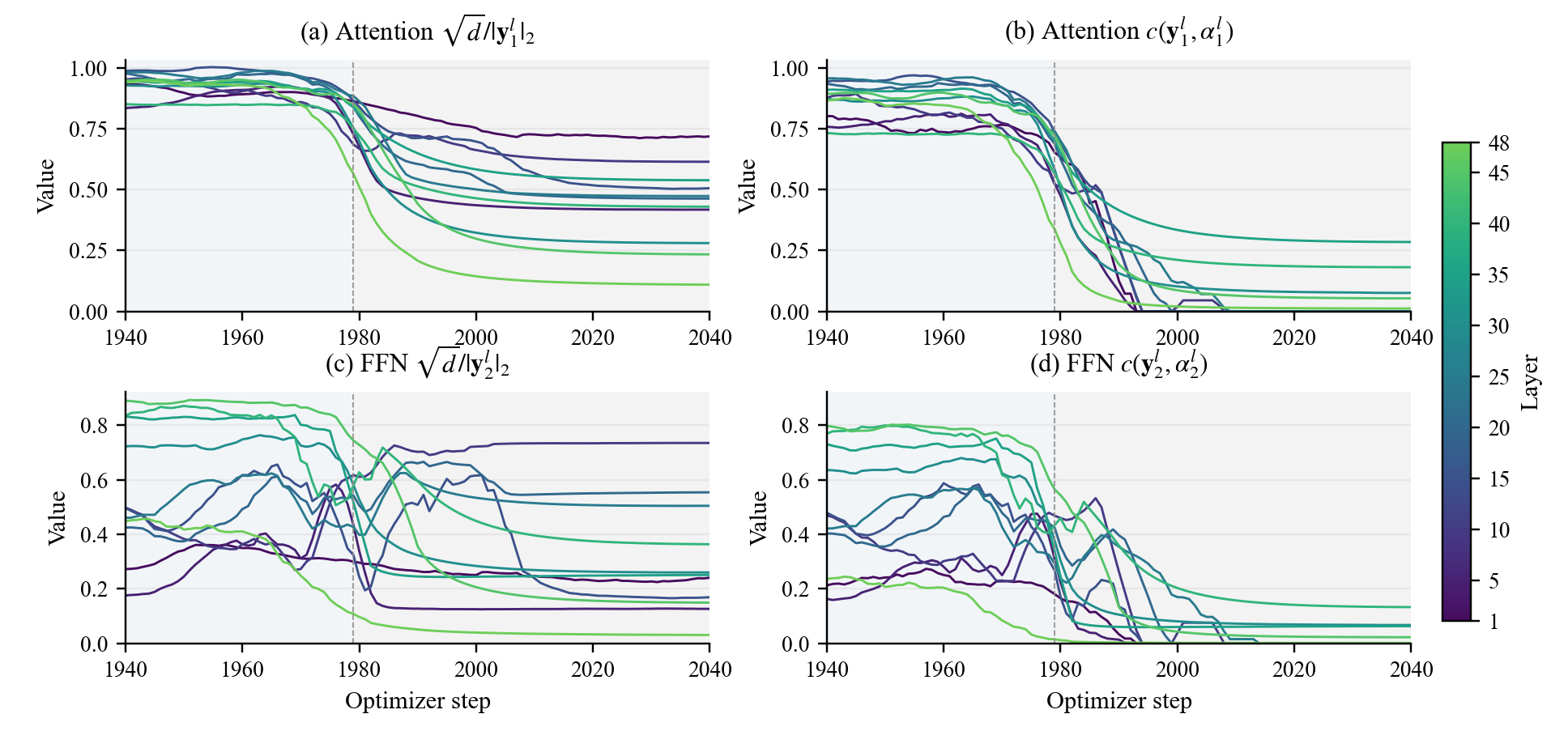}
    \caption{Transition-window diagnostic for another collapsing run. Dashed vertical line: transition step $\approx 1979$; blue/gray shading: pre/post-transition windows. \textbf{(a)} Attention $\sqrt{d}/\|\by_1^l\|_2$. \textbf{(b)} Attention $c(\by_1^l,\alpha_1^l)$. \textbf{(c)} FFN $\sqrt{d}/\|\by_2^l\|_2$. \textbf{(d)} FFN $c(\by_2^l,\alpha_2^l)$.}
    \label{fig:appendix_vii_native_bridge}
\end{figure}

\subsection{Learning-Rate Ablation for Backward Incapacity}
\label{app:masked_lr_sweep_backward_incapacity}
This subsection adds the same backward-incapacity diagnostics for collapsed Post-Norm runs with learning rates $1.2\times10^{-3}$, $1.5\times10^{-3}$, and $1.8\times10^{-3}$. For visual readability, the plots masks the sublayer ambplification and $c$ value when the submodule's input and output gradient norm fall below $10^{-10}$.

\begin{figure}[H]
    \centering
    \includegraphics[width=0.8\textwidth]{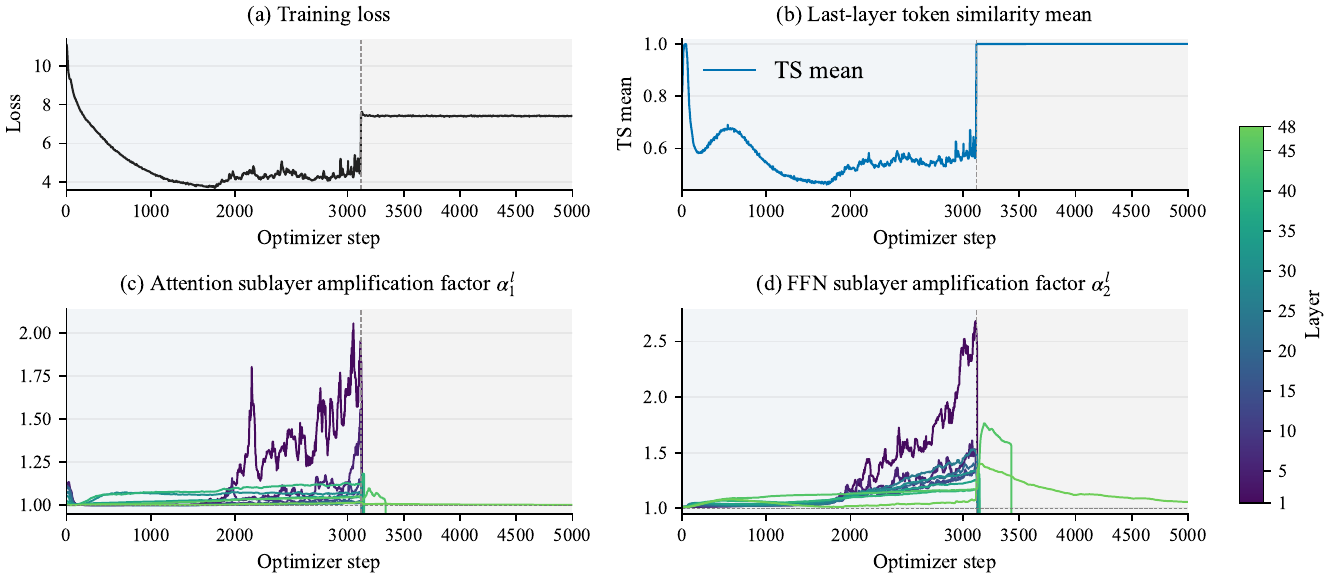}
    \caption{Sublayer gradient contribution for the collapsed Post-Norm run with LR=$1.2\times10^{-3}$.}
    \label{fig:appendix_lr12_sublayer_contribution}
\end{figure}

\begin{figure}[H]
    \centering
    \includegraphics[width=0.8\textwidth]{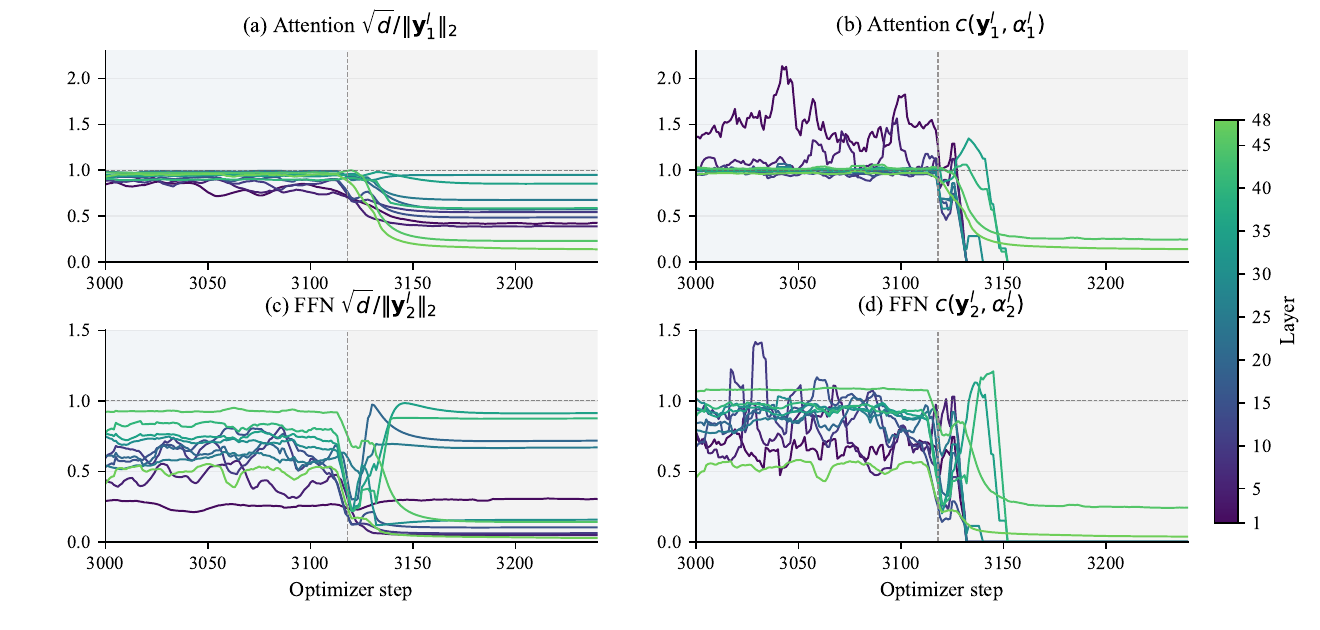}
    \caption{Backward-side diagnostic for the collapsed Post-Norm run with LR=$1.2\times10^{-3}$.}
    \label{fig:appendix_lr12_native_bridge}
\end{figure}

\begin{figure}[H]
    \centering
    \includegraphics[width=0.7\textwidth]{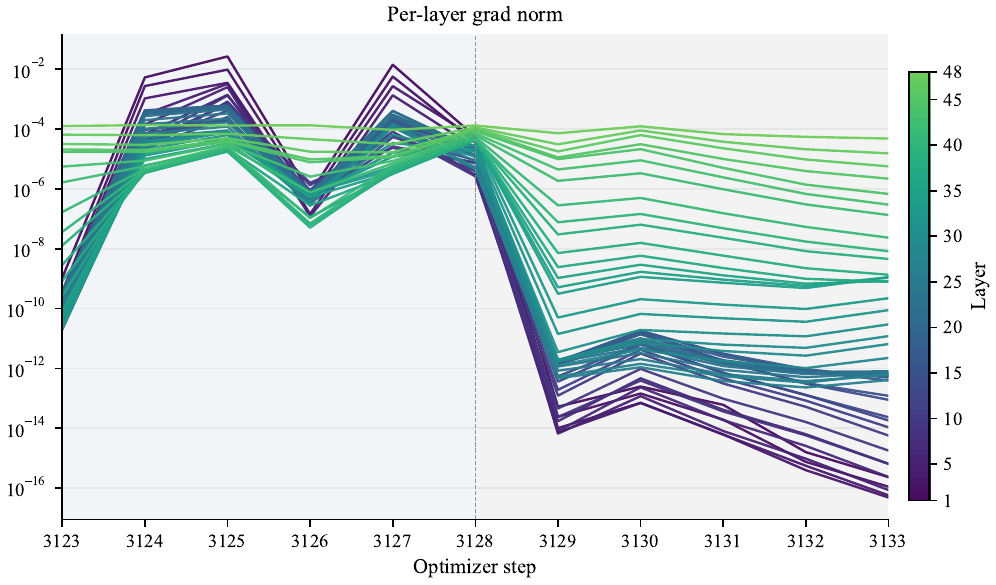}
    \caption{Per-layer gradient norm for the collapsed Post-Norm run with LR=$1.2\times10^{-3}$ over optimizer steps 3123--3133.}
    \label{fig:appendix_lr12_grad_profile}
\end{figure}

\begin{figure}[H]
    \centering
    \includegraphics[width=0.8\textwidth]{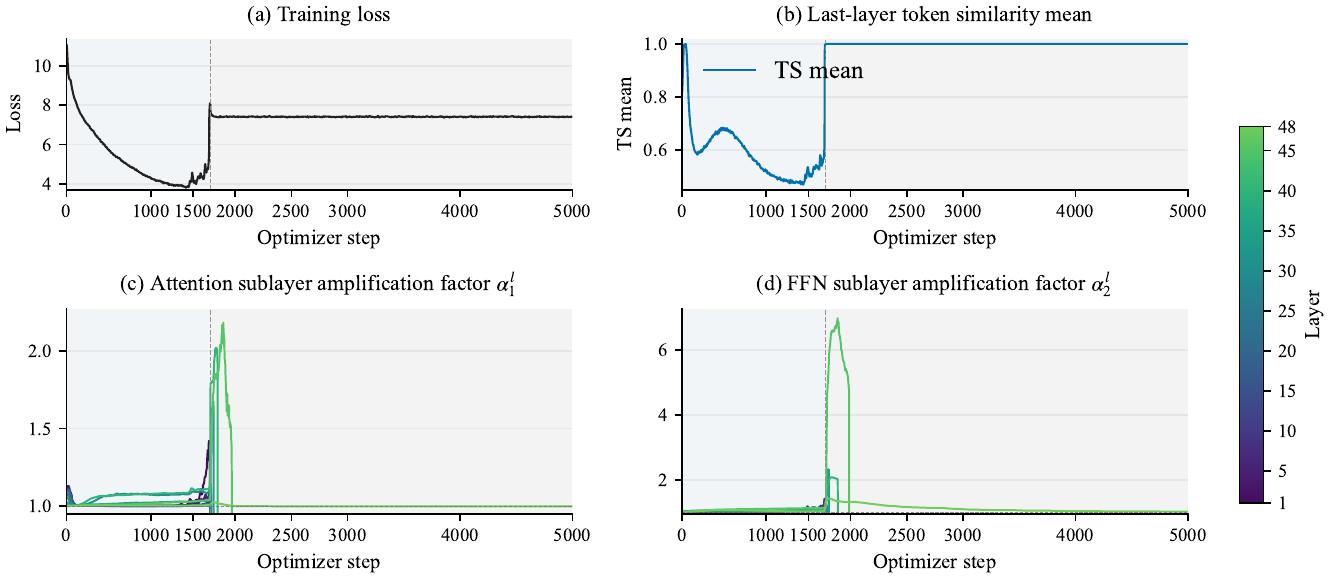}
    \caption{Sublayer gradient contribution for the collapsed Post-Norm run with LR=$1.5\times10^{-3}$.}
    \label{fig:appendix_lr15_sublayer_contribution}
\end{figure}

\begin{figure}[H]
    \centering
    \includegraphics[width=0.8\textwidth]{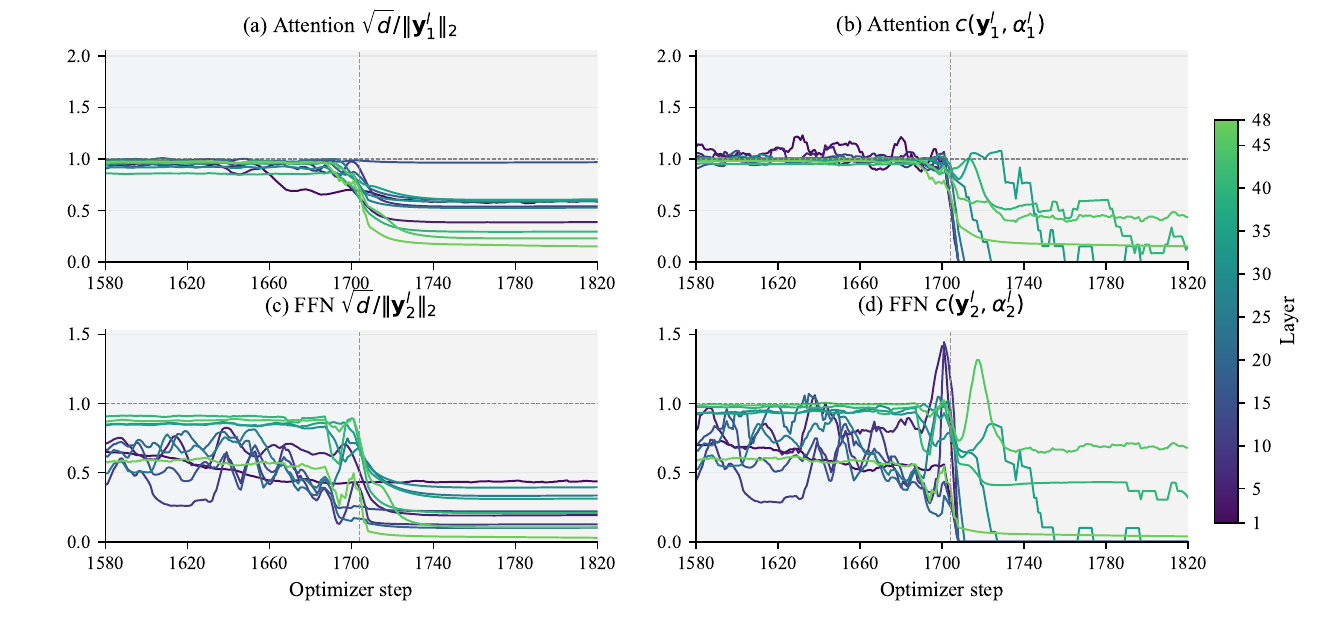}
    \caption{Backward-side diagnostic for the collapsed Post-Norm run with LR=$1.5\times10^{-3}$.}
    \label{fig:appendix_lr15_native_bridge}
\end{figure}

\begin{figure}[H]
    \centering
    \includegraphics[width=0.7\textwidth]{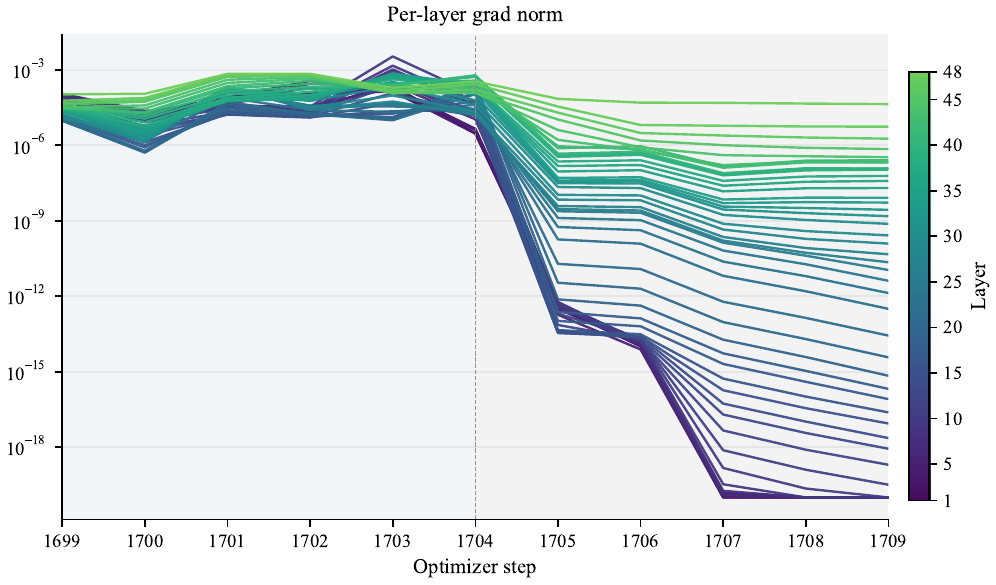}
    \caption{Per-layer gradient norm for the collapsed Post-Norm run with LR=$1.5\times10^{-3}$ over optimizer steps 1699--1709.}
    \label{fig:appendix_lr15_grad_profile}
\end{figure}

\begin{figure}[H]
    \centering
    \includegraphics[width=0.8\textwidth]{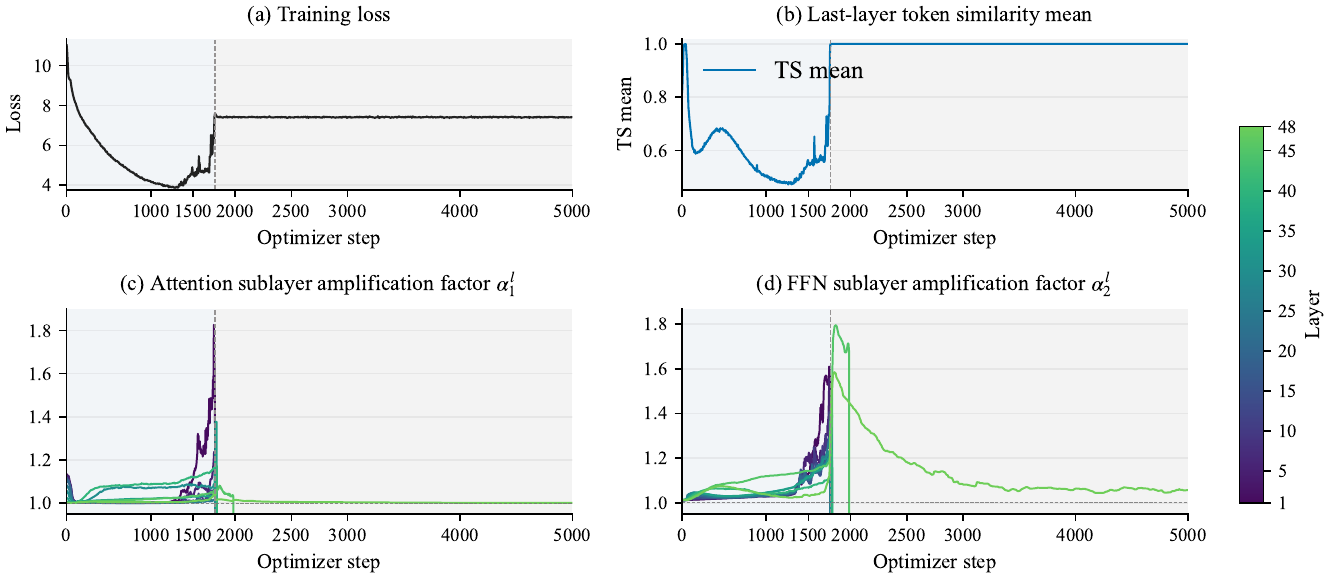}
    \caption{Sublayer gradient contribution for the collapsed Post-Norm run with LR=$1.8\times10^{-3}$.}
    \label{fig:appendix_lr18_sublayer_contribution}
\end{figure}

\begin{figure}[H]
    \centering
    \includegraphics[width=0.8\textwidth]{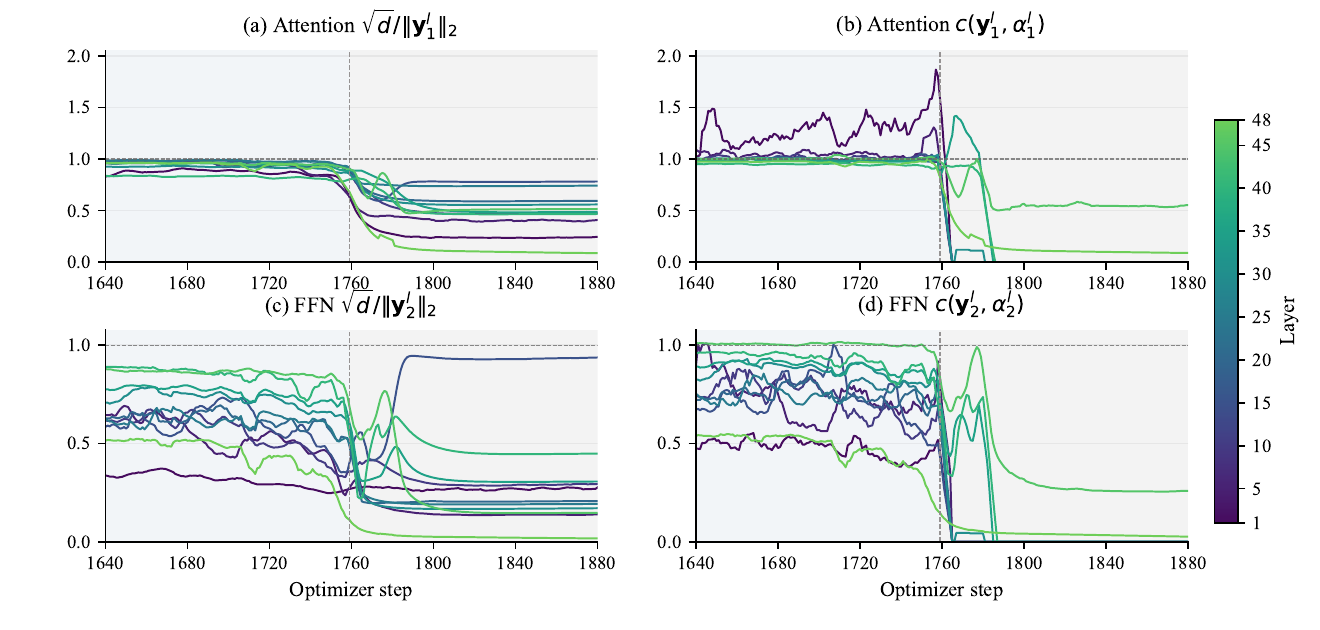}
    \caption{Backward-side diagnostic for the collapsed Post-Norm run with LR=$1.8\times10^{-3}$.}
    \label{fig:appendix_lr18_native_bridge}
\end{figure}

\begin{figure}[H]
    \centering
    \includegraphics[width=0.7\textwidth]{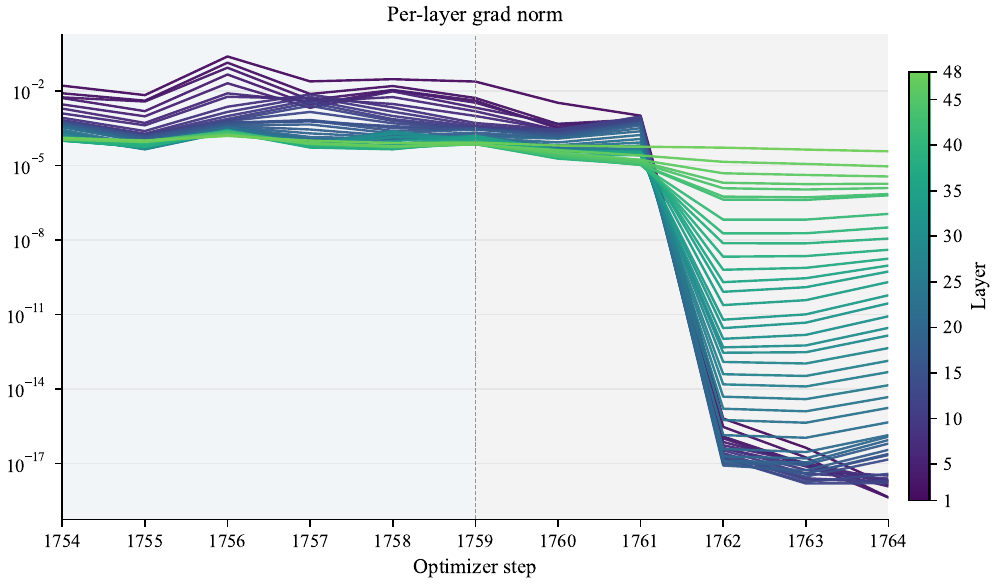}
    \caption{Per-layer gradient norm for the collapsed Post-Norm run with LR=$1.8\times10^{-3}$ over optimizer steps 1754--1764.}
    \label{fig:appendix_lr18_grad_profile}
\end{figure}

\subsection{Non-Collapsed Post-Norm Control Run}
\label{app:non_collapsed_postnorm_control}
This subsection gives a non-collapsed Post-Norm run as a contrast case. In this run, the last-layer token similarity stays away from one and the training loss decreases smoothly. The backward-side quantities also remain stable across training, providing a comparison to the transition pattern shown in the collapsing runs.

\begin{figure}[H]
    \centering
    \includegraphics[width=0.8\textwidth]{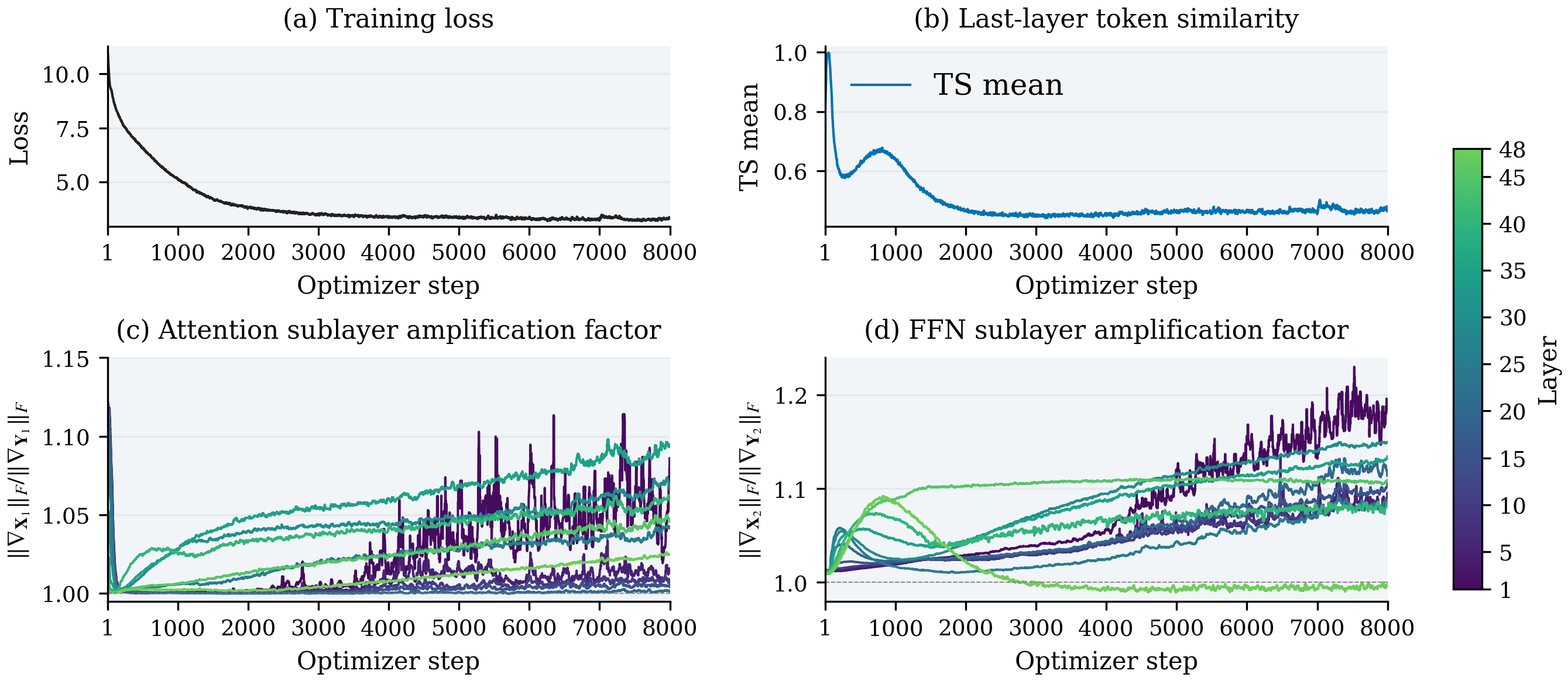}
    \caption{Sublayer gradient contribution in a non-collapsed Post-Norm run. \textbf{(a)} Training loss. \textbf{(b)} Last-layer token-similarity mean. \textbf{(c)} Attention gradient contribution across layers 1, 5, 10, 15, 20, 25, 30, 35, 40, 45, and 48. \textbf{(d)} FFN gradient contribution across the same layers.}
    \label{fig:appendix_noncollapsed_sublayer_contribution}
\end{figure}

\begin{figure}[H]
    \centering
    \includegraphics[width=0.8\textwidth]{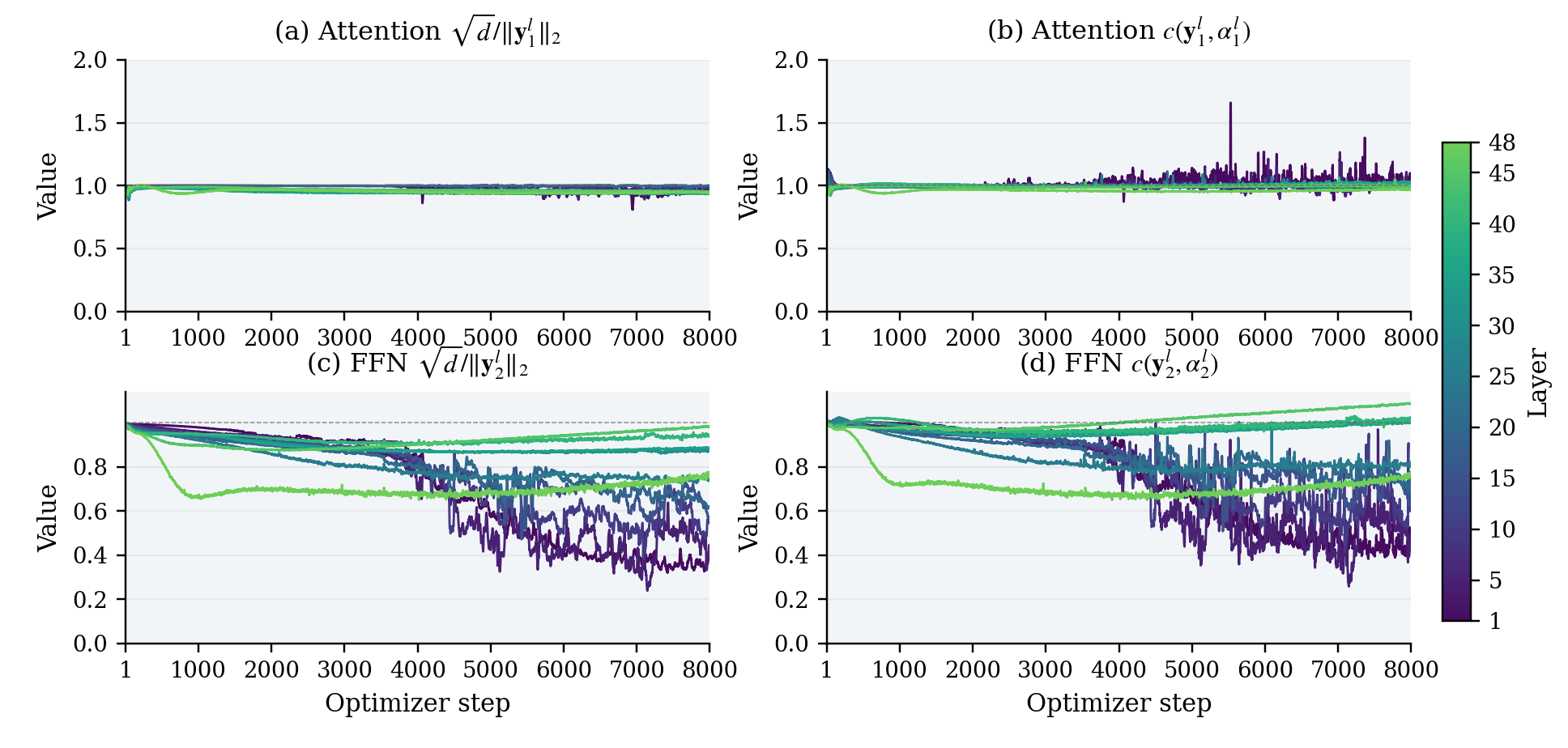}
    \caption{Backward-side diagnostic for the same non-collapsed Post-Norm run. \textbf{(a)} Attention $\sqrt{d}/\|\by_1^l\|_2$. \textbf{(b)} Attention $c(\by_1^l,\alpha_1^l)$. \textbf{(c)} FFN $\sqrt{d}/\|\by_2^l\|_2$. \textbf{(d)} FFN $c(\by_2^l,\alpha_2^l)$.}
    \label{fig:appendix_noncollapsed_native_bridge}
\end{figure}

\end{document}